\documentclass{article}

    \PassOptionsToPackage{numbers, compress}{natbib}

\usepackage[final]{neurips_2020}

\usepackage[utf8]{inputenc} \usepackage[T1]{fontenc}    \usepackage{hyperref}       \usepackage{url}            \usepackage{booktabs}       \usepackage{amsfonts}       \usepackage{nicefrac}       \usepackage{microtype}      

\usepackage[usenames,dvipsnames]{xcolor}

\usepackage{nicefrac}       \usepackage{xspace}
\usepackage{wrapfig}
\usepackage{amsmath}
\usepackage{mathtools}

\usepackage{graphicx}

\usepackage{caption}
\usepackage{subcaption}
\usepackage{float}

\usepackage{multirow}
\usepackage{array}
\newcommand{\PreserveBackslash}[1]{\let\temp=\\#1\let\\=\temp}
\newcolumntype{C}[1]{>{\PreserveBackslash\centering}p{#1}}
\newcolumntype{R}[1]{>{\PreserveBackslash\raggedleft}p{#1}}
\newcolumntype{L}[1]{>{\PreserveBackslash\raggedright}p{#1}}

\usepackage{footnote}
\makesavenoteenv{tabular}
\makesavenoteenv{table}

\usepackage[ruled,fillcomment]{algorithm2e}
\DontPrintSemicolon

\SetCommentSty{mycommentsty}

\usepackage[capitalise]{cleveref}
\crefname{algorithm}{Algorithm}{Algorithms}
\crefname{appendix}{Appendix}{Appendices}

\usepackage{amsthm}

\newtheorem*{example*}{Example}
\newtheorem{example}{Example}
\newtheorem{lemma}{Lemma}
\newtheorem{proposition}{Proposition}

\newtheorem{corollary}{Corollary}

\usepackage{amsmath}
\usepackage{amssymb}

\usepackage[usenames,dvipsnames]{xcolor}
\usepackage{soul}
\sethlcolor{SkyBlue}

\newcommand{\tr}{\text{train}}
\newcommand{\val}{\text{val}}

\newcommand{\pll}{\mathrm{PLL}}

\newcommand{\pos}{\mathrm{joint}}

\newcommand{\T}{\mathcal{T}}
\newcommand{\x}{\mathbf{x}}

\newcommand{\D}{\mathcal{D}}

\DeclareMathOperator*{\argmin}{argmin}

\newcommand{\metav}{\phi}
\newcommand{\taskv}{\theta}

\newcommand{\vmetav}{\vphi}
\newcommand{\vtaskv}{\vtheta}

\newcommand{\shrpfx}{$\vsigma$}\newcommand{\shrreptile}{\shrpfx{}-Reptile}
\newcommand{\shrmaml}{\shrpfx{}-MAML}
\newcommand{\shrimaml}{\shrpfx{}-iMAML}

\newcommand{\elltr}{\ell_{t}^\tr}
\newcommand{\ellval}{\ell_{t}^\val}
\newcommand{\ellpll}{\ell_\pll}

\newcommand{\real}{\mathbb{R}}

\newcommand{\gauss}{{\cal N}}

\newcommand{\myvec}[1]{\mathbf{#1}}
\newcommand{\myvecsym}[1]{\boldsymbol{#1}}

\newcommand{\vone}{\myvecsym{1}}

\newcommand{\vmu}{\myvecsym{\mu}}

\newcommand{\vphi}{\myvecsym{\phi}}
\newcommand{\vPhi}{\myvecsym{\Phi}}

\newcommand{\vtheta}{\myvecsym{\theta}}

\newcommand{\vTheta}{\myvecsym{\Theta}}
\newcommand{\vsigma}{\myvecsym{\sigma}}
\newcommand{\vSigma}{\myvecsym{\Sigma}}

\newcommand{\vxi}{\myvecsym{\xi}}

\newcommand{\vg}{\myvec{g}}

\newcommand{\vv}{\myvec{v}}

\newcommand{\vx}{\myvec{x}}

\newcommand{\vy}{\myvec{y}}

\newcommand{\vB}{\myvec{B}}

\newcommand{\vH}{\myvec{H}}
\newcommand{\vI}{\myvec{I}}

\newcommand{\vP}{\myvec{P}}

\newcommand{\calN}{{\cal N}}

\newcommand{\be}{\begin{equation}}
\newcommand{\ee}{\end{equation}}
\newcommand{\bea}{\begin{eqnarray}}
\newcommand{\eea}{\end{eqnarray}}
\newcommand{\beaa}{\begin{eqnarray*}}
\newcommand{\eeaa}{\end{eqnarray*}}

\DeclareMathAlphabet{\mathpzc}{OT1}{pzc}{m}{n}

\newcommand{\ba}{\[\begin{aligned}}
\newcommand{\ea}{\end{aligned}\]}
\newcommand{\eq}[1]{\begin{align}#1\end{align}}
\renewcommand{\r}{\mathrm{True}}
\newcommand{\nn}{\nonumber}

\title{Modular Meta-Learning with Shrinkage}

\author{  Yutian Chen\thanks{Equal contribution.}
  \And
  Abram L.~Friesen$^*$
  \And
  Feryal Behbahani
  \And
  Arnaud Doucet
  \And
  David Budden
  \And
  Matthew W.~Hoffman \\ \\
  DeepMind\\
  London, UK\\
  \texttt{\{yutianc, abef\}@google.com} \\
  \And
  Nando de Freitas
}

\begin{document}

\maketitle

\begin{abstract}
Many real-world problems, including multi-speaker text-to-speech synthesis, 
can greatly benefit from the ability to meta-learn large models with only a few task-specific components.
Updating only these task-specific modules then allows the model to be adapted to low-data tasks for as many steps as necessary without risking overfitting.
Unfortunately, existing meta-learning methods either do not scale to long adaptation or else
rely on handcrafted task-specific architectures.
Here, we propose a meta-learning approach that
obviates the 
need for this often sub-optimal hand-selection.
In particular, we 
develop general techniques based on Bayesian shrinkage to automatically discover and learn both task-specific and general reusable modules.
Empirically, we demonstrate that our method discovers a small set of meaningful task-specific modules and outperforms existing meta-learning approaches in domains like few-shot text-to-speech that have little task data and long adaptation horizons.
We also show that existing meta-learning methods including MAML, iMAML, and Reptile emerge as special cases of our method.
\looseness=-1

\end{abstract}

\section{Introduction}
\label{sec:intro}

The goal of meta-learning is to extract shared knowledge from a 
large set of training tasks to solve held-out tasks more efficiently.
One avenue for achieving this is
to learn task-agnostic modules and reuse or repurpose these for new tasks.
Reusing or repurposing modules can
reduce overfitting in low-data regimes, improve interpretability,
and facilitate the deployment of large multi-task models 
on limited-resource devices
as parameter sharing allows for significant savings in memory.
It can also enable batch evaluation of reused modules across tasks, which can speed up inference time on GPUs.

These considerations are important in domains like few-shot text-to-speech synthesis (TTS), characterized by large speaker-adaptable models, limited training data for speaker adaptation, and long adaptation horizons~\citep{chen2019sample}.
Adapting the model to a new task for more optimization steps generally improves the model capacity without increasing the number of parameters. However, many meta-learning methods are designed for quick adaptation, and hence are inapplicable in this 
\emph{few data and long adaptation} regime. 
For those that are applicable~\citep{Nichol2018,rajeswaran2019meta,flennerhag2019leap,flennerhag2019meta}, adapting the full model to few data can then fail because of overfitting.
To overcome this, modern TTS models combine shared core modules with handcrafted, adaptable, speaker-specific modules~\citep{arik2018neural,chen2019sample,jia2018transfer,taigman2018voiceloop}. 
This hard coding strategy is often suboptimal. As data increases, these hard-coded modules quickly become a bottleneck for further improvement, even in a few-shot regime. For this reason, we would like to automatically learn the smallest set of modules needed to
adapt to a new speaker and then adapt these for as long as needed.

Automatically learning reusable and broadly applicable modular mechanisms is an open challenge in causality, transfer learning, and domain adaptation \citep{peters2017elements,parascandolo2018learning,arjovsky2019invariant,bengio2020a}.
In meta-learning, most existing gradient-based algorithms, such as MAML \citep{Finn2017}, do not encourage meta-training to develop reusable and general modules, and either ignore reusability or manually choose the modules to fix~\citep{lee2018gradient,ParkMetaCurvature,flennerhag2019meta,zintgraf2018fast,raghu2019rapid,rusu2018metalearning}. Some methods implicitly learn a simple form of modularity 
for some datasets~
\citep{raghu2019rapid,arnold2019decoupling} 
but it is limited.
\looseness=-1

In this paper, we introduce a novel approach for automatically finding reusable modules. Our approach employs a principled hierarchical Bayesian model that exploits a statistical property known as shrinkage,
meaning that low-evidence estimates tend towards their prior mean; 
e.g., see \citet{gelman2013bayesian}.
This is accomplished by first partitioning any neural network into arbitrary groups of parameters, which we refer to as modules. We assign a Gaussian prior to each module with a scalar variance.
When the variance parameter shrinks to zero for a specific module, 
as it does if the data does not require the module parameters to deviate from
the prior mean, 
then all of the module's parameters become tied to the prior mean during task adaptation. 
This results in a set of automatically learned modules that can be reused at deployment time and a sparse set of remaining modules that
are adapted subject to the estimated prior.
\looseness=-1

Estimating the prior parameters in our model corresponds to meta-learning, and we present
two principled methods for this based on maximizing the predictive log-likelihood.
Importantly, both methods allow many adaptation steps.
By considering non-modular variants of our model, we show that
MAML~\citep{Finn2017}, Reptile~\citep{Nichol2018}, and iMAML~\citep{rajeswaran2019meta} emerge as special cases.
We compare our proposed shrinkage-based methods with their non-modular baselines on multiple low-data, long-adaptation domains, including
a challenging variant of Omniglot and TTS.
Our modular, shrinkage-based methods exhibit higher predictive power in low-data regimes
without sacrificing performance when more data is available.
Further, the discovered modular structures corroborate 
common knowledge about network structure
in computer vision and
provide new insights about WaveNet \cite{van2016wavenet} layers in TTS.
\looseness=-1

In summary, we introduce a hierarchical Bayesian model for modular meta-learning along with two 
parameter-estimation methods, which
we show generalize existing meta-learning algorithms.
We then demonstrate that our approach enables
identification of a small set of meaningful task-specific modules.
Finally, we show that our
method prevents overfitting and improves predictive performance
on problems that require many adaptation steps
given only small amounts of data.
\looseness=-1

\subsection{Related Work}
\label{sec:related}

Multiple Bayesian meta-learning approaches have been proposed to either provide model uncertainty in few-shot learning \citep{ravi2018amortized,neuralstatistician,garnelo2018conditional,gordon2019metalearning} or to provide a probabilistic interpretation and extend existing non-Bayesian works \citep{Grant2018,yoon2018bayesian,finn2018probabilistic}. However, to the best of our knowledge, none of these account for modular structure in their formulation. While we use point estimates of variables for computational reasons, more sophisticated inference methods from these works can also be used within our framework.

Modular meta-learning approaches based on MAML-style 
backpropagation through short task adaptation horizons
have also been proposed.
The most relevant of these, \citet{alet2018modular}, proposes to learn a modular network architecture, whereas our work identifies the adaptability of each module. 
In other work, \citet{zintgraf2018fast} hand-designs the task-specific and shared parameters, and
the M-Net in \citet{lee2018gradient} provides an alternative method for learning adaptable modules by sampling binary mask variables.
In all of the above, however, backpropagating through task adaptation is computationally prohibitive when applied to problems that require longer adaptation horizons.
While it is worth investigating how to extend these works to this setting, we leave this for future work.
\looseness=-1

Many other meta-learning works \citep{lee2018gradient,ParkMetaCurvature,flennerhag2019meta,li2017meta,Lee2020Learning,khodak2019adaptive} learn the learning rate (or a preconditioning matrix) which can have a similar modular regularization effect as our approach if applied modularly with fixed adaptation steps. However, these approaches and ours pursue fundamentally different and orthogonal goals. Learning the learning rate aims at fast optimization by fitting to local curvature, without changing the task loss or associated stationary point.
In contrast, our method learns to control how far
each module will move from its initialization by changing
that stationary point.
In problems requiring long task adaptation, these two approaches lead to different behaviors, as we demonstrate in \cref{sec:synthetic_exps}. Further, most of these works also rely on backpropagating through gradients, which does not scale well to the long adaptation horizons considered in this work. 
Overall, generalizing beyond the training horizon is challenging for 
``learning-to-optimize'' approaches ~\citep{andrychowicz2016learning,wu2018understanding}. 
While WarpGrad \citep{flennerhag2019meta} does allow long adaptation horizons, it is not straightforward to apply its use of a functional mapping from inputs to a preconditioning matrix for module discovery and we leave this for future work.\looseness=-1

Finally, L2 regularization is also used for task adaptation in continual learning \citep{kirkpatrick2017overcoming,zenke2017continual} 
and meta-learning \citep{denevi2019learning,NIPS2018_8220,zhou2019efficient,rajeswaran2019meta}. 
However, in these works, the regularization scale(s) are either treated as hyper-parameters
or adapted per dimension with a different criterion from our approach.
It is not clear how to learn the regularization scale at a module level in these works.
\looseness=-1

\vspace{-.5em}

\section{Gradient-based Meta-Learning}
\label{sec:background}

\vspace{-.5em}

We begin with a general overview of gradient-based meta-learning, as this is one of the most common approaches to meta-learning. In this regime, we assume that there are many tasks, indexed by $t$, and that each of these tasks has few data. That is,
each task is associated with a finite dataset $\D_t=\{\x_{t,n}\}$ of size $N_t$, which can be partitioned into training and validation sets, $\D_t^\tr$ and $\D_t^\val$ respectively.
~To solve a task, gradient-based meta-learning adapts \emph{task-specific parameters} $\vtaskv_t \in \real^D$ by minimizing a loss function $\ell(\D_t;\vtaskv_t)$ using a local optimizer.
Adaptation is made more efficient by sharing a set of \emph{meta parameters} $\vmetav\in\real^D$ 
between tasks, which are typically used to initialize the task parameters.

Algorithm 1 summarizes a typical stochastic meta-training procedure, which
includes MAML~\citep{Finn2017}, implicit MAML (iMAML)~\citep{rajeswaran2019meta}, and Reptile~\citep{Nichol2018}.
Here, \textsc{TaskAdapt} executes one step of optimization of the task parameters.
The meta-update $\Delta_t$ specifies the contribution of task $t$ to the meta parameters.
At test time, multiple steps of \textsc{TaskAdapt} are run on each new test task.

\begin{figure}
  \centering
  \begin{minipage}{.46\textwidth}
    \centering
{\footnotesize    
    \begin{algorithm}[H]
    \SetAlgoLined
    \KwIn{Batch size $B$, steps $L$, and learning rate $\alpha$}
    Initialize $\vmetav$ \;
    \While{not done}{
      $\{t_1, \dots, t_B\} \leftarrow$ sample mini-batch of tasks \;
      \For{each task $t$ in $\{t_1, \dots, t_B\}$} {
	    Initialize $\vtaskv_t \leftarrow \vphi$\;
    		\For{step $l = 1 \dots L$} {
      	    $\vtaskv_t \leftarrow \textsc{TaskAdapt}(\D_t, \vmetav, \vtaskv_t)$
    	    }
  	  }
  	  \tcp*[l]{Meta update}
 	  $\vmetav \leftarrow \vmetav - \alpha \cdot \frac{1}{B} \sum_{t} \Delta_t(\D_t, \vmetav, \vtaskv_t)$  \;
  	  
  	}
  	\caption{Meta-learning pseudocode.}
    \label{alg:meta-learning}
    \end{algorithm} 
}
  \end{minipage}
  ~
  \begin{minipage}{.52\textwidth}
    \centering
    \includegraphics[width=0.78\columnwidth]{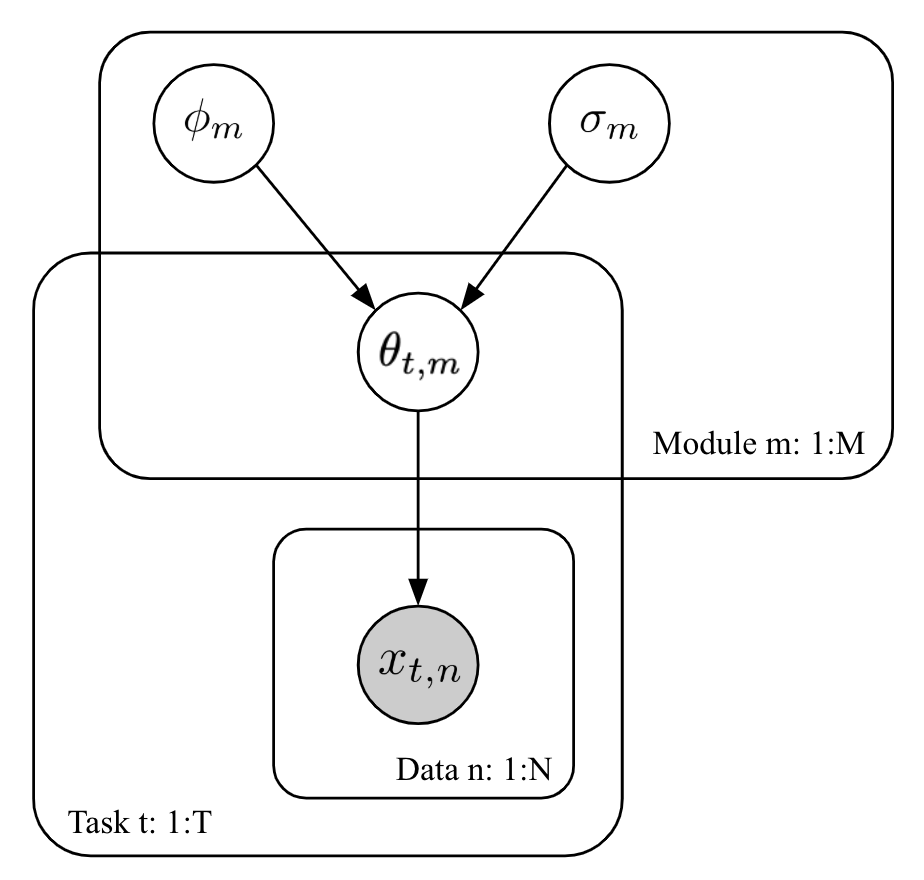}
  \end{minipage}
      \caption{
      (Left) Structure of a typical meta-learning algorithm. 
      (Right) Bayesian shrinkage graphical model. 
      The shared meta parameters $\vphi$ serve as the initialization
      of the neural network parameters for each task $\vtheta_t$.
          The $\vsigma$ are shrinkage parameters. By learning these, the model automatically decides which subsets of parameters (i.e., modules) to fix for all tasks and which to adapt at test time.}
        \label{fig:graphical_model}
\end{figure}

MAML implements task adaptation by applying gradient descent to minimize the training loss $\elltr(\vtaskv_t)=\ell(\D^\tr_t;\vtaskv_t)$ with respect to the task parameters.
It then updates the meta parameters by gradient descent on the validation loss $\ellval(\vtaskv_t)=\ell(\D^\val_t;\vtaskv_t)$, resulting in the meta update
$\Delta_t^{\text{MAML}} = \nabla_{\vmetav}\, \ellval(\vtaskv_t(\vmetav)).$
This approach treats the task parameters as a function of the meta parameters, and hence 
requires backpropagating through the entire $L$-step task adaptation process.
When $L$ is large, as in TTS systems~\citep{chen2019sample}, 
this becomes computationally prohibitive.

Reptile and iMAML avoid this computational burden of MAML. Reptile instead optimizes $\vtaskv_t$ on the entire dataset $\D_t$, and moves $\vmetav$ towards the adapted task parameters, yielding
$\Delta_t^{\text{Reptile}} = \vmetav - \vtaskv_t.$
Conversely, iMAML introduces an L2 regularizer
$\frac{\lambda}{2} ||\vtaskv_t - \vmetav||^2$ and
optimizes the task parameters on the regularized training loss.
Provided that this task adaptation process converges to a stationary point, 
\emph{implicit differentiation} enables the computation of
the meta gradient based only on the final solution of the
adaptation process,
$\Delta_t^{\text{iMAML}} = 
\big(\vI + \tfrac{1}{\lambda} \nabla^2_{\vtaskv_t}\elltr(\vtaskv_t) 
\big)^{-1} \nabla_{\vtaskv_t} \ellval(\vtaskv_t).$
See \citet{rajeswaran2019meta} for details.

\vspace{-.5em}

\section{Modular Bayesian Meta-Learning}
\label{sec:model}

\vspace{-.5em}

In standard meta-learning, the meta parameters $\vphi$ provide an initialization for the task parameters $\vtheta$ at test time. That is, all the neural network parameters are treated equally, and hence they must all be updated at test time. This strategy is inefficient and prone to overfitting. To overcome it, researchers often split the network parameters into two groups, a group that varies across tasks and a group that is shared; see for example \cite{chen2019sample,zintgraf2018fast}. This division is heuristic, so in this work we explore ways of automating it to achieve better results and to enable automatic discovery of task independent modules.  
More precisely, we assume that the network 
parameters can be partitioned into $M$ disjoint modules
$\vtaskv_t = (\vtaskv_{t,1}, \ldots, \vtaskv_{t,m}, \ldots, \vtaskv_{t,M})$ where 
$\vtaskv_{t,m}$
are the parameters in module $m$ for task $t$. This view of modules is very general. Modules can correspond to layers, receptive fields, the encoder and decoder in an auto-encoder, the heads in a multi-task learning model, or any other grouping of interest.

We adopt a hierarchical Bayesian model, shown in Figure~\ref{fig:graphical_model}, with a factored probability density:
\be
p(\vtaskv_{1:T},\D|\vsigma^2,\vmetav) 
= 
\prod_{t=1}^{T} \prod_{m=1}^M \gauss(\vtaskv_{t,m} | \vmetav_m,\sigma_m^2\vI)
\prod_{t=1}^{T} p(\D_t|\vtaskv_t).
\label{eqn:loss}
\ee
The $\vtaskv_{t,m}$ are conditionally independent and normally distributed
$\vtaskv_{t,m}\sim \gauss(\vtaskv_{t,m} | \vmetav_m,\sigma_m^2\vI)$
with mean $\vmetav_{m}$ and variance $\sigma_m^2$, where $\vI$ is the identity matrix.

The $m$-th module scalar shrinkage parameter $\sigma_m^2$ measures 
the degree to which $\vtaskv_{t,m}$ can deviate from $\vmetav_m$. More precisely, for values of $\sigma_m^2$ near zero, the difference between parameters $\vtaskv_{t,m}$ and mean $\vmetav_m$ will be shrunk to zero and thus module $m$ will become task independent.
Thus by learning the parameters $\vsigma^2$, we discover which modules are task independent. These independent modules can be reused at test time, reducing the computational burden of adaptation and likely improving generalization.
Shrinkage is often used in automatic relevance determination for sparse feature selection~\citep{tipping2001sparse}.
\looseness=-1

We place uninformative priors on $\vmetav_m$ and $\sigma_m$, and follow an empirical Bayes approach to 
learn their values from data.
This formulation allows the model to automatically learn which modules to reuse---i.e.\ those modules for which $\sigma_m^2$ is near zero---and which to adapt at test time.

\section{Meta-Learning as Parameter Estimation}
\label{sec:learning}

By adopting the hierarchical Bayesian model from the previous section, the problem of 
meta-learning becomes one of estimating the parameters $\vmetav$ and $\vsigma^2$. 
A standard solution to this problem is to maximize the marginal likelihood $p(\D|\vmetav,\vsigma^2)=\int p(\D|\vtaskv)p(\vtaskv|\vmetav,\vsigma^2)\,\textrm{d}\vtaskv$. 
We can also assign a prior over $\vmetav$. 
In both cases, the marginalizations are intractable, so we must seek scalable approximations.

It may be tempting to estimate the parameters by maximizing $p(\vtaskv_{1:T},\D|\vsigma^2,\vmetav)$ w.r.t. $(\vsigma^2,\vphi,\vtheta_{1:T})$, but the following lemma suggests that this naive approach leads to all task parameters being tied to the prior mean, i.e. no adaptation will occur (see \cref{sec:analysis} for a proof):
\begin{lemma}\label{lemma:jointmaximizationisbad}
The function $f: (\vsigma^2,\vphi,\vtheta_{1:T}) \mapsto \log p(\vtaskv_{1:T},\D|\vsigma^2,\vmetav)$ diverges to $+\infty$ as 
$\vsigma^2 \rightarrow 0^{+}$ when $\vtheta_{t,m}=\vmetav_m$ for all $t\in\{ 1,...,T \},m\in\{1,...,M \}$.
\end{lemma}

The undesirable result of \cref{lemma:jointmaximizationisbad} is caused by the use of point estimate of $\vtheta_{1:T}$ in maximum likelihood estimation of $\vsigma^2$ and $\vphi$.
In the following two subsections, we propose two principled alternative approaches for parameter estimation based on maximizing the predictive likelihood over validation subsets.

\subsection{Parameter estimation via the predictive likelihood}
\label{sec:alg_pll}
Our goal is to minimize the average negative predictive log-likelihood over $T$ validation tasks, 
\begin{align}
    \ellpll(\vsigma^2\hspace{-1mm},\vmetav)
    &\hspace{-1mm}=-\frac{1}{T}\hspace{-1mm}\sum_{t=1}^T \hspace{-.4mm}\log p\left(\D_{t}^\val | \D_{t}^\tr \hspace{-1mm}, \vsigma^2\hspace{-1mm},\vmetav\right) 
        \hspace{-1mm}=\frac{1}{T}\hspace{-1mm}\sum_{t=1}^T \hspace{-.4mm}\log \hspace{-1mm} \int \hspace{-1mm} p(\D_{t}^\val|\vtheta_{t}) \,p(\vtheta_{t}|\D_{t}^\tr\hspace{-1mm},\vsigma^2\hspace{-1mm},\vmetav) \,\textrm{d}\vtheta_{t}.
    \label{eq:predictivedistribution}
\end{align}
To justify this goal, assume that the training and validation data is distributed i.i.d\ according to some distribution $\nu(\D_{t}^\tr,\D_{t}^\val)$. Then, the law of large numbers implies that
as $T\rightarrow \infty$,
\eq{\label{eq:reinterpretationKL}
\ellpll(\vsigma^2,\vmetav)\rightarrow
\mathbb{E}_{\nu(\D_{t}^\tr)}[\textrm{KL}(\nu(\D_{t}^\val|\D_{t}^\tr)||p\left(\D_{t}^\val | \D_{t}^\tr, \vsigma^2, \vmetav \right))]+\textrm{H}(\nu(\D_{t}^\val|\D_{t}^\tr)),
}
where $\text{KL}$ denotes the Kullback-Leibler divergence and $\textrm{H}$ the entropy. Thus minimizing $\ellpll$ with respect to the meta parameters corresponds to selecting the predictive distribution $p\left(\D_{t}^\val | \D_{t}^\tr, \vsigma^2,\vmetav\right)$ that is closest (approximately) to the true predictive distribution $\nu(\D_{t}^\val|\D_{t}^\tr)$ on average. 
This criterion can be thought of as an approximation to a Bayesian cross-validation criterion \cite{fong2020marginal}.

Computing $\ellpll$ is not feasible due to the intractable integral in equation~(\ref{eq:predictivedistribution}).
Instead we make use of a simple maximum a posteriori (MAP) approximation of the task parameters:
\begin{align}
    \hat{\vtheta}_{t}(\vsigma^2, \vmetav)
    =
    \argmin\limits_{\vtheta_{t}} \elltr(\vtheta_{t}, \vsigma^2, \vmetav), ~\text{where }
    \elltr
    \coloneqq
    -\log p\left(\D_{t}^\tr|\vtheta_{t}\right)
    - \log p\left(\vtheta_{t}|\vsigma^2,\vmetav\right).
    \label{eq:map_v1}
\end{align}
We note for clarity that $\elltr$ corresponds to the negative log 
of equation~(\ref{eqn:loss}) for a single task.
Using these MAP estimates, we can approximate $\ellpll(\vsigma^2,\vmetav)$ as follows:
\begin{align}
    \hat{\ell}_\pll(\vsigma^2, \vphi) = \frac{1}{T}\sum_{t=1}^T \ellval(\hat{\vtheta}_{t}(\vsigma^2, \vphi)),\;\; ~\text{ where }
    \ellval
    \coloneqq
    -\log p\big(\D_{t}^\val | \hat{\vtaskv}_t\big).
    \label{eq:train-objective_v1}
\end{align}
We use this loss to construct a meta-learning algorithm with the same structure as Algorithm 1. Individual task adaptation follows from equation~(\ref{eq:map_v1}) 
and meta updating from minimizing $\hat{\ell}_\pll(\vsigma^2, \vphi)$ in equation~(\ref{eq:train-objective_v1}) 
with an unbiased gradient estimator and a mini-batch of sampled tasks.

Minimizing equation~(\ref{eq:train-objective_v1}) is a bi-level optimization problem that requires solving equation~(\ref{eq:map_v1})  implicitly. 
If optimizing $\elltr$ requires only a small number of local optimization steps, we can compute the update for $\vphi$ and $\vsigma^2$ with back-propagation through $\hat{\vtheta}_t$, yielding
\eq{
\Delta_t^{\text{\shrmaml{}}} = 
\nabla_{\vsigma^2, \vmetav}\, \ell_t^\val(\hat{\vtaskv}_t(\vsigma^2, \vmetav)).
\label{eq:s-maml}
}
This update reduces to that of MAML if $\sigma_m^2 \rightarrow \infty$ for all modules and is thus denoted as \shrmaml.

We are however more interested in long adaptation horizons for
which
back-propagation through the adaptation becomes computationally expensive and numerically unstable. Instead, we apply the \emph{implicit function theorem} on equation~(\ref{eq:map_v1})
to compute the gradient of $\hat{\vtaskv}_t$ with respect to $\vsigma^2$ and $\vmetav$, giving \looseness=-1
\eq{
\Delta_t^{\text{\shrimaml{}}}
= - \nabla_{\vtaskv_t} \ell_t^\val(\vtaskv_t) \vH_{\vtheta_t \vtheta_t}^{-1} \vH_{\vtheta_t \vPhi},
\label{eq:valid_gd_v1}
}
where $\vPhi=(\vsigma^2, \vmetav)$ is the full vector of meta-parameters, $\vH_{a b} = \nabla_{a, b}^2 \elltr$, and derivatives are evaluated at the stationary point $\vtheta_t=\hat\vtaskv_t(\vsigma^2,\vmetav)$. A detailed derivation is provided in \cref{sec:valid_gd_derivation_v1}. 
Various approaches have been proposed to approximate the inverse Hessian \citep{rajeswaran2019meta,bengio2000gradient,pedregosa2016hyperparameter,lorraine2019optimizing}. 
We use the conjugate gradient algorithm. 
We show in \cref{sec:equiv_to_imaml} that the meta update for $\vmetav$ is equivalent to that of iMAML when $\sigma_m^2$ is constant for all $m$, and thus refer to this more general method as \shrimaml.
\looseness=-1

In summary, our goal of maximizing the predictive likelihood of validation data conditioned on training data for our hierarchical Bayesian model results in modular 
generalizations of the MAML and iMAML approaches. In the following subsection, we will see that 
the same is also true for Reptile.
\looseness=-1

\subsection{Estimating \texorpdfstring{$\vmetav$}{phi} via MAP approximation}
\label{sec:alg_map}
If we instead let $\vmetav$ be a random variable with prior distribution $p(\vmetav)$, 
then we can derive a different variant of the above algorithm.
We return to the predictive log-likelihood introduced in the previous section but 
now integrate out the central parameter $\vmetav$. Since $\vmetav$ depends
on all training data we will rewrite the average predictive likelihood in terms of the joint posterior
over $(\vtheta_{1:T},\vmetav)$, i.e.
\begin{align}
    \ellpll(\vsigma^2) 
    &\hspace{-.5mm}=\hspace{-.5mm}-\frac{1}{T}\log p(\D_{1:T}^\val|\D_{1:T}^\tr,\vsigma^2)
    \hspace{-.5mm}=\hspace{-.5mm}-\frac{1}{T}\log \hspace{-1mm} \int \hspace{-1.5mm}
    \Big(\hspace{-.5mm}\prod_{t=1}^T p(\D_{t}^\val|\vtheta_t)\hspace{-.5mm}\Big)
    p(\vtheta_{1:T},\vmetav|\D_{1:T}^\tr, \vsigma^2)
    \textrm{d}\vtheta_{1:T}
    \textrm{d}\vmetav.
    \label{eq:log_p_sigma2_full}
\end{align}
Again, to address the intractable integral we make use of
a MAP approximation, except in this case we approximate both the task parameters and
the central meta parameter as
\begin{align}
    &\hspace{-1em}(\hat{\vtheta}_{1:T}(\vsigma^2), \hat{\vmetav}(\vsigma^2))     \hspace{-.5mm}= \argmin \limits_{\vtheta_{1:T},\hspace{-.5mm}\vmetav} \,
    \hspace{-1mm}-\hspace{-.5mm}\log p(\vtheta_{1:T},\vmetav|\D_{1:T}^\tr, \vsigma^2)     \hspace{-.5mm}= \argmin \limits_{\vtheta_{1:T},\hspace{-.5mm}\vmetav}
        \hspace{-1mm} \sum_{t=1}^T \hspace{-.5mm}\elltr(\vtaskv_t, \vsigma^2, \hspace{-.5mm}\vmetav) \hspace{-.5mm}-
            \hspace{-.5mm}\log p(\vmetav).
    \label{eq:log_p_train_joint}
\end{align}
We assume a flat prior for $\vmetav$ in this paper and thus drop the second term above. Note that when the number of training tasks $T$ is small, an informative prior would be preferred. We can then estimate the shrinkage parameter $\vsigma^2$ by plugging this approximation into \cref{eq:log_p_sigma2_full}. This method gives the same task adaptation update as the previous section but a meta update of the form
\eq{
\Delta_{t,\vmetav_m}^{\text{\shrreptile{}}} 
= 
\frac{1}{\sigma_m^2}(\vmetav_m - \hat{\vtaskv}_{m,t}), ~~
\Delta_{t, \vsigma^2}^{\text{\shrreptile{}}}
= 
- \nabla_{\vtaskv_t} \ell_t^\val(\vtaskv_t) \vH_{\vtheta_t \vtheta_t}^{-1} \vH_{\vtheta_t \vsigma^2} \,,
\label{eq:valid_gd_v2}
}
where derivatives are evaluated at $\vtheta_t = \hat\vtheta_t(\vsigma^2)$, and the gradient of $\hat{\vmetav}$ w.r.t.\ $\vsigma^2$ is ignored. Due to lack of space, 
further justification and derivation of this approach
is provided in \cref{sec_jointmap,sec:valid_gd_derivation_v2}.
\looseness=-1

We can see that Reptile is a special case of this method when $\sigma_m^2 \rightarrow \infty$ 
and we choose a learning rate proportional to $\sigma_m^2$ for $\vmetav_m$.
We thus refer to it as \shrreptile{}.

\cref{tab:algs} compares our proposed algorithms with existing algorithms in the literature. 
Our three algorithms reduce to the algorithms on the left when $\sigma_m^2 \rightarrow \infty$ or a constant scalar for all modules. 
Another variant of MAML for long adaptation, first-order MAML, can be recovered as a special case of iMAML when using one step of conjugate gradient descent to compute the inverse Hessian~\citep{rajeswaran2019meta}.

\begin{table}[tb]
{\footnotesize
\begin{center}
\begin{tabular}{c cll|c}
\multicolumn{2}{c}{}
& Fixed $\vsigma^2$ & Learned $\vsigma^2$
& Allows long adaptation?\\
\toprule
\multicolumn{2}{l}{$\hat{\vphi}_\pos$} & Reptile & \shrreptile & \checkmark \\
\midrule
\multirow{2}{*}{$\hat{\vphi}_\pll$}
&
Back-prop. & MAML & \shrmaml & $\times$ \\
& 
Implicit grad. & iMAML & \shrimaml & \checkmark
\\
\bottomrule
\end{tabular}
\end{center}
}
\caption {The above algorithms result from different approximations to the predictive likelihood.} 
\label{tab:algs}
\vspace{-1em}
\end{table}

\subsection{Task-Specific Module Selection}

When the parameters $\vmetav$ and $\vsigma^2$ are estimated accurately, the values of $\sigma_m^2$ for task-independent modules shrink towards zero. The remaining modules with non-zero $\sigma_m^2$ are considered task-specific~\citep{tipping2001sparse}. In practice, however, the estimated value of $\sigma_m^2$ will never be exactly zero due to the use of stochastic optimization and the approximation in estimating the meta gradients. We therefore apply a weak regularization on $\vsigma^2$ to encourage its components to be small unless nonzero values are supported by evidence from the task data 
(see \cref{sec:alg_details} for details).

While the learned $\sigma_m^2$ values are still non-zero, 
in most of our experiments below, the estimated value $\sigma_m^2$ for task-independent modules is either numerically indistinguishable from $0$ or at least two orders of magnitude smaller than the value for the task-specific modules.
This reduces the gap between meta-training, where all modules have non-zero $\sigma_m^2$, and meta-testing, where a sparse set of modules are selected for adaptation.
The exception to this is the text-to-speech experiment where we find the gap of $\sigma_m^2$ between modules to not be as large as in the image experiments. Thus, we instead rank the modules by the value of $\sigma_m^2$ and select the top-ranked modules as task-specific.

Ranking and selecting task-specific modules using the estimated value of $\sigma_m^2$ allows us to trade off between module sparsity and model capacity in practice, and achieves robustness against overfitting. It remains unclear in theory, however, if the task sensitivity of a module is always positively correlated with $\sigma_m^2$ especially when the size of modules varies widely. 
This is an interesting question that we leave for future investigation.

\section{Experimental Evaluation}
\label{sec:experiments}

We evaluate our shrinkage-based methods on challenging meta-learning
domains that have small amounts of data
and require long adaptation horizons, 
such as few-shot text-to-speech voice synthesis.
The aim of our evaluation is to answer the following three questions:
(1) Does shrinkage enable automatic discovery of a small set of task-specific modules? 
(2) Can we adapt only the task-specific modules without sacrificing performance?
(3) Does incorporating a shrinkage prior improve performance and robustness to overfitting in problems with little data and long adaptation horizons?

\subsection{Experiment setup}
\label{sec:exp_setup}

Our focus is on long adaptation, low data regimes.
To this end, 
we compare iMAML and Reptile to their corresponding
shrinkage variants,  \shrimaml{} and \shrreptile{}.
For task adaptation with the shrinkage variants, 
we use proximal gradient descent~\citep{singer2009proxgd} for the image experiments, and introduce a proximal variant of Adam~\citep{kingma2014adam} (pseudocode in \cref{alg:prox-adam}) for the text-to-speech (TTS) experiments. The proximal methods provide robustness to changes in the prior strength $\vsigma^2$ over time.
We provide further algorithmic details in \cref{sec:alg_details}. 
We evaluate on the following domains.

\textbf{Few-shot image classification.}
We use the augmented Omniglot protocol of \citet{flennerhag2019leap},
which necessitates long-horizon adaptation.
For each alphabet,
$20$ characters are sampled 
to define a $20$-class classification problem.
The domain is challenging because both train and test images are 
randomly augmented. Following \citet{flennerhag2019leap}, we use a 4-layer convnet
and perform
$100$ steps of task adaptation.
We consider two regimes:
(\textbf{Large-data regime})
We use $30$ training alphabets ($T=30$), $15$ training images ($K=15$),
and $5$ validation images per class.
Each image is randomly re-scaled, translated, and rotated.
(\textbf{Small-data regime})
To study the effect of overfitting,
we vary 
$T\in\{5,10,15,20\}$
and 
$K\in\{1,3,5,10,15\}$,
and augment only by scaling and translating. 
\looseness=-1

\textbf{Text-to-speech voice synthesis.}
Training a neural TTS model from scratch typically requires tens of hours of speech data. In the few-shot learning setting \citep{arik2018neural,chen2019sample,jia2018transfer,taigman2018voiceloop}, the goal is to adapt a trained model to a new speaker based on only a few minutes of data.
Earlier work unsuccessfully applied fast-adaptation methods such as MAML to synthesizing utterances~\citep{chen2019sample}. Instead, their state-of-the-art method first pretrains a
multi-speaker model comprised of a shared core network and a speaker embedding module and then finetunes either the entire model or the embedding only.
We remove the manually-designed speaker embedding layers and perform task adaptation and meta-updates on only the core network. 

The core network is a WaveNet vocoder model \citep{van2016wavenet} with $30$ residual causal dilated convolutional blocks as the backbone, consisting of roughly 3M parameters. For computational reasons, we use only one quarter of the channels of the standard network. As a result, sample quality does not reach production level but we expect the comparative results to apply to the full network. We meta-train with tasks of $100$ training utterances (about $8$ minutes) using 100 task adaptation steps, then evaluate on held-out speakers with either $100$ or $50$ (about $4$ minutes) utterances and up to $10,000$ adaptation steps. For more details see \cref{sec:apx_tts}.
\looseness=-1

\textbf{Short adaptation.}
While our focus is on long adaptation, we conduct experiments on short-adaptation datasets (sinusoid regression, standard Omniglot, and \textit{mini}ImageNet) for completeness
in \cref{sec:apx_experiment}.

\subsection{Module discovery}
\label{sec:exp_module_discovery}

To determine whether shrinkage discovers task-specific modules, 
we examine the learned prior strengths of each module and then adapt individual (or a subset of the) modules to assess the effect on performance due to adapting the chosen modules.
In these experiments, we treat each layer as a module but other choices of modules are straightforward.

\begin{figure}[t!]
\centering
\hspace{-0.1em}
\begin{subfigure}{.24\textwidth}
\includegraphics[width=\textwidth]{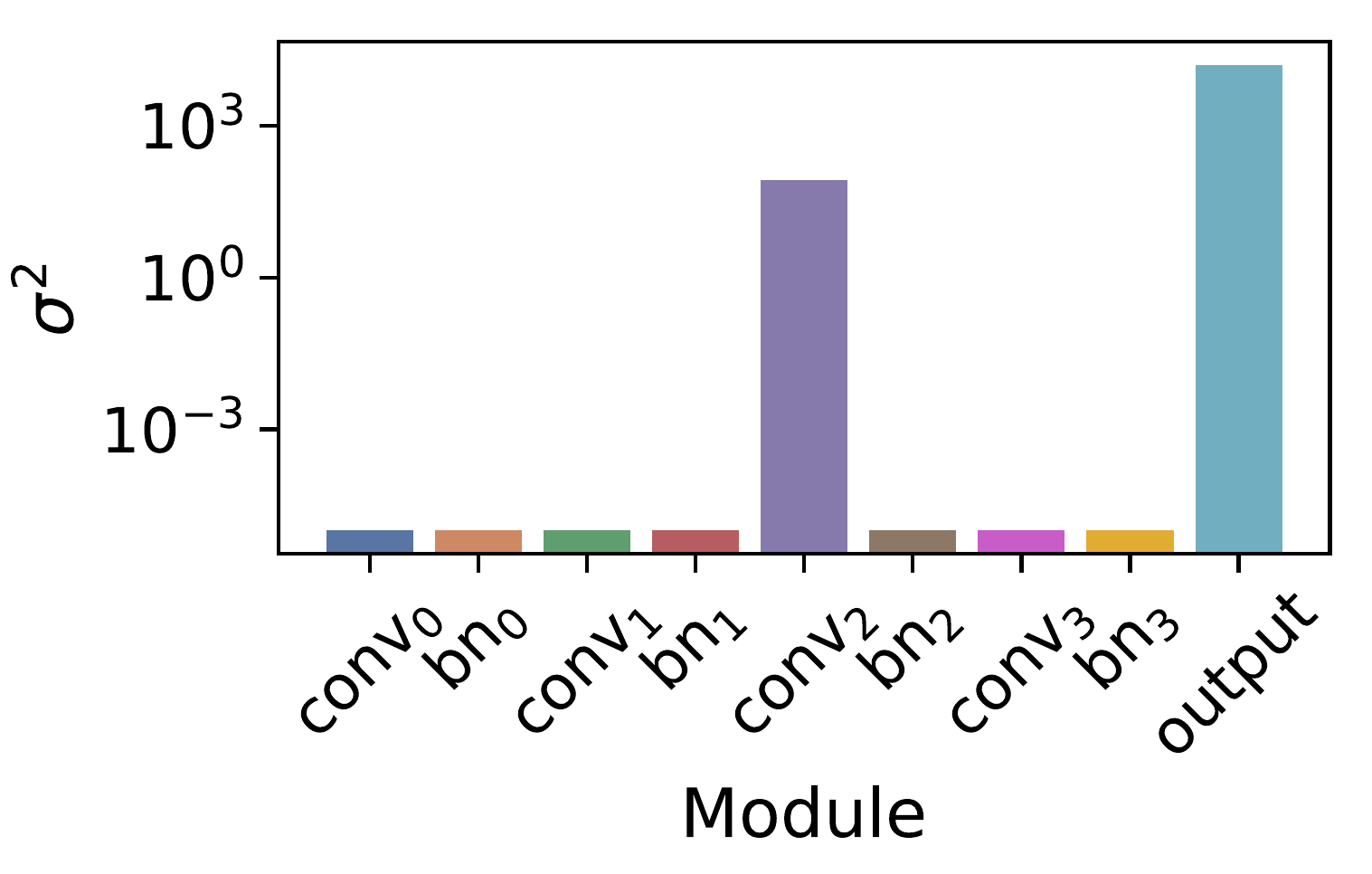}
\caption{$\vsigma^2$ for \shrimaml{}.}
\label{fig:aug_omniglot_shrinkage_imaml_var}
\end{subfigure}
\begin{subfigure}{.24\textwidth}
\includegraphics[width=\textwidth]{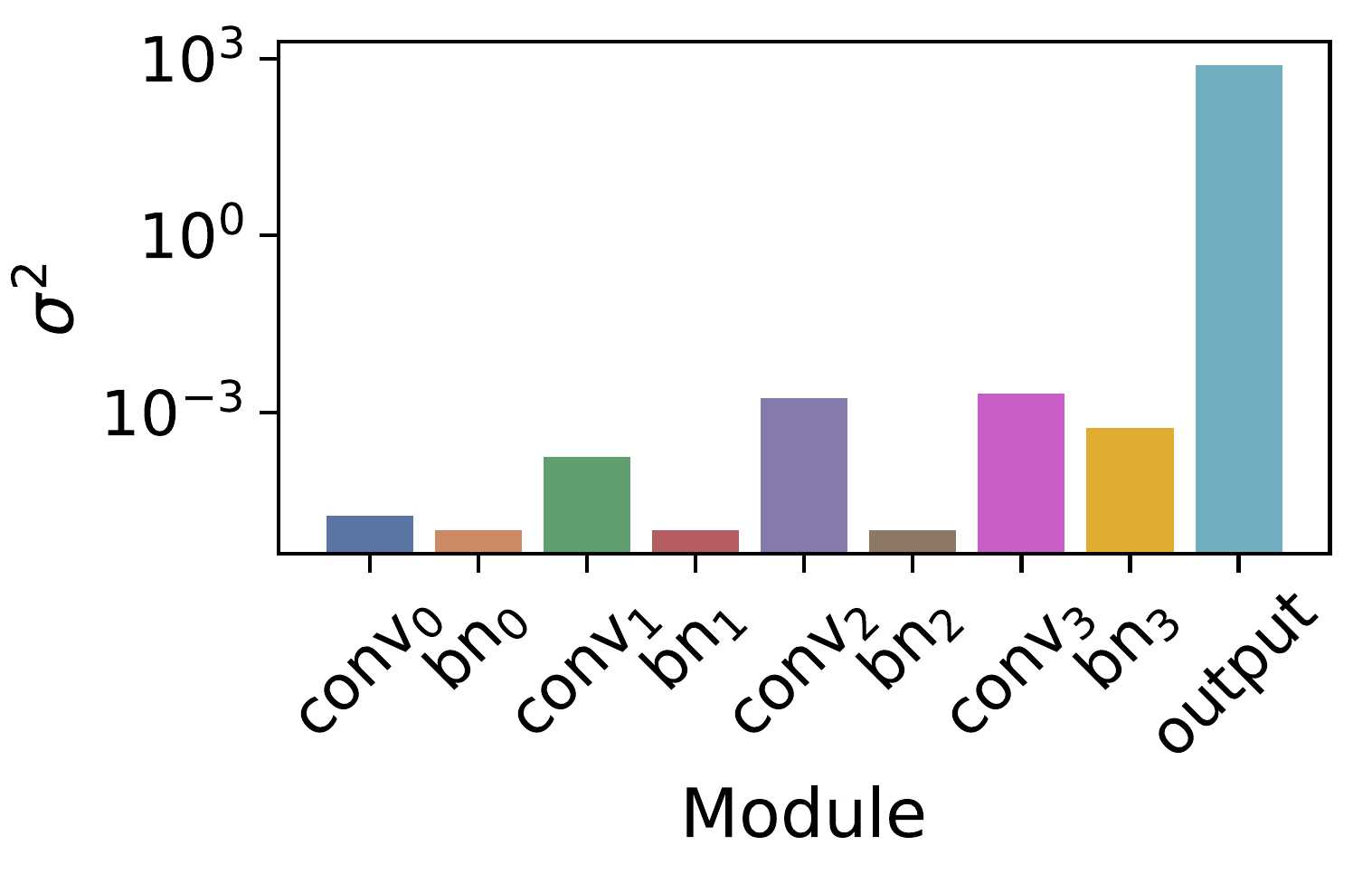}
\caption{$\vsigma^2$ for \shrreptile{}.}
\label{fig:aug_omniglot_shrinkage_reptile_var}
\end{subfigure}
\begin{subfigure}{.24\textwidth}
\includegraphics[width=\textwidth]{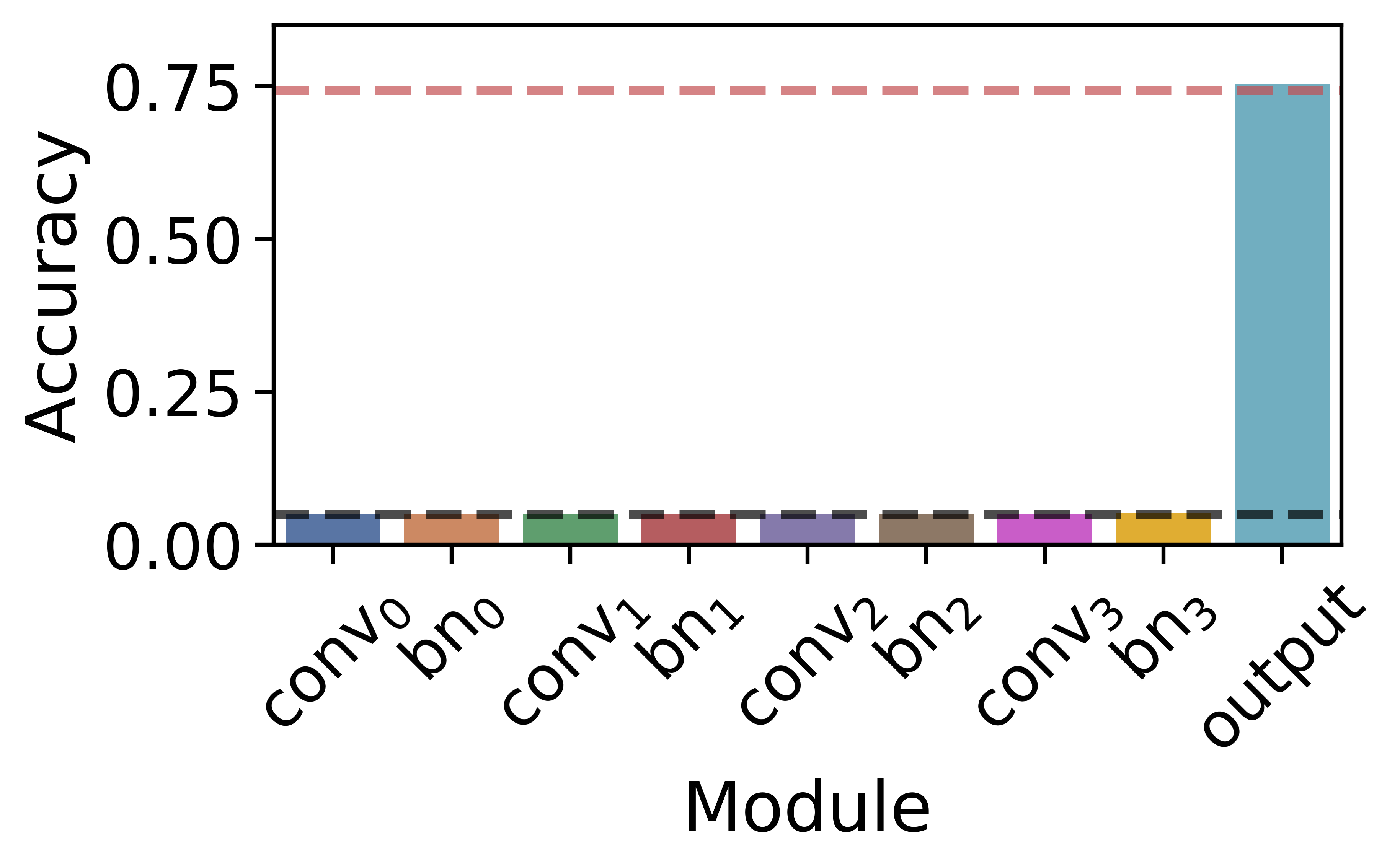}
\caption{\shrimaml{} accuracy.}
\label{fig:aug_omniglot_shrinkage_imaml_acc}
\end{subfigure}
\begin{subfigure}{.24\textwidth}
\includegraphics[width=\textwidth]{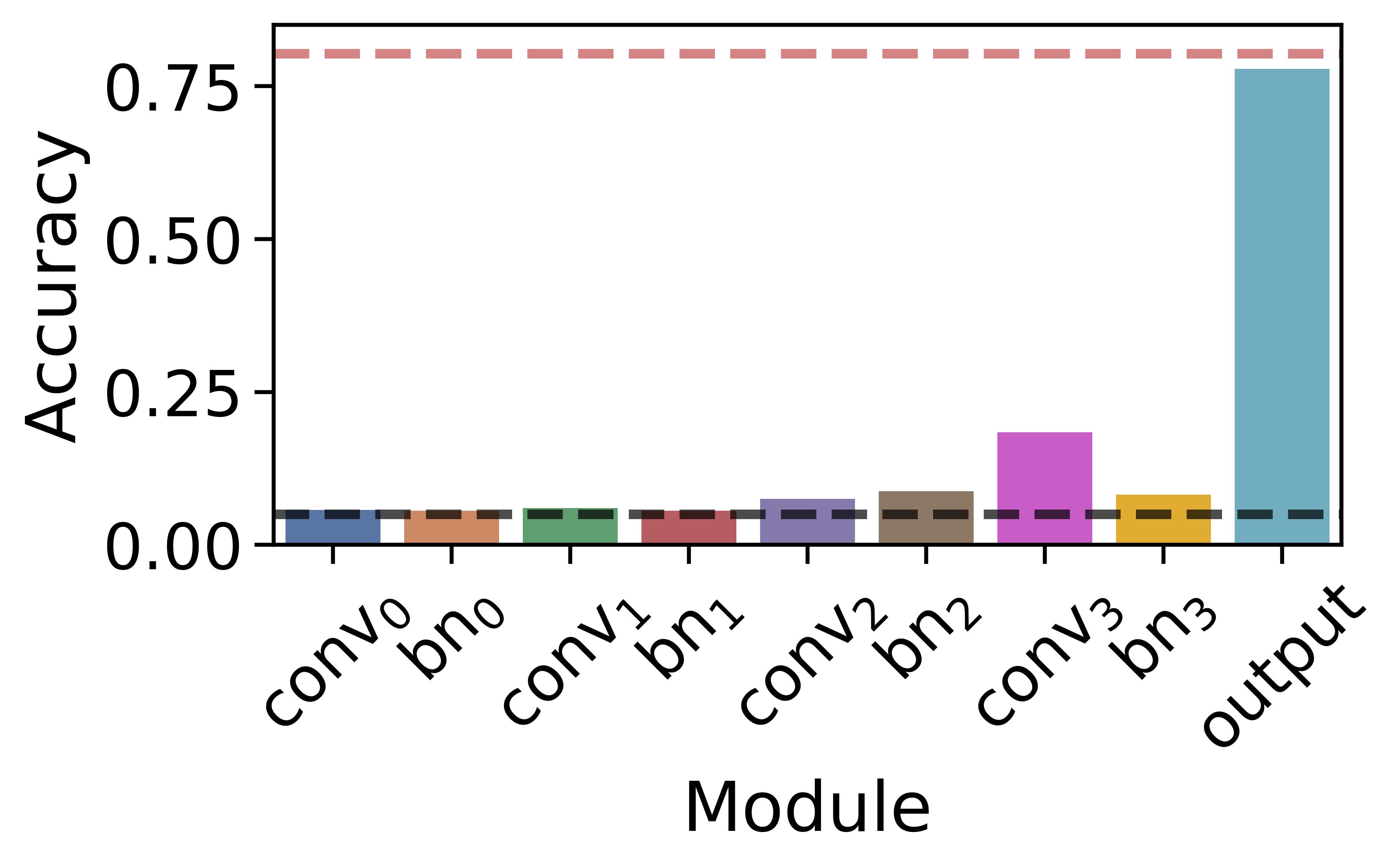}
\caption{\shrreptile{} accuracy.}
\label{fig:aug_omniglot_shrinkage_reptile_acc}
\end{subfigure}
\caption{
Module discovery with \shrimaml{} and \shrreptile{}
for large-data augmented Omniglot.
(a,b) The learned $\vsigma^2$ for each module (y-axis is log scale).
(c,d) Test accuracy when only that layer is adapted versus when all layers 
are adapted using the learned $\vsigma^2$.
}
\label{fig:module_aug_omniglot}
\end{figure}

\textbf{Image classification.}
\cref{fig:module_aug_omniglot}
shows our module discovery results on large-data augmented Omniglot 
for \shrimaml{} and \shrreptile{} using the standard network (4 conv layers, 4 batch-norm layers, and a linear output layer).
In each case, the learned $\vsigma^2$ (\cref{fig:module_aug_omniglot}(a,b)) 
of the output layer is
considerably larger than the others.
\cref{fig:module_aug_omniglot}(c,d) show that
the model achieves high accuracy when adapting only this shrinkage-identified module, giving comparable performance to that achieved by adapting all layers according to $\vsigma^2$.
This corroborates the conventional belief that the output layers of image classification
networks are more task-specific while the input layers are more general, and bolsters other meta-learning studies~\citep{raghu2019rapid,gordon2019metalearning}
that propose to adapt only the output layer. 
\looseness=-1

However, the full story is not so clear cut. Our module discovery results (\cref{sec:apx_experiment}) on standard few-shot short-adaptation image classification show that in those domains
adapting the \textit{penultimate} layer is best, which matches 
an observation in \citet{arnold2019decoupling}. 
Further, on sinusoid regression, adapting
the \textit{first} layer performed best.
Thus, there is no single best modular structure across domains.

\textbf{Text-to-speech.}
\cref{fig:tts_module} shows the learned $\sigma^2$ for each layer of our TTS WaveNet model,
which consists of 4 layers per residual block and 123 layers in total (\cref{sec:apx_tts} shows the full architecture).
Most $\sigma^2$ values 
are too small to be visible.
The dilated conv layers between blocks 10 and 21 
have the largest $\sigma^2$ values and thus
require the most adaptability,
suggesting that these blocks model the most speaker-specific features.
These layers have a receptive field of $43$--$85$ ms, which
matches our intuition about the domain because 
earlier blocks learn to model high-frequency sinusoid-like waveforms and later blocks model slow-changing prosody.
WaveNet inputs include the fundamental frequency (f0), which controls the change of pitch, and a sequence of linguistic features that provides the prosody.
Both the earlier and later residual blocks can learn to be speaker-invariant given these inputs.
Therefore, it is this middle range of temporal variability that contains the information about speaker identity. We select the $12$ layers with $\sigma^2$ values above $3.0$ for adaptation below. This requires adding only 16\% of the network parameters for each new voice. 
Note that this domain exhibits yet another type of modular structure from those above.

\subsection{Predictive Performance}
\label{sec:exp_performance}

\textbf{Image classification accuracy.}
For each algorithm, we perform extensive hyperparameter tuning on validation data. Details are provided in \cref{sec:apx_experiment}.
\cref{table:aug_omni_classification} shows the test accuracy for augmented Omniglot in the large-data regime.
Both shrinkage variants 
obtain modest accuracy improvements over their non-modular counterparts.
We expect only this small improvement over the non-shrinkage variants, however,
as the heavy data augmentation in this domain reduces overfitting.
\looseness=-1

We now reduce the amount of augmentation and data to make the domain more challenging. \cref{fig:aug_omni_reduced} shows our results in this small-data regime.
Both shrinkage variants significantly improve over their non-shrinkage 
counterparts when there are few training instances per alphabet.
This gap grows as the number of training instances decreases, demonstrating that shrinkage helps prevent overfitting.
Interestingly, the Reptile variants begin to outperform the iMAML variants as the number
of training instances increases, despite the extra validation data used by the iMAML variants.
Results for all combinations of alphabets and instances are shown
in the appendix. 
\looseness=-1

\begin{table}[t!]
    \centering
    \begin{minipage}{0.44\textwidth}
        \begin{center}
        \includegraphics[width=0.9\textwidth]{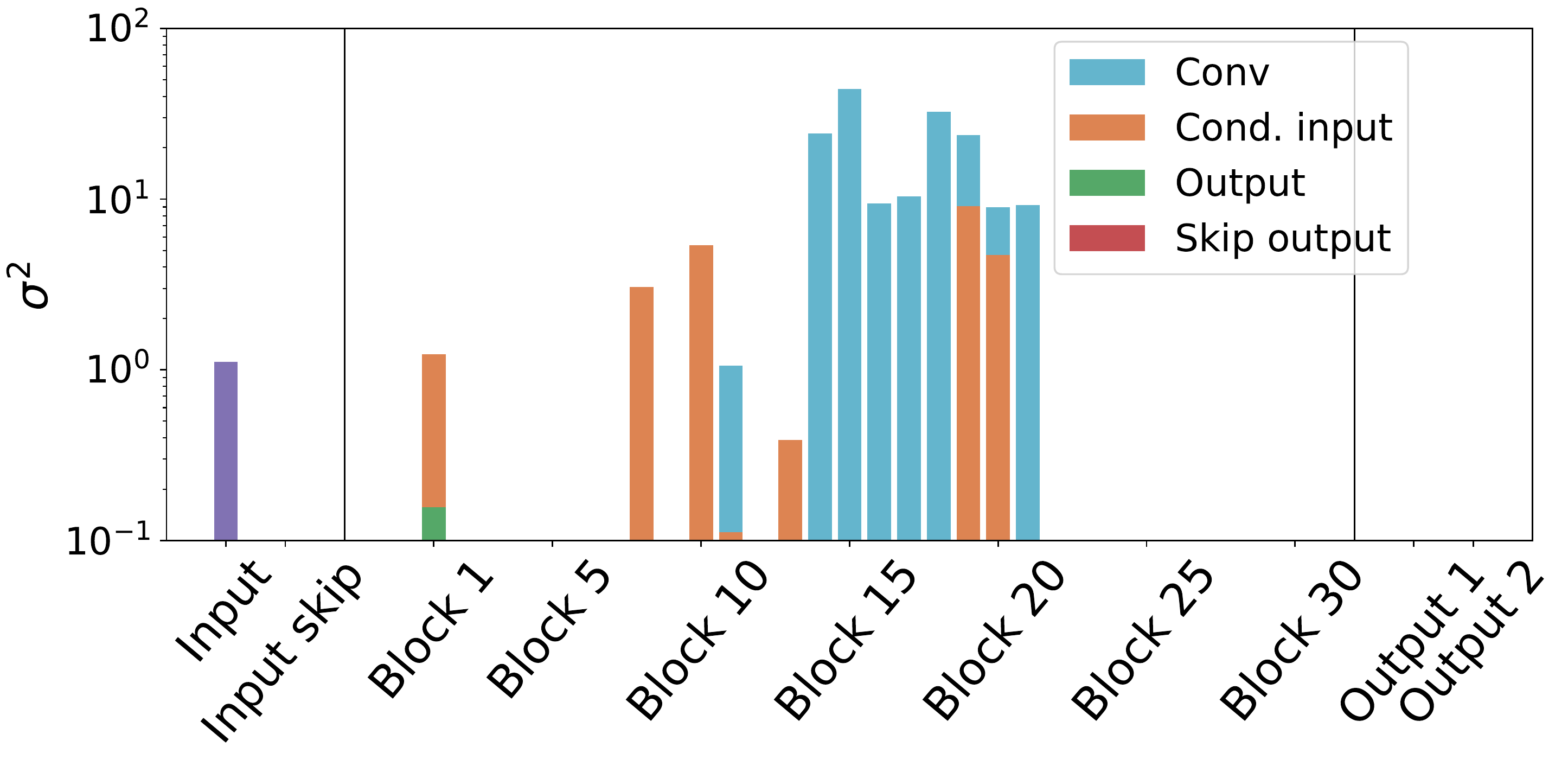}
        \captionof{figure}{Learned $\vsigma^2$ of WaveNet modules. Every block contains four layers (See Appendix~\ref{sec:apx_tts} for details).}
        \label{fig:tts_module}
        \end{center}
    \end{minipage}
    ~
    \begin{minipage}{0.45\textwidth}
        \begin{center}
        \begin{tabular}{lccc}
            \toprule
                                                   \multicolumn{1}{l}{{ \shrimaml{}}} & {\footnotesize$73.6 \pm 1.3\%$} \\
            \multicolumn{1}{l}{{\footnotesize iMAML}} & {\footnotesize$72.8 \pm 1.2\%$} \\ 
            \midrule
            \multicolumn{1}{l}{{\footnotesize \shrreptile{}}} & {\footnotesize $78.9 \pm 1.2\%$} \\
            \multicolumn{1}{l}{{\footnotesize Reptile}} & {\footnotesize $77.8 \pm 1.1\%$} \\  
            \bottomrule
        \end{tabular}
         \end{center}
        \caption{Average test accuracy and $95\%$ confidence intervals for $10$ runs on large-data augmented Omniglot.} 
        \label{table:aug_omni_classification}
    \end{minipage}
\vspace{-1em}
\end{table}

\begin{figure}[t!]
\centering
\begin{subfigure}{0.4\columnwidth}
\centering
\includegraphics[width=\columnwidth]{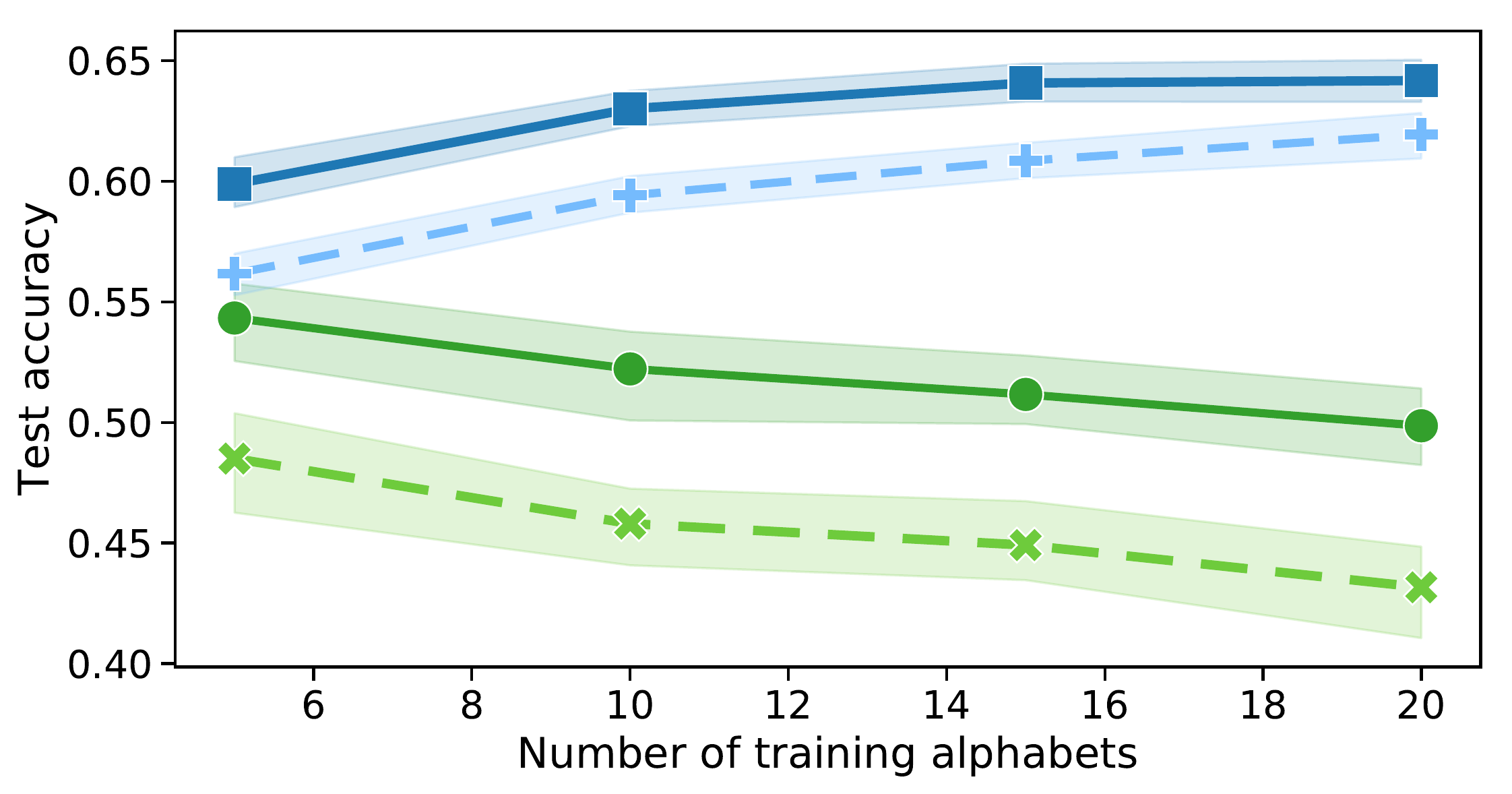}
\label{fig:aug_omni_reduced_instances}
\end{subfigure}
~
\begin{subfigure}{0.4\columnwidth}
\centering
\includegraphics[width=\columnwidth]{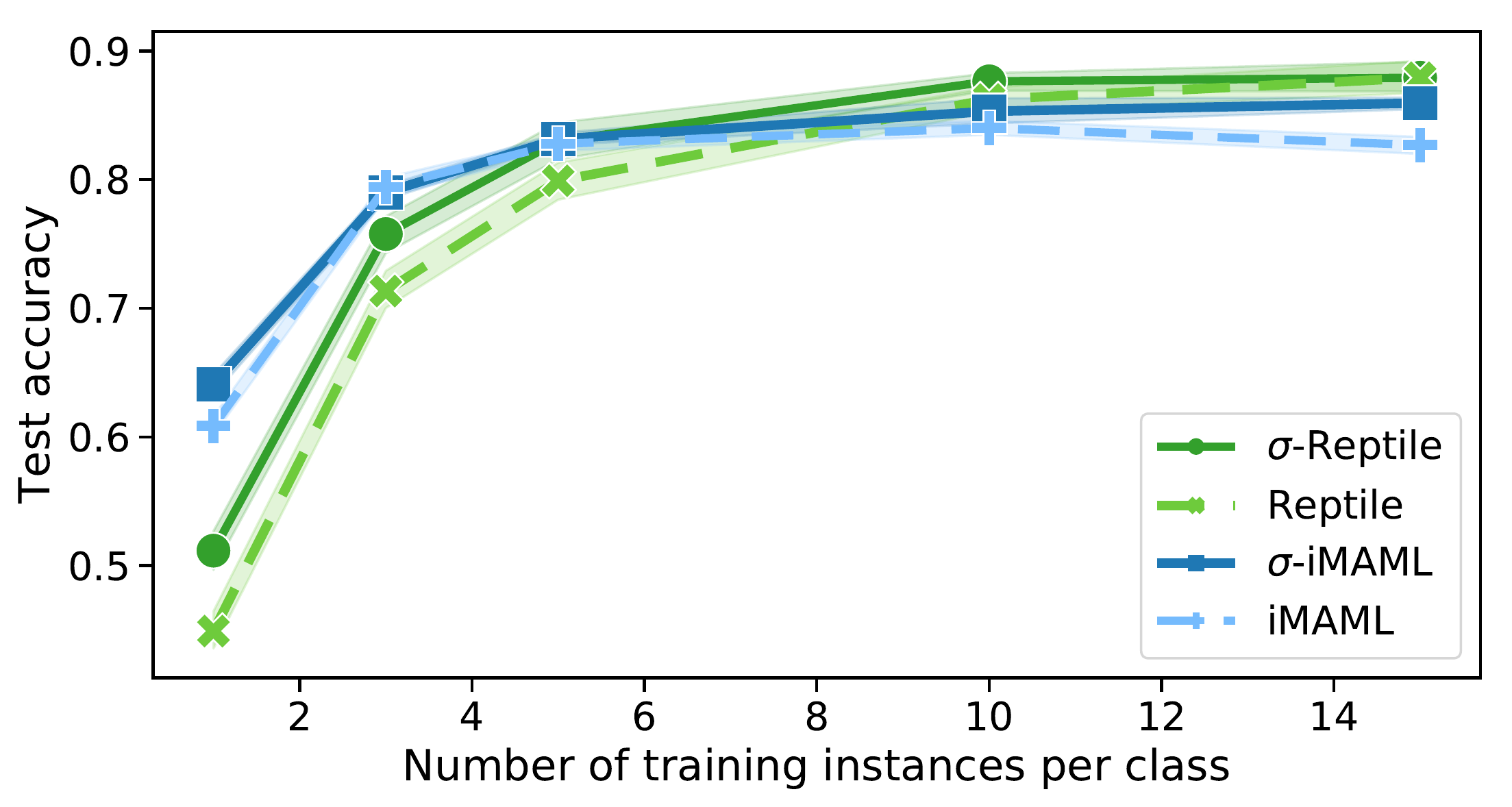}
\label{fig:aug_omni_reduced_alphabets}
\end{subfigure}
\vspace{-5mm}
\caption{Mean test accuracy and $95\%$ confidence intervals for $10$ runs on small-data aug. Omniglot as a function of the number of alphabets (left, $1$ image per character) and instances (right, $15$ alphabets).}
\label{fig:aug_omni_reduced}
\vspace{-1em}
\end{figure}

\textbf{Text-to-speech sample quality.}
The state-of-the art approaches for this domain \citep{chen2019sample} 
are to finetune either the entire model (aka.\ SEA-All) 
or just the speaker embedding (SEA-Emb).
We compare these
two methods to meta-training with Reptile and \shrreptile{}.
We also tried to run \shrmaml{} and \shrimaml{} but
\shrmaml{} ran out of memory with one adaptation step and 
\shrimaml{} trained too slowly.
\looseness=-1

We evaluate the generated sample quality using two voice synthesis metrics:
(1) the voice similarity between a sample and real speaker utterances using a speaker verification model \citep{wan2017generalized,chen2019sample}, and 
(2) the sample naturalness measured by the mean opinion score (MOS) from human raters. 
\cref{fig:tts_similarity} shows the distribution of sample similarities for each method, along with an upper (lower) bound computed from real utterances between the same (different) speakers.
Sample naturalness for each method is shown in \cref{tab:tts_mos}, along with an upper bound created by training the same model on $40$ hours of data.

\shrreptile{} and Reptile clearly outperform SEA-All and SEA-Emb.
\shrreptile{} has comparable median similarity with Reptile, and its sample naturalness surpasses Reptile with both 8 minutes of speech data, and 4 minutes, which is less than used in meta-training.
Overall, the \shrreptile{} samples have the highest quality
despite adapting only $12$ of the $123$ modules.
SEA-All and Reptile, which adapt all modules, overfit quickly
and underperform, despite adaptation being early-stopped.
Conversely, SEA-Emb underfits and does not improve
with more data because it only tunes the speaker embedding.
\looseness=-1

\begin{table}[th!]
    \centering
    \begin{minipage}{0.5\textwidth}
        \centering
        \includegraphics[width=\columnwidth]{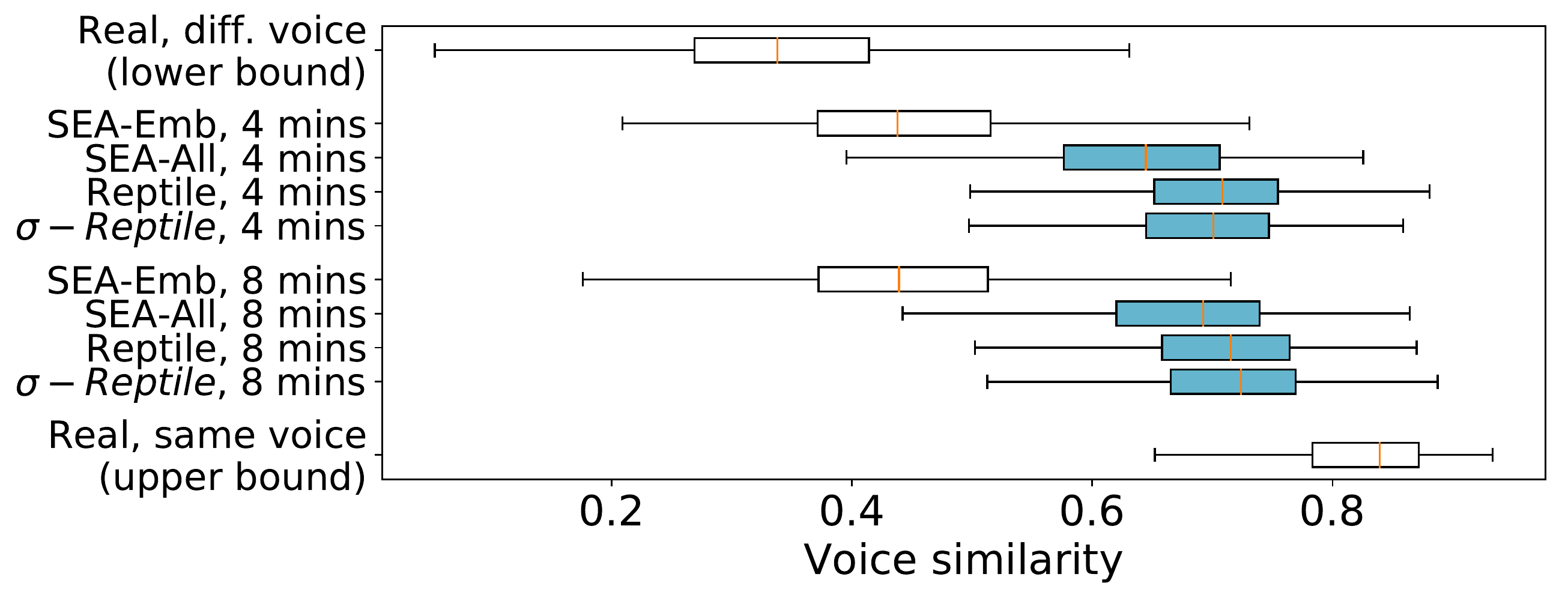}
        \captionof{figure}{Box-plot of voice similarity measurements from utterances (higher is better).}
        \label{fig:tts_similarity}
    \end{minipage}
    ~
        \begin{minipage}{.48\textwidth}
        \begin{center}
                        \resizebox{\columnwidth}{!}{        \begin{tabular}{l|cc}
            & 4 mins & 8 mins \\
            \hline
            SEA-Emb & $1.51 \pm 0.05$ & $1.57 \pm 0.05$ \\
            SEA-All & $1.41 \pm 0.04$ & $1.73 \pm 0.06$ \\
            Reptile & $\textbf{1.93} \pm 0.05$ & $2.09 \pm 0.06$ \\
            \shrreptile{} & $\textbf{1.98} \pm 0.06$ & $\textbf{2.28} \pm 0.06$\\
            \hline
            \parbox[t]{3cm}{Trained with 40 hours\\ of data (upper bound)}
             & \multicolumn{2}{c}{$2.59 \pm 0.07$}
        \end{tabular}
        }
        \end{center}
        \caption{Mean opinion score of sample naturalness. Scores range from 1--5 (higher is better).} 
        \label{tab:tts_mos}
    \end{minipage}
    \vspace{-1em}
\end{table}

\vspace{-.4cm}
\subsection{Discussion}
We thus answer all three experimental questions in the affirmative. 
In both image classification and text-to-speech, the learned shrinkage priors correspond to meaningful and interesting
task-specific modules. 
These modules differ between domains, however, indicating that
they should be learned from data.
Studying these learned modules allows us to discover new or existing knowledge
about the behavior of different parts of the network, 
while adapting only the task-specific modules
provides the same performance as adapting all layers.
Finally, learning and using our shrinkage prior helps prevent
overfitting and improves performance in low-data, long-adaptation regimes. 

\section{Conclusions}

This work proposes a hierarchical Bayesian model for meta-learning that places a shrinkage
prior on each module to allow learning the extent to which each module should adapt, 
without a limit on the adaptation horizon.
Our formulation includes MAML, Reptile, and iMAML as special cases,
empirically discovers a small set of task-specific modules in various domains,
and shows promising improvement in a practical TTS application
with low data and long task adaptation.
As a general modular meta-learning framework, 
it allows many interesting extensions, including
incorporating
alternative 
Bayesian inference algorithms,
modular structure learning, and 
learn-to-optimize methods.

\section*{Broader Impact}

This paper presents a general meta-learning technique to automatically identify task-specific modules in a model for few-shot machine learning problems. It reduces the need for domain experts to hand-design task-specific architectures,
and thus further democratizes machine learning,
which we hope will have a positive societal impact.
In particular, general practitioners who can not afford to collect a large amount of labeled data will be able to take advantage of a pre-trained generic meta-model and adapt its task-specific components to a new task based on limited data. One example application might be to adapt a multilingual text-to-speech model to a low-resource language or the dialect of a minority ethnic group.

As a data-driven method, like other machine learning techniques, the task-independent and task-specific modules discovered by our method are based on the distribution of tasks in the meta-training phase. Adaptation may not generalize to a task with characteristics that fundamentally differ from those of the training distribution. 
Applying our method to a new task without examining the task similarity runs the risk of transferring induced bias from meta-training to the out-of-distribution task. For example, a meta image classification model trained only on vehicles is unlikely to be able to be finetuned to accurately identify a pedestrian based on the adaptable modules discovered during meta-training. To mitigate this problem, we suggest ML practitioners first understand whether the characteristics of the new task match those of the training task distribution before applying our method.
\looseness=-1

\begin{ack}
The authors thank Alexander Novikov for his code for computing Hessian-vector products; Yi Yang for helpful discussions on the meta-learning dataset design; Sander Dieleman, Norman Casagrande and Chenjie Gu for their advice on the TTS model; and Sarah Henderson and Claudia Pope for their organizational help. 

All authors are employees of DeepMind, which was the sole source of funding for this work. None of the authors have competing interests.

\end{ack}

{\small
\bibliography{references}
\bibliographystyle{unsrtnat}
}

\clearpage

\appendix

{{\textbf{\Large{Appendix}}}}

This appendix contains the supplementary material for the main text.
In \cref{sec:analysis}, we first prove \cref{lemma:jointmaximizationisbad} and then provide a detailed analysis of the introduced estimation approaches on a simple example.
In \cref{sec:valid_gd_derivation}, we provide details of the derivation of the implicit gradients in \cref{eq:valid_gd_v1,eq:valid_gd_v2}, show the equivalence of \shrimaml{} and iMAML when $\sigma_m^2$ is shared across all modules and fixed, and provide more discussion on the objective introduced in \cref{sec:alg_map}.
In \cref{sec:synthetic_exps}, we demonstrate the different behavior of approaches that meta-learn the learning rate per module versus our approach that meta-learns the shrinkage prior per module, on two synthetic problems.
In \cref{sec:alg_details}, we provide additional details about our implementation of the iMAML baseline and our shrinkage algorithms. Finally, in \cref{sec:apx_experiment}, we explain the experiments in more detail, provide information about the setup and hyperparameters, and present additional results.

\section{Analysis}
\label{sec:analysis}

In this section, we first prove \cref{lemma:jointmaximizationisbad}, which shows in general that it is not feasible to estimate all the parameters $(\vsigma^2, \vmetav, \vtaskv_{1:T})$ jointly by optimizing the joint log-density. We then provide detailed analysis of the properties of the two estimation approaches introduced in \cref{sec:learning} for a simple hierarchical normal distribution example.
Finally, we discuss the pathological behavior when estimating $\vsigma^2$ from the joint log-density in the same example.

\subsection{Maximization of \texorpdfstring{$\log p(\vtaskv_{1:T},\D|\vsigma^2,\vmetav)$}{log p(theta, D|sigma2, phi)}}
\cref{lemma:jointmaximizationisbad} states that the function $f: (\vsigma^2,\vphi,\vtheta_{1:T}) \mapsto \log p(\vtaskv_{1:T},\D|\vsigma^2,\vmetav)$ diverges to $+\infty$ as 
$\vsigma \rightarrow 0^{+}$ when $\theta_{t,m}=\vmetav_m$ for all $t\in\{ 1,...,T \},m\in\{1,...,M \}$. We first establish here this result.
\begin{proof}[Proof of \cref{lemma:jointmaximizationisbad}]
For simplicity, we prove the result for $M=1$ and $N_t=N$. The extension to the general case is straightforward. 
We have 
\eq{
&\log p(\theta_{1:T},\vx_{1:T} | \phi, \sigma^2) \nn  = - \frac{T}{2} \log \sigma^2 - \frac{1}{2} \sum_{t=1}^{T} \frac{(\theta_t-\phi)^2}{\sigma^2} + \sum_{t=1}^{T}\sum_{i=1}^{N} \log p(x_{t,n}|\theta_t)\,.
}
From this expression, it is clear that when $\theta_t=\phi$ for all $t$, then $\log p(\theta_{1:T},\vx_{1:T} | \phi, \sigma^2) \rightarrow +\infty$ as $\sigma^2 \rightarrow 0^{+}$.
\end{proof}

This shows that the global maximum of $\log p(\theta_{1:T},\vx_{1:T} | \phi, \sigma^2)$ does not exist in general, and thus we should not use this method for estimating $\sigma^2$.

This negative result illustrates the need for alternative estimators. In the following sections, we analyze the asymptotic behavior of the estimates of $\vphi$ and $\vsigma^2$ proposed in \cref{sec:learning} as the number of training tasks $T \rightarrow \infty$ in a simple example. In \cref{sec:map_estimate_sigma}, we analyze the behavior of \cref{lemma:jointmaximizationisbad} 
when optimizing $\log p(\theta_{1:T},\vx_{1:T} | \phi, \sigma^2)$ w.r.t.\ $\vsigma^2$ with a local optimizer.

\subsection{Analysis of estimates in a simple example}
\label{sec:analysis_pll_estimate}

To illustrate the asymptotic behavior of different learning strategies as $T\rightarrow \infty$, we consider a simple model with $M=1$ module, $D_t = D = 1$ for all tasks and normally distributed observations. We use non-bolded symbols to denote that all variables are scalar in this section.
\begin{example}[Univariate normal]
\eq{
&M=1, D_t=1, N_t = N, \forall t=1,\dots,T, \nn\\
&x_{t,n} \sim \gauss(x_{t,n} | \theta_t, 1), \forall t=1,\dots,T,  n=1,\dots,N.\nn
}
\label{ex:normal_in_appx}
\end{example}
It follows that 
\eq{
p(\theta_{1:T},\vx_{1:T} | \phi, \sigma^2)= \prod_{t=1}^{T} \left( \gauss(\theta_t | \phi,\sigma^{2}) \prod_{n=1}^N \gauss(x_{t,n} | \theta_t,1) \right)\,,
}
and we use the notation for the negative joint log-density up to a constant
\eq{
\ell_\pos (\theta_{1:T},\phi,\sigma^2) &= -\log p(\theta_{1:T},\vx_{1:T} | \phi, \sigma^2) \nn\\
& = \frac{T}{2} \log \sigma^2 + \frac{1}{2} \sum_{t=1}^{T} \frac{(\theta_t-\phi)^2}{\sigma^2} + \frac{1}{2}\sum_{t=1}^{T}\sum_{n=1}^{N} (x_{t,n}-\theta_t)^2 \nn\\
&= \sum_{t=1}^T \elltr(\taskv_t, \metav, \sigma^2) \,. \label{eq:norm_pos}
}

We also assume there exists a set of independently sampled validation data $\vy_{1:T}$ where $y_{t,n} \sim \calN(y| \theta_t, 1)$ for $t\in\{ 1,\dots,T \}, n\in \{ 1,\dots,K \}$. The corresponding negative log-likelihood given point estimates $\hat\taskv_{1:T}$ is given up to a constant by
\eq{
\hat{\ell}_\pll = \sum_{t=1}^T \ellval(\hat\taskv_t) = \frac{1}{2}\sum_{t=1}^T \sum_{k=1}^K (y_{t,k} - \hat\taskv_t)^2.
}

We denote by $\phi_\r$ and $\sigma_\r$ the true value of $\phi$ and $\sigma$ for the data generating process described above.

\subsubsection{Estimating \texorpdfstring{$\vmetav$}{phi} and \texorpdfstring{$\vmetav$}{sigma2} with predictive log-likelihood}

We first show that when we estimate $\taskv_t$ with MAP on $\elltr$ and estimate $\metav$ and $\sigma^2$ with the predictive log-likelihood as described in \cref{sec:alg_pll}, the estimates $\hat\metav$ and $\hat\sigma^2$ are consistent.

\begin{proposition} 
Let $\hat\theta_t(\phi,\sigma^2)=\arg \min_{\theta_t} \elltr (\theta_t,\phi,\sigma^2)$ and define  $(\hat\phi,\hat\sigma^2)=\arg_{\phi,\sigma^2} \{\nabla_{(\phi, \sigma^2)} \hat\ell_{\pll}(\hat\theta_{1:T})=0 \}$ then, as $T\rightarrow \infty$, we have 
$$
\hat{\phi}(\hat\sigma^2) \rightarrow \phi_\r, \quad
\hat{\sigma}^2 \rightarrow \sigma_\r^2,
$$
in probability.
\end{proposition}

\begin{proof}
We denote the sample average 
\eq{
\bar{x}_t:=\frac{1}{N} \sum_{n,t} x_{n,t}, \quad
\bar{y}_t:=\frac{1}{K} \sum_{k,t} y_{k,t} \,,
\label{def:sampleaverage}
}
and the average over all tasks $\bar{x}=\frac{1}{T}\sum_{t=1}^T \bar{x}_t$, $\bar{y}=\frac{1}{T}\sum_{t=1}^T \bar{y}_t$.

The equation $\nabla_{\theta_t} \elltr  = 0$ gives
\eq{
\hat\theta_t(\metav, \sigma^2) = \frac{\sum_{n=1}^{N} x_{t,n} + \phi/\sigma^2}{N+1/\sigma^2} = \frac{\bar{x}_t + \phi/N\sigma^2}{1+1/N\sigma^2} . \label{eq:thetahat} 
}

By plugging \cref{eq:thetahat} in \cref{eq:valid_gd_v1}, it follows that
\eq{
\nabla_{\sigma^2} \hat\ell_{\pll} &= -\sum_{t} \nabla_{\taskv_t} \ellval(\taskv_t) \vH^{-1}_{\taskv_t \taskv_t} \vH_{\taskv_t \sigma^2}\nn\\
&= \frac{K}{N\sigma^4\left(1+\frac{1}{N\sigma^2}\right)}
\sum_t (\taskv_t - \bar{y}_t)(\taskv_t - \metav) \nn\\
&= \frac{K}{N\sigma^4(1 + \frac{1}{N\sigma^2})^3} \cdot \nn\\
&\quad\quad \sum_t \left(\bar{x}_t - \bar{y}_t + \frac{1}{N\sigma^2} (\metav - \bar{y}_t)\right) (\bar{x}_t - \metav) \label{eq:normal_example_valid_gd}\,.
}
We then solve $\nabla_{\sigma^2} \hat\ell_{\pll} = 0$ as a function of $\metav$,
\eq{
\hat\sigma^2(\metav)=\frac{\frac{1}{T}\sum_t (\bar{x}_t - \metav)(\bar{y}_t-\metav)}{\frac{N}{T} \sum_t (\bar{x}_t - \bar{y}_t)(\bar{x}_t - \metav)}. \label{eq:root_sigma_interm}
}

Similarly, we solve $\nabla_{\phi} \hat\ell_{\pll}(\hat\theta_{1:T}(\metav, \sigma^2) = 0$ as
\eq{
\nabla_{\phi} \hat\ell_{\pll} &= -\sum_{t} \nabla_{\taskv_t} \ellval(\taskv_t) \vH^{-1}_{\taskv_t \taskv_t} \vH_{\taskv_t \metav}\nn\\
&= -\frac{K}{1+N\sigma^2} \sum_{t=1}^T \left(\bar{y}_t - \frac{\bar{x}_t + \phi/N\sigma^2}{1+1/N\sigma^2} \right) = 0 \nn,\\
}
this yields
\eq{
\hat{\metav}(\sigma^2) = \bar{y} + N\sigma^2(\bar{y} - \bar{x}). \label{eq:root_phi_interm}
}

By combining \cref{eq:root_sigma_interm} to \cref{eq:root_phi_interm}, we obtain
\eq{
\hat\metav = \frac{\overline{x(x-y)}\bar{y} +
\overline{xy}(\bar{y}-\bar{x})}{\bar{x}(\bar{y}-\bar{x})+\overline{x(x-y)}}
}
where, for a function $f$ of variable $x, y$, we define $\overline{f(x,y)} := \frac{1}{T}\sum_{t} f(\bar{x}_t, \bar{y}_t)$.

Following the data generating process, the joint distribution of $\bar{x}_t$ and $\bar{y}_t$ with $\theta_t$ integrated out is jointly normal and satisfies
\eq{
[\bar{x}_t, \bar{y}_t]^T \sim \gauss\left(\phi_\r \vone_2, \begin{bmatrix}
\sigma_\r^2 + \frac{1}{N} & \sigma_\r^2\\
\sigma_\r^2 & \sigma_\r^2 + \frac{1}{K}
\end{bmatrix} \right).
\label{eq:jointnormal}
}
As $T \rightarrow \infty$, it follows from the law of large numbers and Slutsky's lemma that $\hat\phi \rightarrow \phi_\r$ in probability. Consequently, it also follows from (\ref{eq:root_sigma_interm}) that $\hat\sigma^2 \rightarrow \sigma^2_\r$ in probability.
\end{proof}

\subsubsection{Estimating \texorpdfstring{$\metav$}{phi} with MAP and \texorpdfstring{$\sigma^2$}{sigma2} with predictive log-likelihood}

Alternatively, we can follow the approach described in \cref{sec:alg_map} to estimate both $\taskv_{1:T}$ and $\metav$ with MAP on $\ell_\pos$~(\cref{eq:norm_pos}), i.e.,
\eq{
(\hat\phi(\sigma^2),\hat\theta_{1:T}(\sigma^2))=\arg \max_{(\phi,\theta_{1:t})} \ell_\pos (\theta_{1:T},\phi,\sigma^2)
\label{def:mapcondsigma} \,,
}
and estimate $\sigma^2$ with the predictive log-likelihood by 
finding the root of the approximate implicit gradient in \cref{eq:valid_gd_v2}.

We first show that given any fixed value of $\sigma^2$, the MAP estimate of $\metav$ is consistent.
\begin{proposition}[Consistency of MAP estimate of $\phi$]
\label{prop:map_phi}
For any fixed $\sigma>0$, $\hat\phi(\sigma^2)=\bar{x}$, so $\hat\phi(\sigma^2) \rightarrow \phi_\r$ in probability as $T\rightarrow\infty$. 
\end{proposition}
\begin{proof}[Proof of \cref{prop:map_phi}]
The equation $\nabla_\phi \ell_\pos = 0$ for $\ell_\pos$ defined in \cref{eq:norm_pos} gives
\eq{
\hat\phi(\sigma^2) = \frac{1}{T} \sum_{t=1}^{T} \hat\theta_t(\sigma^2).\label{eq:phihat}
}
By summing \cref{eq:thetahat} over $t=1,...,T$ and using \cref{eq:phihat}, we obtain 
\eq{
\hat\phi(\sigma^2)&=\bar{x},\quad 
\hat\theta_t(\sigma^2)=\frac{\bar{x}_t+\frac{\bar{x}}{N\sigma^2}}{1+\frac{1}{N\sigma^2}}.
\label{eq:solutionMAP}
}
The distribution of $\bar{x}$ follows directly from \cref{eq:jointnormal}. Therefore, $\hat\phi(\sigma^2)$ is unbiased for any $T$ and it is additionally consistent by Chebyshev's inequality as $T\rightarrow \infty$.
\end{proof}

Now we show that estimating $\sigma^2$ by finding the roots of the implicit gradient in \cref{eq:valid_gd_v2} is also consistent.

\begin{proposition} Let $(\hat\phi(\sigma^2),\hat\theta_{1:T}(\sigma^2))=\arg \max_{(\phi,\theta_{1:T})} \ell_\pos (\theta_{1:T},\phi,\sigma^2)$ and define  $\hat\sigma^2=\arg_{\sigma^2} \{ \nabla_{\sigma^2} \hat\ell_{\pll}(\hat\theta_{1:T})=0 \}$ where the gradient is defined as in \cref{eq:valid_gd_v2}, then, as $T\rightarrow \infty$, we have 
$$
\hat{\phi}(\hat\sigma^2) \rightarrow \phi_\r, \quad \hat{\sigma}^2 \rightarrow \sigma_\r^2,
$$
in probability.
\end{proposition}
\begin{proof}
From \cref{eq:valid_gd_v2}, we follow the same derivation as \cref{eq:root_sigma_interm} and get the root of the gradient for $\sigma^2$ as a function of $\hat\metav$:
\eq{
\hat\sigma^2(\hat\metav)=\frac{\frac{1}{T}\sum_t (\bar{x}_t - \hat\metav)(\bar{y}_t-\hat\metav)}{\frac{N}{T} \sum_t (\bar{x}_t - \bar{y}_t)(\bar{x}_t - \hat\metav)}.
}
By plugging \cref{eq:solutionMAP}, it follows from the joint distribution of $(\bar{x}_t, \bar{y}_t)$ (Eq.~\ref{eq:jointnormal}), the law of large numbers and Slutsky’s lemma that $\hat\sigma^2 \rightarrow \sigma^2_\r$ in probability as $T \rightarrow \infty$.
\end{proof}

\subsubsection{MAP estimate of \texorpdfstring{$\sigma^2$}{sigma2}}
\label{sec:map_estimate_sigma}
From \cref{lemma:jointmaximizationisbad}, we know that maximizing $\ell_\pos$ (Eq.~\ref{eq:norm_pos}) w.r.t. $(\theta_t, \phi, \sigma^2)$ is bound to fail. Here we show the specific $\sigma^2$ estimate one would obtain by following gradient descent on $\ell_\pos$ in the running example.

Let $S$ denote the sample variance of $\bar{x}_t$ across tasks, that is
\eq{
S=\frac{\sum_{t=1}^{T} (\bar{x}_t-\bar{x})^2}{T} \,.
\label{eq:samplevariance}
}
We can easily establish from \cref{eq:jointnormal} that
\eq{
S\sim \frac{\sigma_\r^2 + \frac{1}{N}}{T}\chi_{T-1}^2 \,,
}
where $\chi_{T-1}^2$ is the standard Chi-squared random variable with $T-1$ degrees of freedom. Although $\hat\phi(\sigma^2)$ is consistent whenever $\sigma^2>0$, the following proposition shows that maximizing $\ell_\pos(\hat\theta_{1:T}(\sigma^2), \hat\phi(\sigma^2), \sigma^2)$ w.r.t. $\sigma^2$ remains problematic. 

\begin{proposition}[Estimation of $\sigma$ by gradient descent]
\label{prop:map_sigma}
Minimizing the function $\sigma^2 \mapsto \ell_\pos(\hat\theta_{1:T}(\sigma^2), \hat\phi(\sigma^2), \sigma^2)$ by gradient descent will diverge at $\sigma \rightarrow 0^{+}$ if either of the following two conditions is satisfied
\begin{enumerate}
    \item $S < \frac{4}{N}$,
    \item $\sigma^2$ is initialized in
$\left(0, \frac{1}{2}\left(S - \frac{2}{N} - \sqrt{S (S - \frac{4}{N})} \right)\right)$.
\end{enumerate}
Otherwise, it converges to a local minimum
$
\hat\sigma^2 = \frac{1}{2}\left(S - \frac{2}{N} + \sqrt{S (S - \frac{4}{N})} \right)
$.
\end{proposition}

\begin{corollary}
As the number of training tasks $T \rightarrow \infty$, condition 1 is equivalent to
$$\sigma_\r^2 < \frac{3}{N}$$
while the upper endpoint of the interval in condition 2 becomes 
$$
\frac{1}{2}\left(\sigma_\r^2 - \frac{1}{N} - \sqrt{(\sigma_\r^2 + \frac{1}{N}) (\sigma_\r^2 - \frac{3}{N})} \right)\,,
$$
beyond which the $\sigma^2$ estimate converges to
\eq{
\hat\sigma^2 = \frac{1}{2}\left(\sigma_\r^2 - \frac{1}{N} + \sqrt{(\sigma_\r^2 + \frac{1}{N}) (\sigma_\r^2 - \frac{3}{N})} \right)\,.
}
\label{cor:map_sigma}
\end{corollary}

\begin{proof}[Proofs of \cref{prop:map_sigma} and \cref{cor:map_sigma}]
It follows from \cref{eq:solutionMAP} that 
\eq{
\hat\theta_t(\sigma)-\hat\phi(\sigma) &= \frac{\bar{x}_t-\bar{x}}{1+\frac{1}{N \sigma^2}}.
\label{eq:theta_minus_phi}
}

By solving $\nabla_{\sigma^2} \ell_\pos  = 0$, we obtain  
\begin{align}
\hat\sigma^2 = \frac{\sum_{t=1}^{T} (\theta_t-\phi)^2}{T}. \label{eq:sig2hat}
\end{align}
Plugging \cref{eq:theta_minus_phi} into this expression yields
\eq{
\hat\sigma^2 =\frac{\sum_{t=1}^{T} (\bar{x}_t-\bar{x})^2}{T(1+\frac{1}{N \hat\sigma^2})^2}=\frac{S}{(1+\frac{1}{N \hat\sigma^2})^2}. \label{eq:sig2hat_2}
}
Hence, by rearranging this expression, we obtain the following quadratic equation for $\hat\sigma^2$ 
\eq{
\hat\sigma^4+(2/N-S) \hat\sigma^2+1/N^2 =0. \label{eq:sig2_root_eq}
}

Positive roots of \cref{eq:sig2_root_eq} exist if and only if
\eq{
S \geq \frac{4}{N}. \label{eq:root_cond_in_proof}
}

When condition (\ref{eq:root_cond_in_proof}) does not hold, no stationary point exists and gradient descent from any initialization will diverge toward $\hat\sigma^2 \rightarrow 0^{+}$ as it can be checked that $\nabla_{\sigma^2} \ell_\pos < 0$. \cref{fig:normal_example_ll_2,fig:normal_example_grad_2} illustrate $\ell_\pos$ and $\nabla_{\sigma^2} \ell_\pos$ as a function of $\sigma^2$ in this case.

When the condition above is satisfied, there exist two (or one when $S=4/N$, an event of zero probability) roots at:
\eq{
\sigma_{\mathrm{root}}^2 &= \frac{1}{2}\left(S - \frac{2}{N} \pm \sqrt{S (S - \frac{4}{N})} \right).
\label{upperendpoint}}

By checking the sign of the gradient $\nabla_{\sigma^2} \ell_\pos$ and plugging it in to \cref{eq:theta_minus_phi}, we find that the left root is a local maximum and the right root is a local minimum. So if one follows gradient descent to estimate $\sigma^2$, it will converge towards $0^{+}$ when $\sigma^2$ is initialized below the left root and to the second root otherwise.
\cref{fig:normal_example_ll_1,fig:normal_example_grad_1} illustrate the function of $\ell_\pos$ and its gradient as a function of $\sigma^2$ when $\phi$ and $\theta_t$ are at their stationary point and condition 1 is satisfied.

To prove the corollary, we note that it follows from \cref{eq:jointnormal} that $S \rightarrow \sigma_\r^2 + \frac{1}{N}$ as $T \rightarrow \infty$. Hence the condition \cref{eq:root_cond_in_proof} approaches 
\eq{
\sigma_\r^2 \geq \frac{3}{N}.
\label{eq:root_cond_inf}
}
Similarly, when $T \rightarrow \infty$, \cref{eq:root_cond_inf} becomes
\eq{
\sigma_{\mathrm{root}}^2 &= \frac{1}{2}\left(\sigma_\r^2 - \frac{1}{N} \pm \sqrt{(\sigma_\r^2 + \frac{1}{N}) (\sigma_\r^2 - \frac{3}{N})} \right).
}

\end{proof}

\begin{figure}
    \centering
    \begin{subfigure}[b]{0.4\textwidth}
        \includegraphics[width=\textwidth]{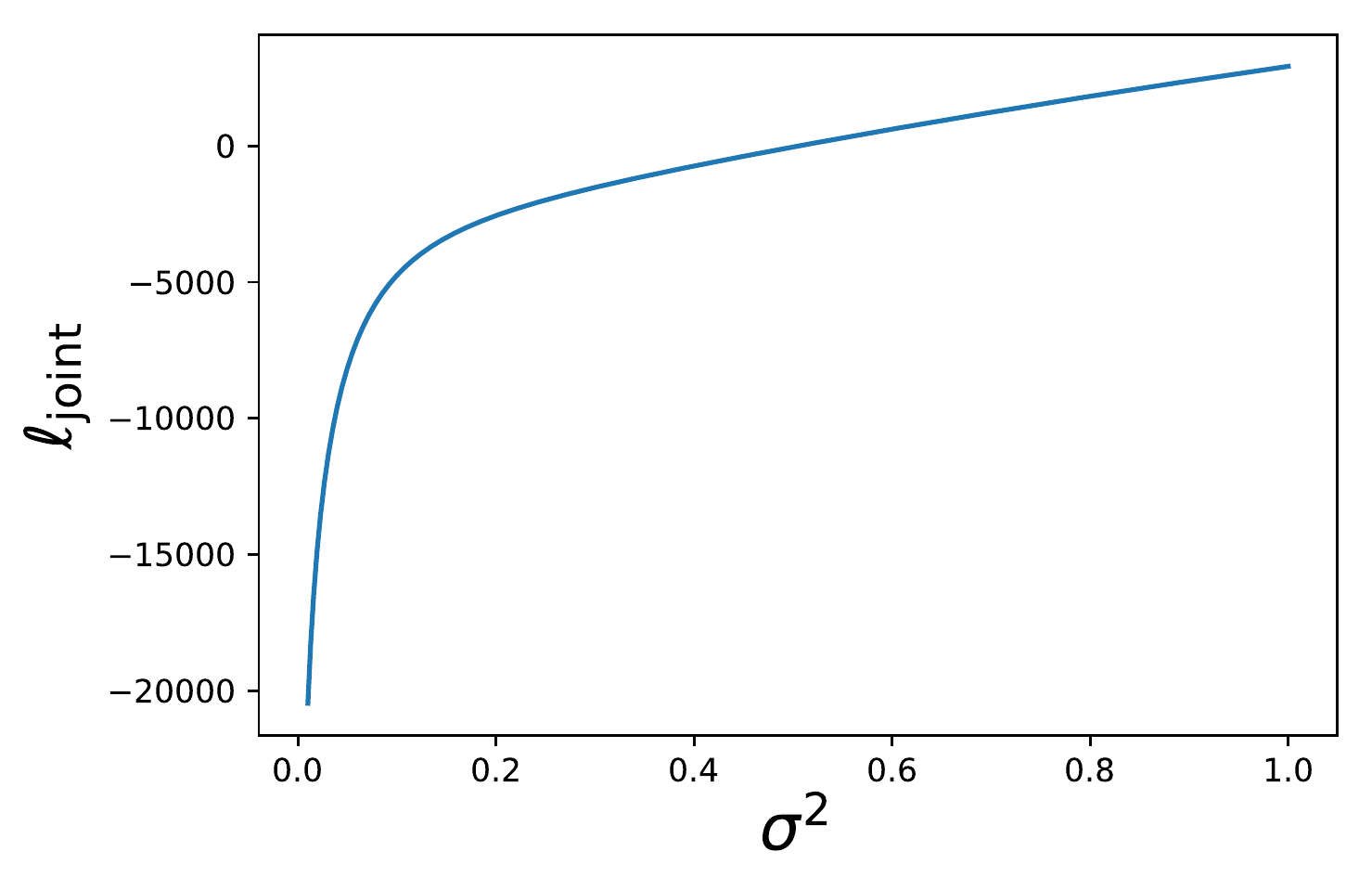}
        \caption{$\ell_\pos(\sigma^2)$ when stationary points do not exist.}
        \label{fig:normal_example_ll_2}
    \end{subfigure}
    ~
    \begin{subfigure}[b]{0.4\textwidth}
        \includegraphics[width=\textwidth]{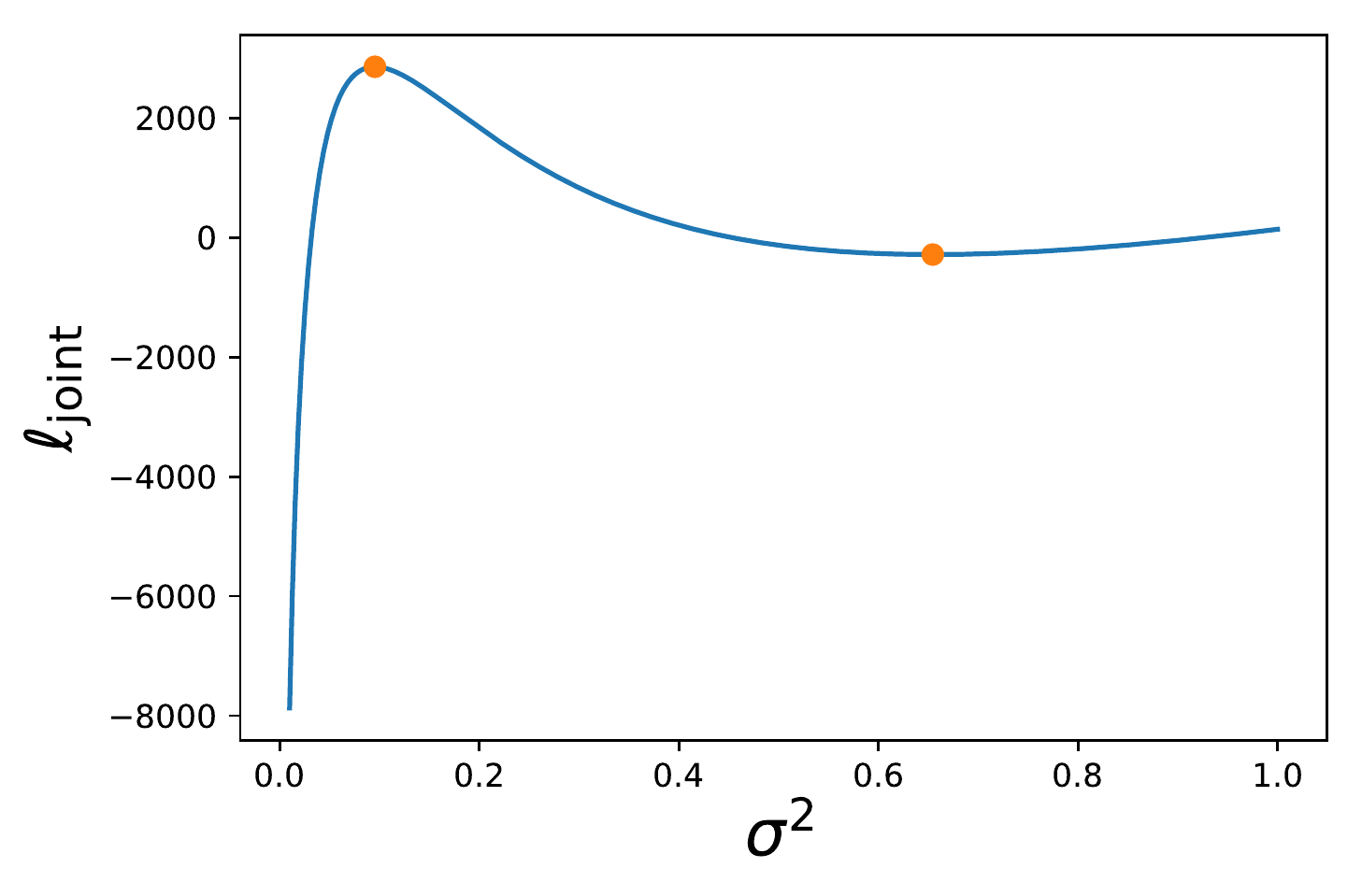}
        \caption{$\ell_\pos(\sigma^2)$ when stationary points exist.}
        \label{fig:normal_example_ll_1}
    \end{subfigure}

    \begin{subfigure}[b]{0.4\textwidth}
        \includegraphics[width=\textwidth]{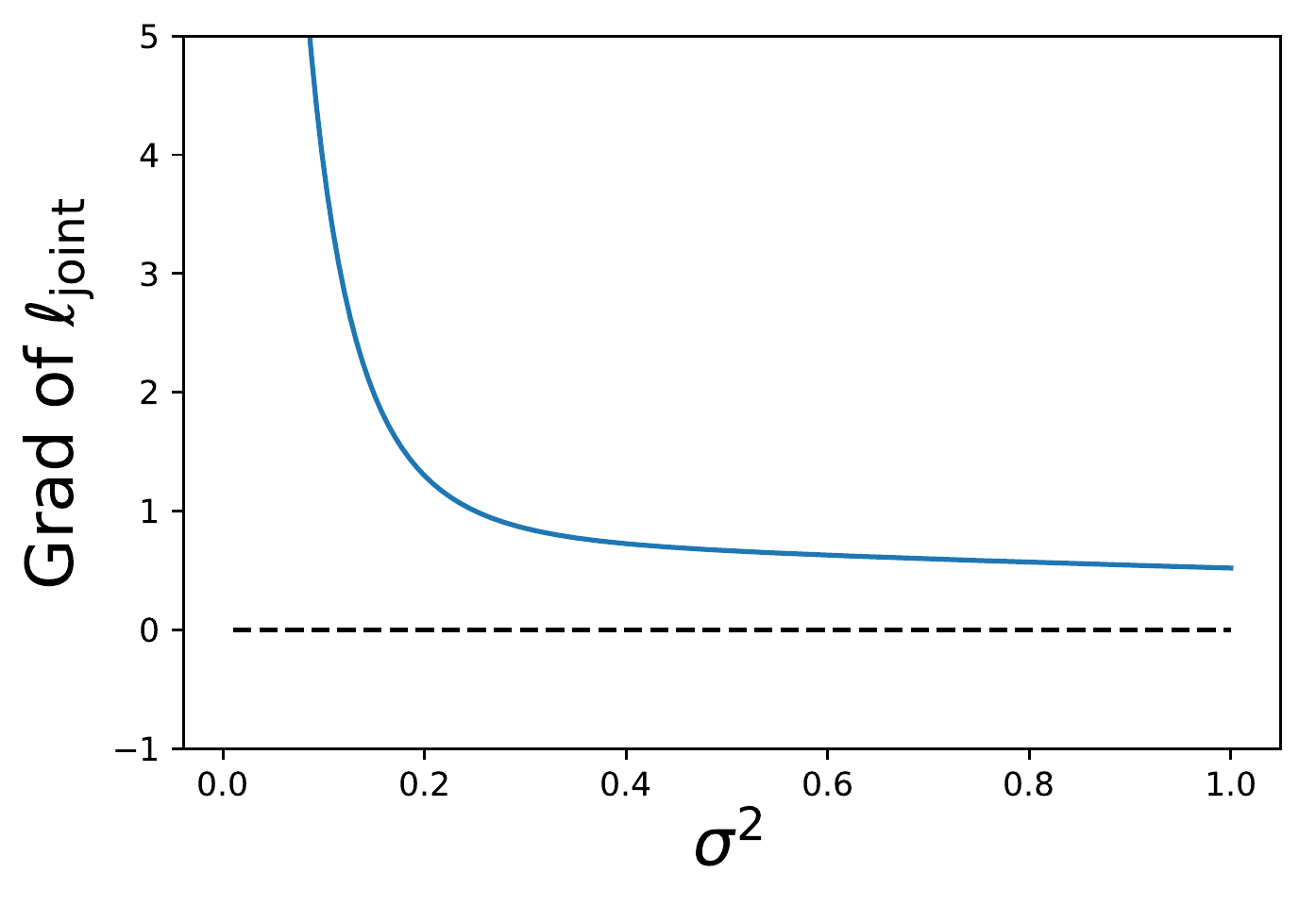}
        \caption{$\nabla_{\sigma^2} \ell_\pos$ when stationary points do not exist.}
        \label{fig:normal_example_grad_2}
    \end{subfigure}
    ~
    \begin{subfigure}[b]{0.4\textwidth}
        \includegraphics[width=\textwidth]{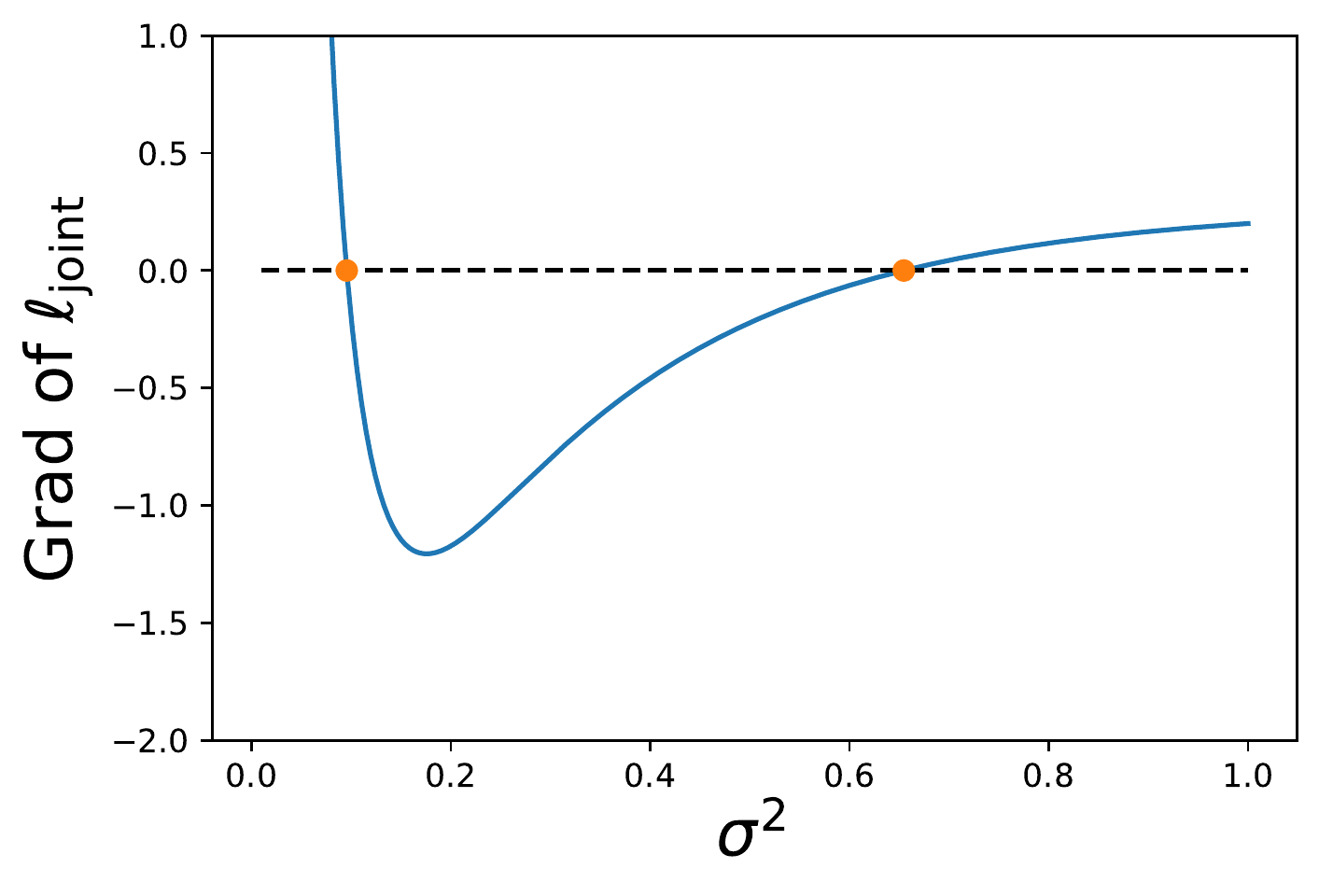}
        \caption{$\nabla_{\sigma^2} \ell_\pos$ when stationary points exist.}
        \label{fig:normal_example_grad_1}
    \end{subfigure}
    \caption{Example of $\ell_\pos(\sigma^2)$ up to a constant, and its gradient w.r.t.~$\sigma^2$. Orange dots denote stationary points.}\label{fig:normal_example_ll_and_grad}
\end{figure}

\section{Derivation of the approximate gradient of predictive log-likelihood in \texorpdfstring{\cref{sec:learning}}{Section 4}}
\label{sec:valid_gd_derivation}

\subsection{Implicit gradient of \texorpdfstring{\shrimaml{}}{Shrinkage-iMAML} in \texorpdfstring{\cref{eq:valid_gd_v1}}{Eq. (7)}}
\label{sec:valid_gd_derivation_v1}

\begin{lemma}(Implicit differentiation)
Let $\hat{\vy}(\vx)$ be the stationary point of function $f(\vx, \vy)$, i.e. $\nabla_{\vy} f(\vx, \vy)|_{\vy=\hat{\vy}(\vx)}=0~\forall \vx$, then the gradient of $\hat{\vy}$ w.r.t. $\vx$ can be computed as
\eq{
\nabla_\vx \hat{\vy}(\vx) = 
- (\nabla_{\vy \vy}^2 f )^{-1} 
\nabla_{\vy \vx}^2f.
}
\label{lem:stationary_derivative}
\end{lemma}

By applying the chain rule, the gradient of the approximate predictive log-likelihood $\ellval(\hat{\vtheta}_{t})$ (from \cref{eq:train-objective_v1}) w.r.t.\ the meta variables $\vPhi=(\vsigma^2, \vmetav)$ is given by
\eq{
\nabla_{\vPhi} \ellval(\hat{\vtheta}_{t}(\vPhi))
= \nabla_{\hat{\vtheta}_t} \ellval(\hat{\vtheta}_{t}) \nabla_{\vPhi}\hat{\vtheta}_{t}(\vPhi) \,.
\label{eq:chain_rule}
}

Applying \cref{lem:stationary_derivative} to the joint log-density on the training subset in \cref{eq:map_v1}, $\elltr(\vtaskv_t, \vPhi)$. We have
\be
\nabla_{\vPhi}\hat{\vtheta}_{t}(\vPhi) =
-\left(\nabla^2_{\vtaskv_t \vtaskv_t}\elltr\right)^{-1} \nabla^2_{\vtaskv_t \vPhi}\elltr.
\label{eqn:derivative_theta_hat_sigma}
\ee
Plug the equation above to \cref{eq:chain_rule} and we obtain the implicit gradient of $\ellval$ in \cref{eq:valid_gd_v1}.

\subsection{Equivalence between \texorpdfstring{\shrimaml{}}{Shrinkage-iMAML} and iMAML when \texorpdfstring{$\vsigma_m^2$}{module variance} is constant}
\label{sec:equiv_to_imaml}

When all modules share a constant variance, $\sigma_m^2 \equiv \sigma^2$, we expand the log-prior term for task parameters $\vtaskv_t$ in $\elltr$ (\ref{eq:map_v1}) and plug in the normal prior assumption as follows,
\eq{
\log p(\vtaskv_t | \vsigma^2, \vphi)
&=\sum_{m=1}^M \bigg(
-\frac{D_m}{2}\log (2\pi \sigma_m^2) - \frac{\|\vtaskv_{mt} - \vmetav_m\|^2}{2\sigma_m^2}\bigg) \nn\\
&= -\frac{D}{2}\log (2\pi \sigma^2) - \frac{\|\vtaskv_t - \vmetav\|^2}{2\sigma^2}.
}

By plugging the equation above to \cref{eq:valid_gd_v1}, we obtain the update for $\vmetav$ as
\eq{
\Delta_t^{\text{\shrimaml{}}}
&= \nabla_{\vtaskv_t} \ell_t^\val(\vtaskv_t) 
\left(\frac{1}{\sigma^2}\vI -\nabla^2_{\vtaskv_t \vtaskv_t}\log p(\D_t^\tr|\vtaskv_t) \right)^{-1} \frac{1}{\sigma^2} \vI \nn\\
&= \nabla_{\vtaskv_t} \ell_t^\val(\vtaskv_t) 
\left(\vI -\sigma^2 \nabla^2_{\vtaskv_t \vtaskv_t}\log p(\D_t^\tr|\vtaskv_t) \right)^{-1}.
}
This is equivalent to the update of iMAML by defining the regularization scale $\lambda=1/\sigma^2$ and plugging in the definition of $\elltr := -\log p(\D_t^\tr|\vtaskv_t)$ in \cref{sec:background}.

\subsection{Discussion on the alternative procedure for Bayesian parameter learning (\texorpdfstring{\cref{sec:alg_map}}{Section 4.2})}\label{sec_jointmap}

By plugging the MAP of $\vmetav$ (\cref{eq:log_p_train_joint}) into \cref{eq:log_p_sigma2_full} and scaling by $1/T$, we derive the approximate predictive log-likelihood as
\eq{
\hat{\ell}_\pll(\vsigma^2)
=
\frac{1}{T} \sum_{t=1}^T \ellval(\hat{\vtheta}_{t}(\vsigma^2)).
\label{eq:log_phat_sigma2}
}

It is a sensible strategy to estimate both the task parameters $\vtaskv_{1:T}$ and the prior center $\vmetav$ with MAP on the training joint log-density and estimate the prior variance $\vsigma^2$ on the predictive log-likelihood. If $\hat{\vmetav}(\vsigma^2)\rightarrow \bar{\vmetav}(\vsigma^2)$ as $T\rightarrow \infty$ we can think of both $\ellpll(\vsigma^2)$ and $\hat{\ell}_{\pll}(\vsigma^2)$ as approximations to  
\eq{\tilde{\ell}_{\pll}(\vsigma^2)=-\frac{1}{T}\sum_{t=1}^T \log p(\D_{t}^\val|\D_{t}^\tr,\bar{\vmetav}(\vsigma^2),\vsigma^2),
}
which, for $(\D_{t}^\tr,\D_{t}^\val)\overset{\textrm{i.i.d.}}{\sim} \nu$, converges by the law of large numbers as $T\rightarrow \infty$ towards 
\eq{\tilde{\ell}_{\pll}(\vsigma^2)\rightarrow -\mathbb{E}_{\nu(\D_{t}^\tr,\D_{t}^\val)}[\log p(\D_{t}^\val|\D_{t}^\tr,\bar{\vmetav}(\vsigma^2),\vsigma^2)].
}
Similarly to \cref{eq:reinterpretationKL}, it can be shown that minimizing the r.h.s. of \cref{eq:log_phat_sigma2} is  equivalent to minimizing the average KL 
\eq{
\mathbb{E}_{\nu(\D_{t}^\tr)}[\textrm{KL}(\nu(\D_{t}^\val|\D_{t}^\tr)||p\left(\D_{t}^\val | \D_{t}^\tr, \bar{\vmetav}(\vsigma^2),\vsigma^2\right))].
}

\subsection{Meta update of \texorpdfstring{\shrreptile{}}{Shrinkage-Reptile} in \texorpdfstring{\cref{eq:valid_gd_v2}}{Eq. (10)}}
\label{sec:valid_gd_derivation_v2}
The meta update for $\vmetav$ can then be obtained by differentiating (\ref{eq:log_p_train_joint}) with respect to $\vmetav$.

To derive the gradient of \cref{eq:log_phat_sigma2} with respect to $\vsigma$, notice that when $\vmetav$ is estimated as the MAP on the training subsets of all tasks, it becomes a function of $\vsigma^2$. Denote by $\ell^{\mathrm{joint}}(\vTheta, \vsigma^2)$ the objective in \cref{eq:log_p_train_joint} where $\vTheta = (\vtaskv_{1:T}, \vmetav)$ is the union of all task parameters $\vtaskv_t$ and the prior central $\vmetav$. It requires us to apply the implicit function theorem to $\ell^{\mathrm{joint}}(\vTheta, \vsigma^2)$ in order to compute the gradient of the approximate predictive log-likelihood w.r.t.\ $\vsigma^2$. However, the Hessian matrix $\nabla^2_{\vTheta,\vTheta}\ell^{\mathrm{joint}}$ has a size of $D(T+1) \times D(T+1)$ where $D$ is the size of a model parameter $\vtaskv_t$, which becomes too expensive to compute when $T$ is large.

Instead, we take an approximation and ignore the dependence of $\hat{\vmetav}$ on $\vsigma^2$. Then $\vmetav$ becomes a constant when computing $\nabla_{\vsigma^2} \hat{\vtaskv}_t(\vsigma^2)$, and the derivation in \cref{sec:valid_gd_derivation_v1} applies by replacing $\vPhi$ with $\vsigma^2$, giving the implicit gradient in \cref{eq:valid_gd_v2}.

\section{Synthetic Experiments to Compare MAML, Meta-SGD and Shrinkage}
\label{sec:synthetic_exps}

In this section, 
we demonstrate the difference in behavior between learning the learning rate per module and learning the shrinkage prior per module on two synthetic few-shot learning problems.

In the first experiment we show that when the number of task adaptation steps in meta-training is sufficiently large for task parameters to reach convergence, meta-learning the learning rate per module has a similar effect as shrinkage regularization when evaluated at the \textit{same} adaptation step in meta-testing; however, this does not generalize to other adaptation horizons. 
In the second experiment, when the required number of adaptation steps is longer than meta-training allows, the learned learning rate is determined by the local curvature of the task likelihood function whereas the learned shrinkage variance is determined by task similarity from the data generation prior. 
Grouping parameters with similar learned learning rates or variances then
induces different ``modules,'' which correspond to different aspects of
the meta-learning problem (namely, adaptation curvature vs. task similarity).

We compare the following three algorithms:
\begin{enumerate}
    \item vanilla MAML\citep{Finn2017}: does not have modular modeling; learns the  initialization and uses a single learning rate determined via hyper-parameter search.
    \item Meta-SGD\citep{li2017meta}: learns the initialization as well as a separate learning rate for each parameter.
    \item \shrreptile{}: learns the initialization and prior variance $\sigma_m^2$ for each module.
\end{enumerate}
We run each algorithm on
two few-shot learning problems, both of which 
have the same hierarchical normal data generation process:
\eq{
\vtheta_{m,t} &\sim \gauss(\vtheta_{m,t} | \vphi_m, \sigma_m^2), \nn\\
\vx_{t,n} &\sim \gauss(\vx_{t,n} | \vmu(\vtheta_t), \Xi) \nn\,,
}
for each latent variable dimension $m$, task $t$, and data point $n$.
The hyper-parameters $\vphi$ and $\vsigma^2$ are shared across all tasks but unknown, and each parameter dimension $\theta_m$ has different prior variance, $\sigma_m^2$. For simplicity, we let every parameter dimension $m$ correspond to one module for Meta-SGD and \shrreptile{}.
The $n$-th observation $\vx_{t,n}$ for task $t$ is sampled from a Gaussian distribution. The mean is a known function of task parameter $\vtheta$. The observation noise variance $\Xi = \text{diag}(\vxi^2) = \text{diag}(\xi_1^2, \dots, \xi_D^2)$ is a fixed and known diagonal matrix. The difference between the two problems is that $\vmu$ is a linear function of $\vtheta_t$ in the first problem and non-linear in the second.

The task is few-shot density estimation, that is, to estimate the parameters $\vtheta_{\tilde{t}}$ of a new task 
$\T_{\tilde{t}}$ given a few observations $\{\vx_{\tilde{t}, n}\}$, where
$N_t^\tr = N_t^\val$ for all tasks.
To understand the behavior of different algorithms under different task loss landscapes we examine different transformation functions $\vmu$.

Note that the data generation process matches the Bayesian hierarchical model of shrinkage and thus \shrreptile{} may have an advantage in terms of predictive performance. However, the main purpose of this section is to demonstrate 
the \textit{behavior} of these approaches and not to focus on which 
performs better in this simple task.

For each method, we use gradient descent for task adaptation to optimize the training loss (negative log-likelihood), and then evaluate the generalization loss on holdout tasks. MAML meta-learns the initialization of 
\textsc{TaskAdapt}, $\vmetav$, while Meta-SGD meta-learns both $\vmetav$ and a per-parameter learning rate $\alpha_m$, and \shrreptile{} meta-learns the shrinkage prior mean $\vmetav$ and per-parameter variance $\vsigma_m^2$. 
Hyperparameters for each algorithm (i.e., learning rates, number of adaptation steps, and initial values of $\vsigma$) were chosen by extensive random search.
We chose the values that minimized the generalization loss in meta-test.
Because the model is small, we are able to run MAML and Meta-SGD for up to 200 task adaptation steps.

\subsection{Experiment 1: Linear transformation}

\begin{figure}[tbh]
    \centering
    \begin{subfigure}{0.48\textwidth}
        \includegraphics[width=\textwidth]{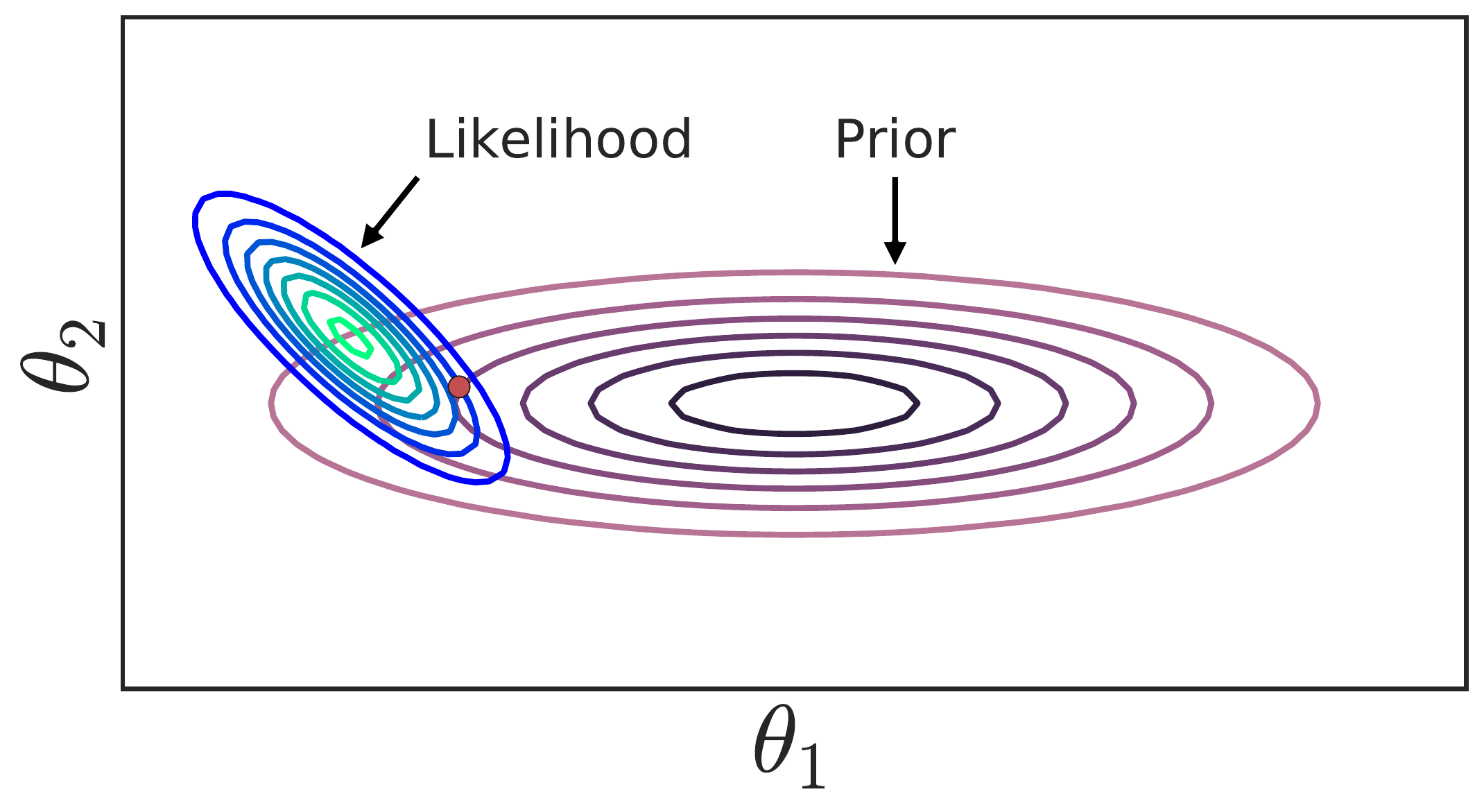}
            \caption{Illustration of the hierarchical normal distribution with linear $\mu$.}
    \label{fig:syn_exp_1_theta}
    \end{subfigure}
        \hspace{0.02\textwidth}
      \begin{subfigure}{0.48\textwidth}
        \includegraphics[width=\textwidth]{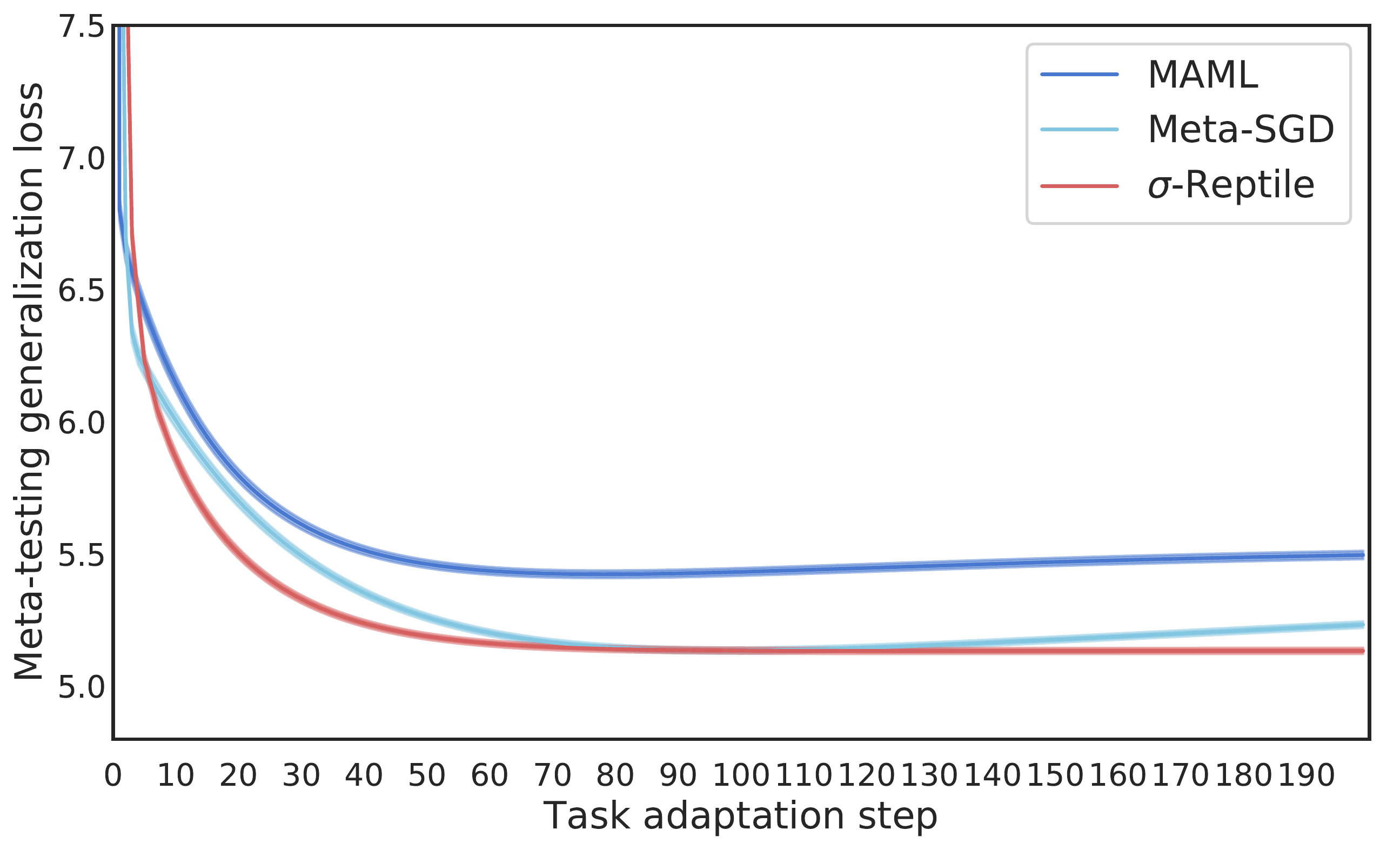}
    \vspace{-1em}
    \caption{Generalization loss in meta-testing.}
    \label{fig:syn_exp_1_loss}
    \end{subfigure}
\caption{Experiment 1: Linear transform.}
\label{fig:syn_exp_1}
\end{figure}

\begin{figure}[tbh]
    \centering
    \includegraphics[width=\textwidth]{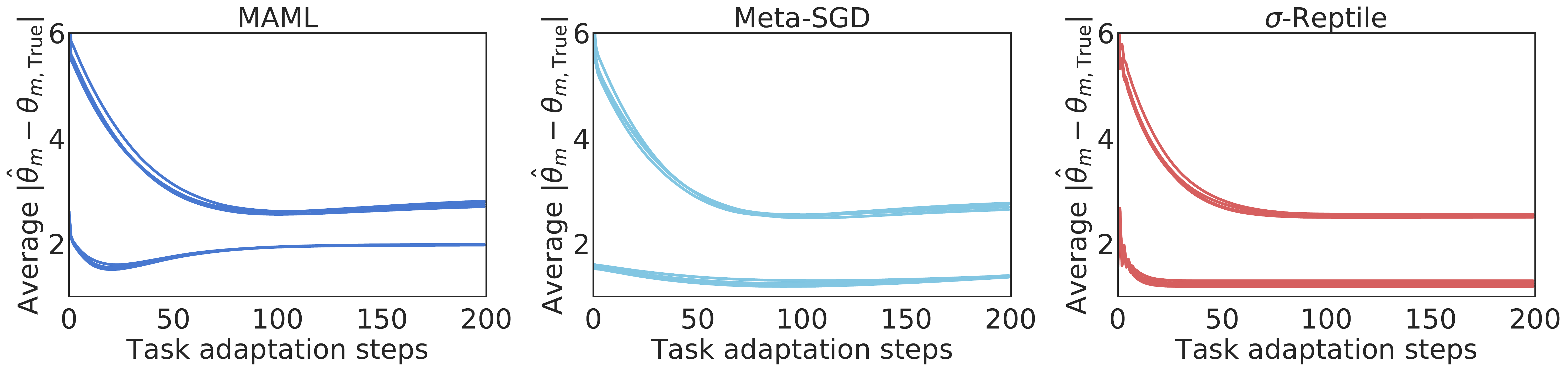}
    \caption{Mean absolute error of the estimate of each parameter as a function of task adaptation step.}
    \label{fig:synthetic_exp1_MAE_per_module}
\end{figure}

We begin with a simple model -- a joint normal posterior distribution 
over $\vtheta_t$ with parameters
\eq{
M=&8, \nonumber \\
D=&9, \nonumber \\
\vphi =& \vone_M, \nonumber \\
\vsigma =& [8,8,8,8,2,2,2,2],  \nonumber \\
\vxi =& [8,8,8,8,5,5,5,5,1], \nonumber 
}
and transformation
\eq{
\vmu(\vtheta_t) = [ \vI_M,  \nicefrac{\vone_M}{\sqrt{M}} ]^\top \vtheta_t. \nonumber
} 
The mean of the observations is thus $\vtheta$ in the first $M$ dimensions
and $\tfrac{1}{\sqrt{M}} \sum_m \theta_m$ in the final dimension.
To make each task nontrivial, we let $\vxi$ be small in the final dimension 
(i.e., $\xi_M$ is small) so that
the posterior of $\vtheta$ is restricted to a small subspace near the
$\sum_m\theta_m = \frac{\sqrt{M}}{N_t^\tr} \sum_n x_{t,n}$ hyperplane.
Gradient descent thus converges slowly regardless of the number of observations.
\cref{fig:syn_exp_1_theta} shows an example of this model for the first two dimensions. 

Clearly, there are two distinct modules, $\vtheta_{1:4}$ and $\vtheta_{5:8}$ but,
in this experiment,
we do not give the module structure to the algorithms and instead
treat each dimension as a separate module.
This allows us to evaluate how well the algorithms can identify the 
module structure from data.

Note that for every task the loss function is quadratic and its Hessian matrix is
constant in the parameter space. 
It is therefore an ideal situation for Meta-SGD to learn an
optimal preconditioning matrix (or learning rate per dimension).

\cref{fig:syn_exp_1_loss} shows the generalization loss on meta-testing tasks.
With the small number of observations for each task, 
the main challenge in this task is overfitting.
Meta-SGD and \shrreptile{} obtain the same minimum generalization loss, and both 
are better than the non-modular MAML algorithm. 
Importantly, Meta-SGD reaches the minimum loss at step 95, which is the number
of steps in meta-training, and then begins to overfit. In contrast, \shrreptile{}
does not overfit due to its learned regularization.

\cref{fig:synthetic_exp1_MAE_per_module} further explains the different behavior of
the three algorithms. The mean absolute error (MAE) for each of the 8 parameter dimensions is
shown as a function of task adaptation step. MAML shares a single learning rate for all
parameters, and thus begins to overfit in different dimensions at different steps, resulting in worse performance. 
Meta-SGD is able to learn two groups of learning rates, one per each ground-truth
module. It learns to coordinate the two learning rates so that they reach the lowest error
at the same step, after which it starts to overfit. The learning rate is limited by the curvature enforced by the last observation dimension.
\shrreptile{} shares a single learning rate so the error of every dimension drops at about the same speed initially and each dimension reaches its minimum at different steps. It learns two groups of variances so that all parameters are properly regularized and maintain their error once converged, instead of overfitting. 

\subsection{Experiment 2: Nonlinear transform.}

\begin{figure}[tbh]
    \centering
    \begin{subfigure}{0.48\textwidth}
        \includegraphics[width=\textwidth]{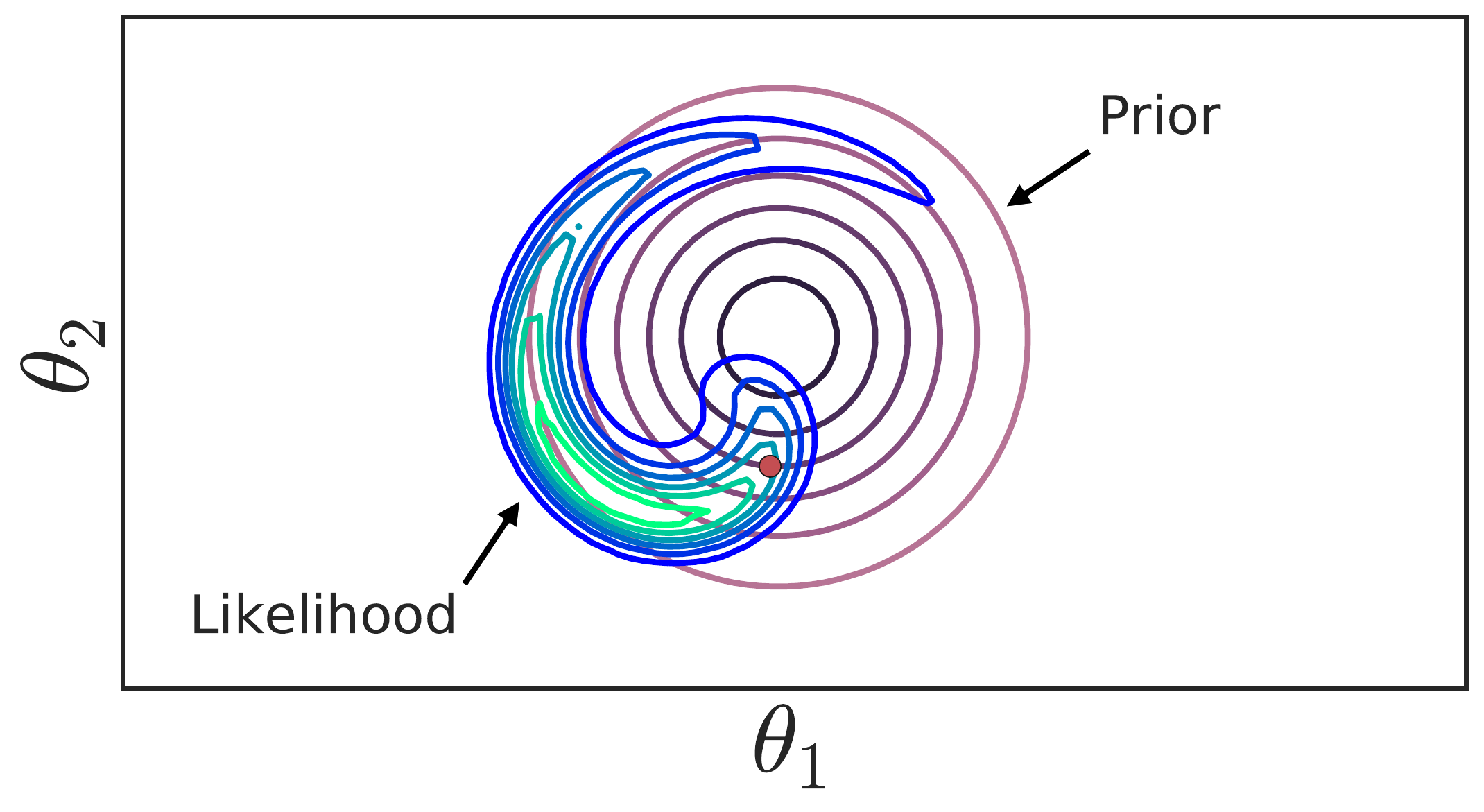}
            \caption{Illustration of the hierarchical normal distribution with non-linear $\mu$.}
    \label{fig:syn_exp_2_theta}
    \end{subfigure}
        \hspace{0.02\textwidth}
      \begin{subfigure}{0.48\textwidth}
        \includegraphics[width=\textwidth]{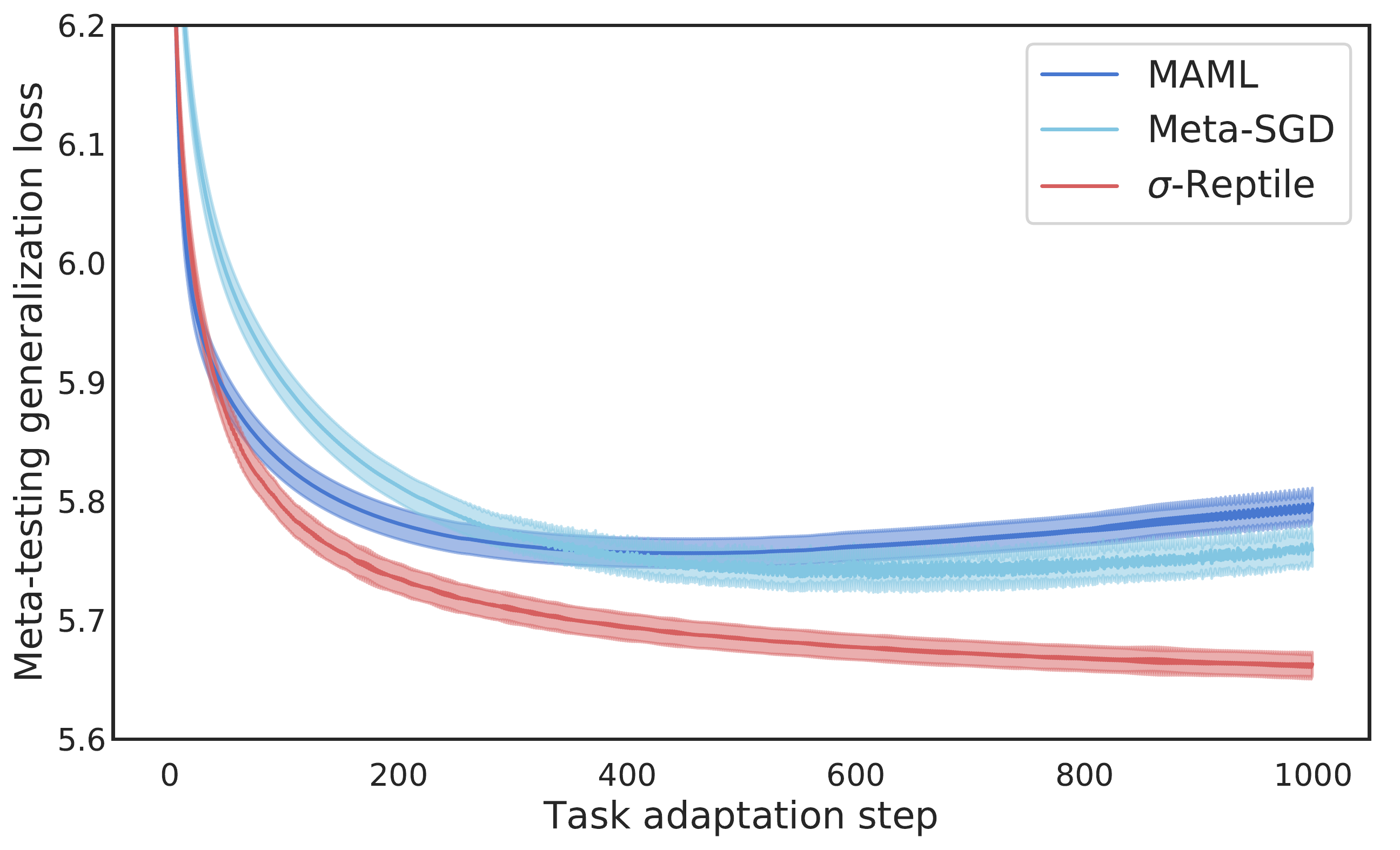}
    \vspace{-1em}
    \caption{Generalization loss in meta-testing.}
    \label{fig:syn_exp_2_loss}
    \end{subfigure}
    \caption{Experiment 2: Nonlinear spiral transform.}
\label{fig:syn_exp_2}
\end{figure}

In the second experiment, we explore a more challenging optimization scenario
where the mean is a nonlinear transformation of the task parameters.
Specifically,
\eq{
M=& 10, \nonumber \\
D=& 10, \nonumber \\
\vsigma =& [4,4,4,4,4,4,4,4,8,8],  \nonumber \\
\vphi =& 2 \cdot \vone_M, \nonumber \\
\vxi =& 10\cdot\vone_D. \nonumber 
}
The true modules are $\vtheta_{1:8}$ and $\vtheta_{9:10}$.
The transformation $\vmu_t(\vtheta_t)$ is a ``swirl'' effect 
that rotates non-overlapping pairs of consecutive parameters
with an angle proportional to their $L_2$ distance from the origin.
Specifically, each consecutive non-overlapping pair $(\mu_{t,d}, \mu_{t, d+1})$ is defined as
\eq{
\begin{bmatrix}
\mu_{t,d}\\
\mu_{t,d+1}
\end{bmatrix} &= \mathrm{Rot}\left(\omega \sqrt{\theta_{t,d}^2 + \theta_{t,d+1}^2}\right)
\cdot
\begin{bmatrix}
\theta_{t,d}\\
\theta_{t,d+1}
\end{bmatrix}, \quad \text{for~~} d=1, 3, ..., M-1, \nonumber
}
where
\eq{
\mathrm{Rot}(\varphi) = \begin{bmatrix}
\cos \varphi & -\sin \varphi\\
\sin \varphi & \cos \varphi\\
\end{bmatrix},
}
and $\omega = \nicefrac{\pi}{5}$ is the angular velocity of the rotation. 
This is a nonlinear volume-preserving mapping that forms a spiral in the 
observation space. \cref{fig:syn_exp_2_theta} shows an example of the prior
and likelihood function in $2$-dimensions of the parameter space.

Compared to the previous example, this highly nonlinear loss surface with
changing curvature is more realistic. First-order
optimizers are constrained by the narrow valley formed
by the ``swirly'' likelihood function, and thus all algorithms require hundreds of 
adaptation steps to minimize the loss.

In this case, the best per-parameter learning rate learned by Meta-SGD
is restricted by the highest curvature in the \textbf{likelihood} function along the
optimization trajectory.
In contrast, the optimal per-parameter variance estimated by \shrreptile{} depends
on the \textbf{prior} variance $\sigma_m^2$ in the data generating process, regardless
of the value of $\omega$, given that optimization eventually converges. 

As a consequence, Meta-SGD and \shrreptile{} exhibit very different behaviors in their
predictive performance in \cref{fig:syn_exp_2_loss}. Meta-SGD overfits after
about 700 steps, while \shrreptile{} keeps improving after 1000 steps.

Also, because MAML and Meta-SGD require backpropagation through the adaptation process,
when the number of adaptations steps is higher than 100, we notice that meta-training
becomes unstable. As a result, the best MAML and Meta-SGD hyperparameter choice 
has fewer than 100 adaptation steps in meta-training. These methods then fail to generalize to the longer
optimization horizon required in this problem.

We do not show the per-parameter MAE trajectory as in the previous section because 
this optimization moves through a spiraling, highly coupled trajectory
in the parameter space, and thus per-parameter MAE is not a good metric to measure the
progress of optimization.

\section{Algorithm Implementation Details}
\label{sec:alg_details}

\subsection{Implementation details for iMAML}

In this work, we implement and compare our algorithm to iMAML-GD --- the version of iMAML that uses gradient descent within task adaptation \citep[Sec.\ 4]{rajeswaran2019meta} --- as this better matches the proximal gradient descent optimizer used in our shrinkage algorithms.

In our implementation of conjugate gradient descent, to approximate the inverse Hessian we apply a damping term following \citet{rajeswaran2019meta} and restrict the number of conjugate gradient steps to 5 for numerical stability. The meta update is then
\eq{
\Delta_t^{\text{iMAML}} = 
\Big((1 + d) \vI + \tfrac{1}{\lambda} \nabla^2_{\vtaskv_t}\elltr(\vtaskv_t) 
\Big)^{-1} \nabla_{\vtaskv_t} \ellval(\vtaskv_t),
\label{eq:imaml_damped}
}
where $d$ is the damping coefficient. We treat $d$ and the number of conjugate gradient steps as hyperparameters and choose values using hyperparameter search.

\subsection{Implementation details for shrinkage prior algorithms}

As in iMAML, we apply damping to the Hessian matrix when running conjugate gradient descent to approximate the product $\nabla_{\vtaskv_t} \ell_t^\val(\hat\vtaskv_t) \vH_{\hat\vtheta_t \hat\vtheta_t}^{-1}$ in \shrimaml{} (Eq.~\ref{eq:valid_gd_v1}) and \shrreptile{} (Eq.~\ref{eq:valid_gd_v2}), 
\eq{
\vH_{\hat\vtheta_t \hat\vtheta_t} = -\tilde{d} \vI - \vSigma^{-1} - 
\nabla_{\hat\vtheta_t \hat\vtheta_t} \log p\left(\D_{t}^\tr|\vtheta_{t}\right),
}
where $\vSigma = \mathrm{Diag}(\sigma_1^2 \vI_{D_1}, \sigma_2^2 \vI_{D_2}, \dots, \sigma_M^2 \vI_{D_M})$, and $\mathrm{Diag}(\dots, \vB_m, \dots)$ denotes a block diagonal matrix with $m$-th block $\vB_m$. Note that the damped update rule reduces to that of iMAML when $\sigma_m^2 = 1/\lambda, \forall m$ and $\tilde{d} = d\lambda$.

Additionally, we apply a diagonal pre-conditioning matrix in the same structure as $\vSigma$, $\vP = \mathrm{Diag}(p_1 \vI_{D_1}, p_2 \vI_{D_2}, \dots, p_M \vI_{D_M})$ with $p_m = \max\{\sigma_m^{-2}/10^3, 1\}$ to prevent an ill-conditioned Hessian matrix when the prior variance becomes small (strong prior).

We also clip the value of $\sigma_m^2$ to be in $[10^{-5}, 10^{5}]$. This clipping improves the stability of meta learning, and the range is large enough to not affect the MAP estimate of $\hat{\vtaskv}_t$. 

Finally, we incorporate a weak regularizer on the shrinkage variance to encourage sparsity in the discovered adaptable modules. The regularized objective for learning $\vsigma^2$ becomes
\eq{
\frac{1}{T} \ell_\pll + \beta \log \mathrm{IG}(\vsigma^2) \,,
}
where $\mathrm{IG}$ is the inverse Gamma distribution with shape $\alpha=1$ and scale $\beta$. Unless otherwise stated, we use $\beta=10^{-5}$ for sinusoid and image experiments, and $\beta=10^{-7}$ for text-to-speech experiments. We find that this regularization simply reduces the learned $\sigma_m^2$ of irrelevant modules without affecting generalization performance.

\subsection{Proximal Gradient Descent and Proximal Adam with L2 regularization}
\label{sec:proximal_gd_and_adam}

The pseudo-code of Proximal Gradient Descent~\citep{singer2009proxgd} with an L2 regularization is presented in \cref{alg:prox-gd}. We also modify the Adam optimizer~\citep{kingma2014adam} to be a proximal method and present the pseudo-code in \cref{alg:prox-adam}.

\begin{figure}[tbh!]
\centering
\begin{minipage}{0.42\textwidth}
    \centering
    \begin{algorithm}[H]
        \SetAlgoLined
        \KwIn{Parameter $\vtheta_t$, gradient $\vg_t$, step size $\alpha_t$, regularization center $\vphi$, L2 regularization scale $\lambda$.} \;
        $\vtheta_{t+\frac{1}{2}} = \vtheta_t - \alpha_t \vg_t$ \;
        $\vtheta_{t+1} = (\vtheta_{t+\frac{1}{2}} - \vphi)/(1 + \lambda \alpha_t) + \vphi$ \;
        \Return $\vtheta_{t+1}$
    \caption{Proximal Gradient Descent with L2 Regularization.}
    \label{alg:prox-gd}
    \end{algorithm}    
\end{minipage}
~
\begin{minipage}{0.52\textwidth}
    \centering
    \begin{algorithm}[H]
        \SetAlgoLined
        \KwIn{Parameter $\vtheta_t$, gradient $\vg_t$, step size $\alpha_t$, regularization center $\vphi$, L2 regularization scale $\lambda$, $\epsilon$ for Adam.} \;
        $\vtheta_{t+\frac{1}{2}}, \hat{\vv}_{t+1} = \mathrm{Adam}(\vtheta_t, \vg_t, \alpha_t)$ \;
        $\vtheta_{t+1} = (\vtheta_{t+\frac{1}{2}} - \vphi) / (1 + \lambda \alpha_t / \sqrt{\hat{\vv}_{t+1}+\epsilon}) + \vphi$ \;
        \Return $\vtheta_{t+1}$
    \caption{Proximal Adam with L2 Regularization.}
    \label{alg:prox-adam}
    \end{algorithm}    
\end{minipage}
\end{figure}

\section{Experiment Details and Additional Short Adaptation Experiments}
\label{sec:apx_experiment}

In our experiments, we treat each layer (e.g., the weights and bias for a convolutional layer are a single layer) as a module, including batch-normalization layers which we adapt as in
previous work~\citep{Finn2017}.
Other choices of modules are straightforward but we leave exploring these for future work.
For the chosen hyperparameters, we perform a random search
for each experiment and choose the setting that performs
best on validation data.

\subsection{Augmented Omniglot experiment details}
\label{sec:apx_manyshot_omniglot}

We follow the many-shot Omniglot protocol of~\citet{flennerhag2019leap}, which takes the 
$46$ Omniglot alphabets that have $20$ or more character classes and creates one $20$-way classification 
task for each alphabet by sampling $20$ character classes from that alphabet. These classes are then
kept fixed for the duration of the experiment. Of the $46$ alphabets, $10$ are set aside as a held-out
test set, and the remainder are split between train and validation. 
The assignments of alphabets to splits
and of classes to tasks is determined by a random seed at the beginning of each experiment.
For each character class, there are $20$ images in total, $15$ of which are set aside 
for training (i.e., task adaptation) and $5$ for validation. This split is
kept consistent across all experiments.
All Omniglot images are downsampled to $28 \times 28$.
Note that this protocol differs significantly from the 
standard few-shot Omniglot protocol (discussed below),
where each task is created by selecting $N$ different characters (from any alphabet),
randomly rotating each character by $\{0, 90, 180, 270\}$ degrees, and then randomly selecting
$K$ image instances of that (rotated) character. 

At each step of task adaptation and when computing the validation loss, a batch of images is sampled
from the task. Each of these images is randomly augmented by re-scaling it by a factor 
sampled from $[0.8, 1.2]$, and translating it by a factor sampled from $[-0.2, 0.2]$.
In the large-data regime, images are also randomly rotated by an angle sampled from $\{0, \dots, 359\}$ degrees.
In the small-data regime, no rotation is applied.

We use the same convolutional network architecture as \citet{flennerhag2019leap}, which
differs slightly from the network used for few-shot Omniglot (detailed below). Specifically, the few-shot Omniglot
architecture employs convolutions with a stride of 2 with no max pooling, whereas the 
architecture for many-shot Omniglot uses a stride of 1 with $2 \times 2$ max pooling.
In detail, the architecture for many-shot Omniglot consists of $4$ convolutional blocks,
each made up of a $3 \times 3$ convolutional layer with $64$ filters, a batch-normalization layer,
a ReLU activation, and a $2 \times 2$ max-pooling layer, in that order. The output of the final
convolutional block is fed into an output linear layer and then a cross-entropy loss.

\begin{table*}[htb!]
    \caption{Hyperparameters for the large-data augmented Omniglot classification experiment. Chosen with random search on validation task set.} 
    \label{table:aug_omniglot_hypers}
    \begin{center}
    \resizebox{\textwidth}{!}{        \begin{tabular}{lcccc}
            \toprule
            ~ & Reptile  & iMAML & \shrreptile{} & \shrimaml{} \\
            \midrule
            Meta-training \\
            \midrule
            ~~Meta optimizer & SGD & Adam & Adam & Adam \\
            ~~Meta learning rate ($\vmetav$) & 1.2 & 1.8e-3 & 6.2e-3 & 5.4e-3  \\
            ~~Meta learning rate ($\log \vsigma^2$) & - & - & 1.6e-2 & 0.5  \\
            ~~Meta training steps & 5k & 5k & 5k & 5k \\
            ~~Meta batch size (\# tasks) & 20 & 20 & 20 & 20 \\
            ~~Damping coefficient & - & 0.1 & 0.16 & 9e-2 \\
            ~~Conjugate gradient steps & - & 1 & 4 & 5 \\
            \midrule
                                    \multicolumn{5}{l}{Task adaptation (adaptation step for meta-test is in parentheses)} \\
            \midrule
            ~~Task optimizer & Adam & ProximalGD & ProximalGD & ProximalGD \\
            ~~Task learning rate & 9.4e-3 & 0.5 & 0.52 & 0.37 \\
            ~~Task adaptation steps  & 100 (100) & 100 (100) & 100 (100) & 100 (100) \\
            ~~Task batch size (\# images) & 20  & 20 & 20 & 20 \\
            \bottomrule
        \end{tabular}
    }
    \end{center}
\end{table*}

\subsubsection{Large-data regime}

In the large-data regime, we use $30$ alphabets for training, each with $15$ image instances
per character class. Hyperparameters for each algorithm in this domain are shown in
\cref{table:aug_omniglot_hypers}. For iMAML, we use $\lambda = 1.3$e$-4$.

\cref{fig:module_aug_omniglot} shows the learned variances and resulting test accuracies when adapting different modules. We find in this dataset all the shrinkage algorithms choose the last linear layer for adaptation, and the performance matches that of adapting all layers with learned variance within a $95\%$ confidence interval of 0.08.

\begin{table*}[htb!]
    \caption{Hyperparameters for the small-data augmented Omniglot experiment. Chosen with random search on validation task set.} 
    \label{table:small_data_omni_hypers}
    \begin{center}
        \begin{tabular}{lcccc}
            \toprule
            ~  & Reptile  & iMAML &  \shrreptile{} & \shrimaml{} \\
            \midrule
            Meta-training \\
            \midrule
            ~~Meta optimizer & SGD & Adam & Adam & Adam \\
            ~~Meta learning rate ($\vmetav$) & 0.1 & 1e-3 & 1.4e-4 & 2e-4 \\
            ~~Meta learning rate ($\log \vsigma^2$) & - & - & 5e-3 & 1.3 \\
            ~~Meta training steps & 1000 & 1000 & 1000 & 1000 \\
            ~~Meta batch size (\# tasks) & 20 & 20 & 20 & 20 \\
            ~~Damping coefficient & - & 0.5 & 1e-3 & 2e-3 \\
            ~~Conjugate gradient steps & - & 2 & 5 & 3 \\
            ~~Regularization scale ($\beta$) & - & - & 1e-6 & 1e-7 \\
            \midrule
                        \multicolumn{5}{l}{Task adaptation
                        } \\
            \midrule
            ~~Task optimizer & Adam & ProximalGD & ProximalGD & ProximalGD \\
            ~~Task learning rate & 4e-3 & 0.4 & \textit{(see \cref{table:small_data_omni_hypers_shrinkage})} & \textit{(see \cref{table:small_data_omni_hypers_shrinkage})} \\
            ~~Task adaptation steps & 100 & 100 & 100 & 100 \\
            ~~Task batch size (\# images) & 20 & 20 & 20 & 20 \\
            \bottomrule
        \end{tabular}
    \end{center}
\end{table*}

\begin{table*}[hbt!]
\caption{Task optimizer learning rate for different numbers of training instances per character class in small-data augmented Omniglot.}
\label{table:small_data_omni_hypers_shrinkage}
  \begin{center}
      \begin{tabular}{lccccc}
      \toprule
      \# train instances & 1 & 3 & 5 & 10 & 15 \\
      \midrule
      \shrreptile{} task LR & 0.4 & 0.31 & 0.22 & 0.12 & 0.03 \\
      \shrimaml{} task LR & 0.25 & 0.17 & 0.14 & 0.1 & 0.05 \\
      \bottomrule
      \end{tabular}
  \end{center}
\end{table*}

\begin{figure*}[hbt!]
\centering
\begin{subfigure}{.45\textwidth}
\centering
\includegraphics[width=.95\textwidth]{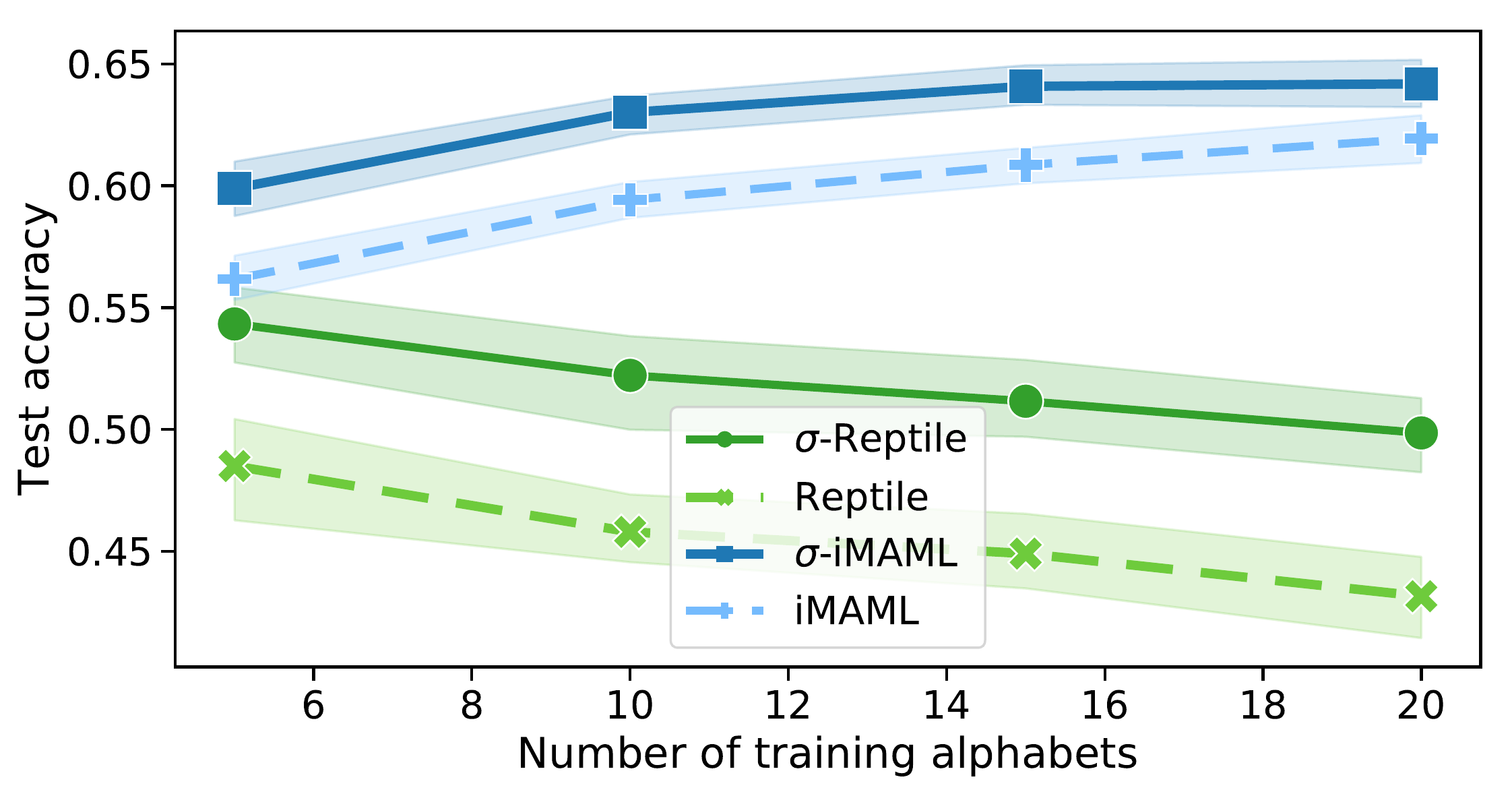}
\caption{$1$ image instance per character.}
\end{subfigure}
~
~
\begin{subfigure}{.45\textwidth}
\centering
\includegraphics[width=.95\textwidth]{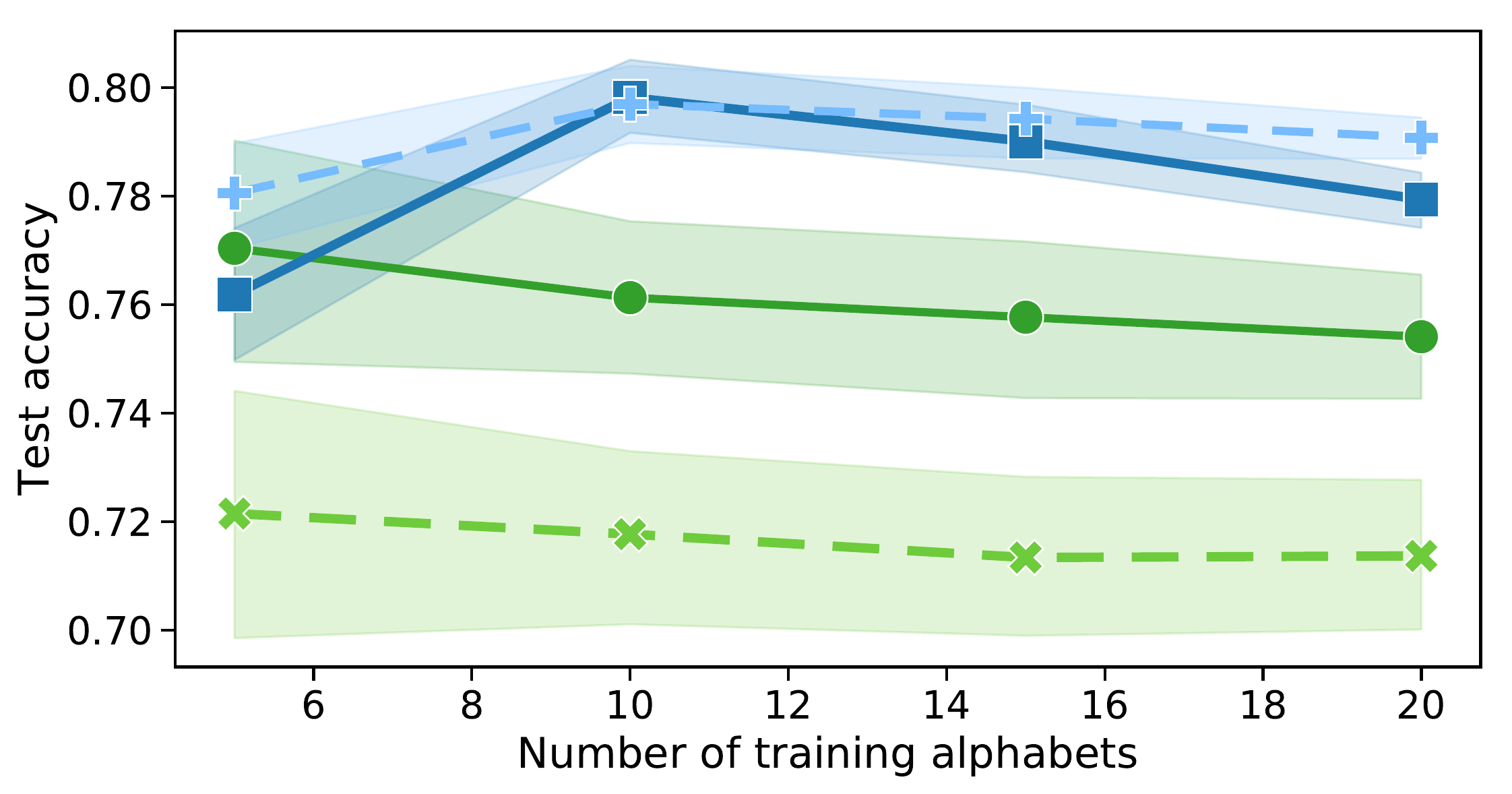}
\caption{$3$ image instance per character.}
\end{subfigure}
\\
\begin{subfigure}{.45\textwidth}
\centering
\includegraphics[width=.95\textwidth]{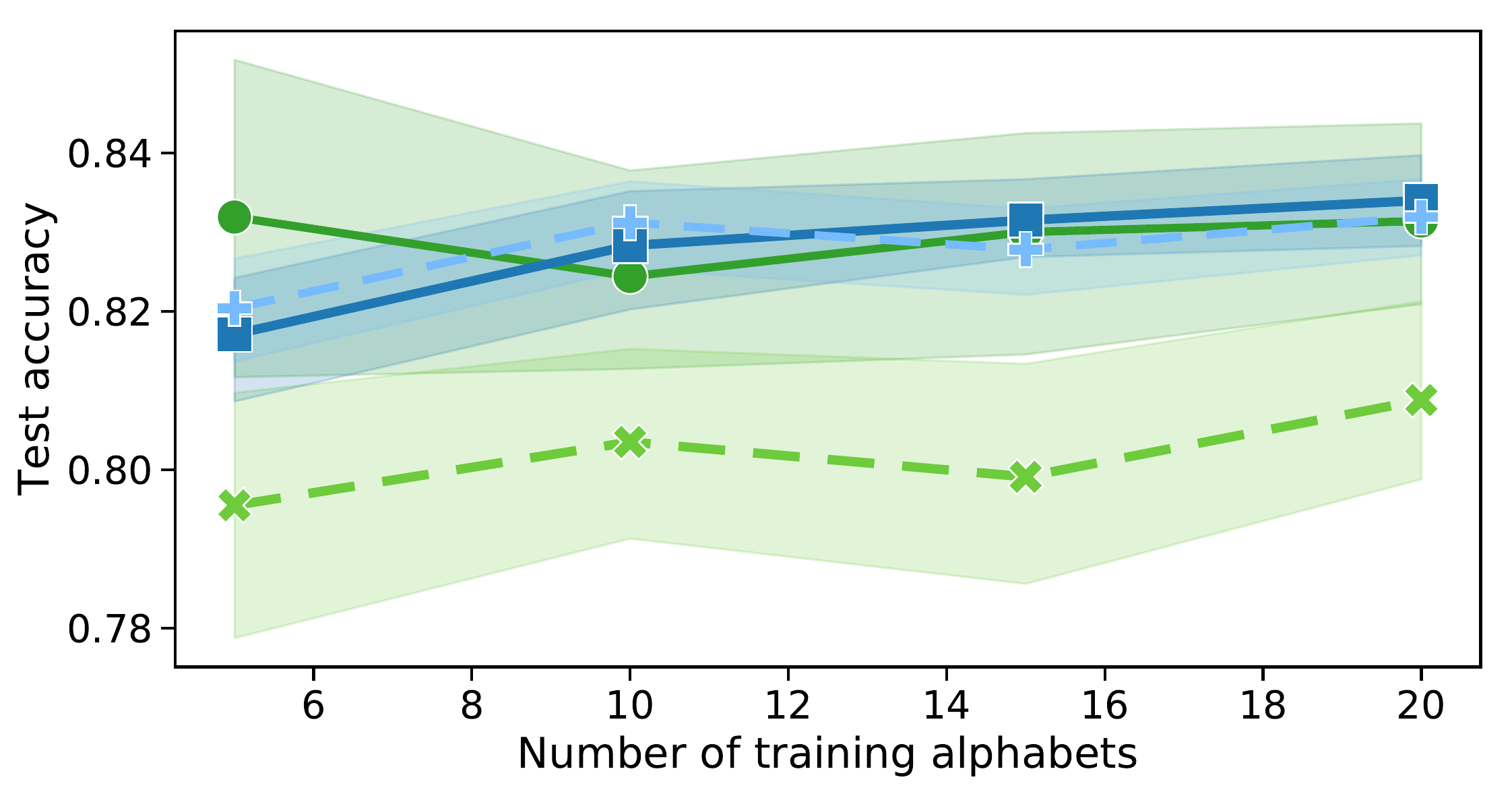}
\caption{$5$ image instance per character.}
\end{subfigure}
~
~
\begin{subfigure}{.45\textwidth}
\centering
\includegraphics[width=.95\textwidth]{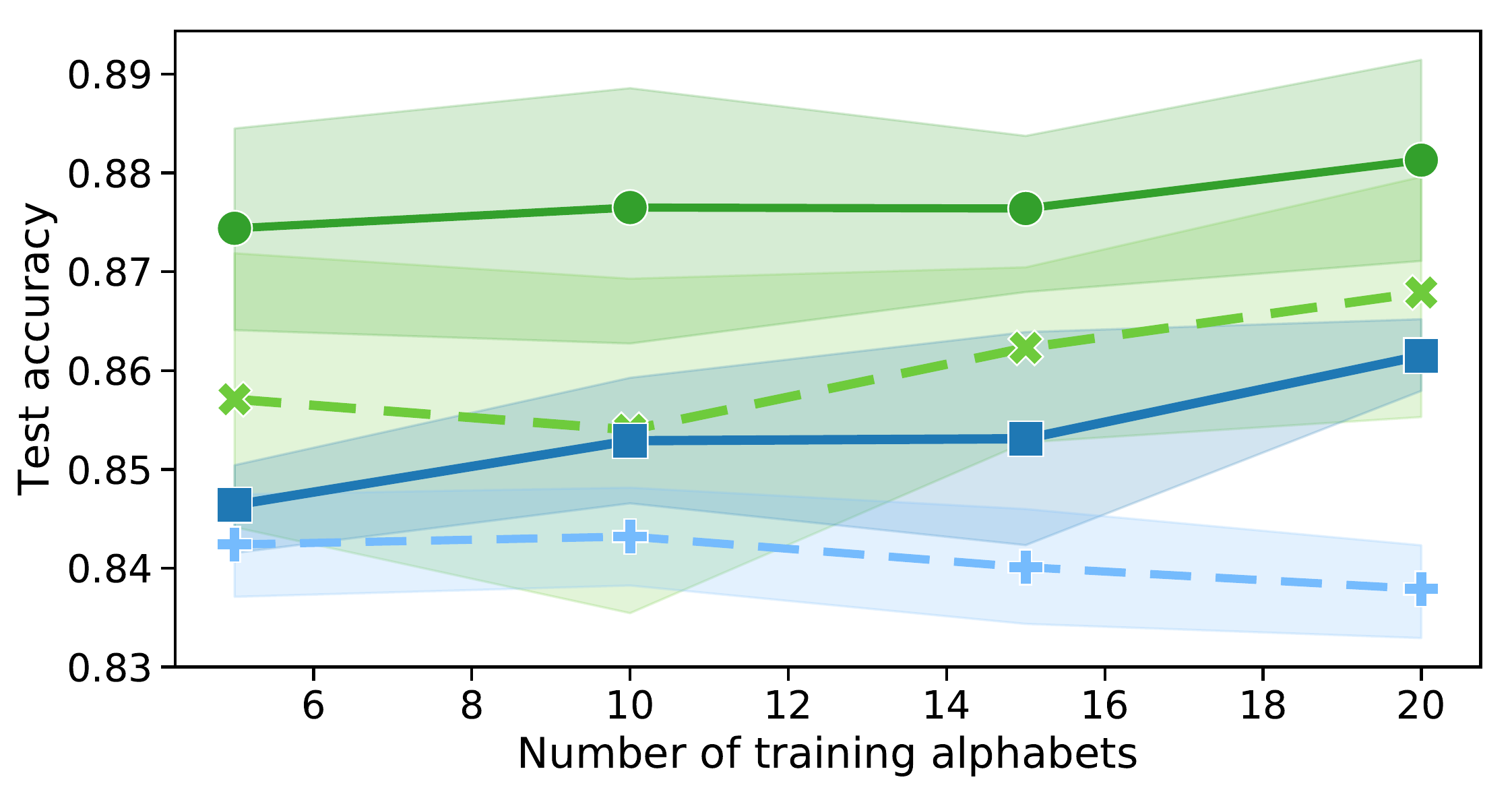}
\caption{$10$ image instance per character.}
\end{subfigure}
\\
\centering
\begin{subfigure}{.45\textwidth}
\centering
\includegraphics[width=.95\textwidth]{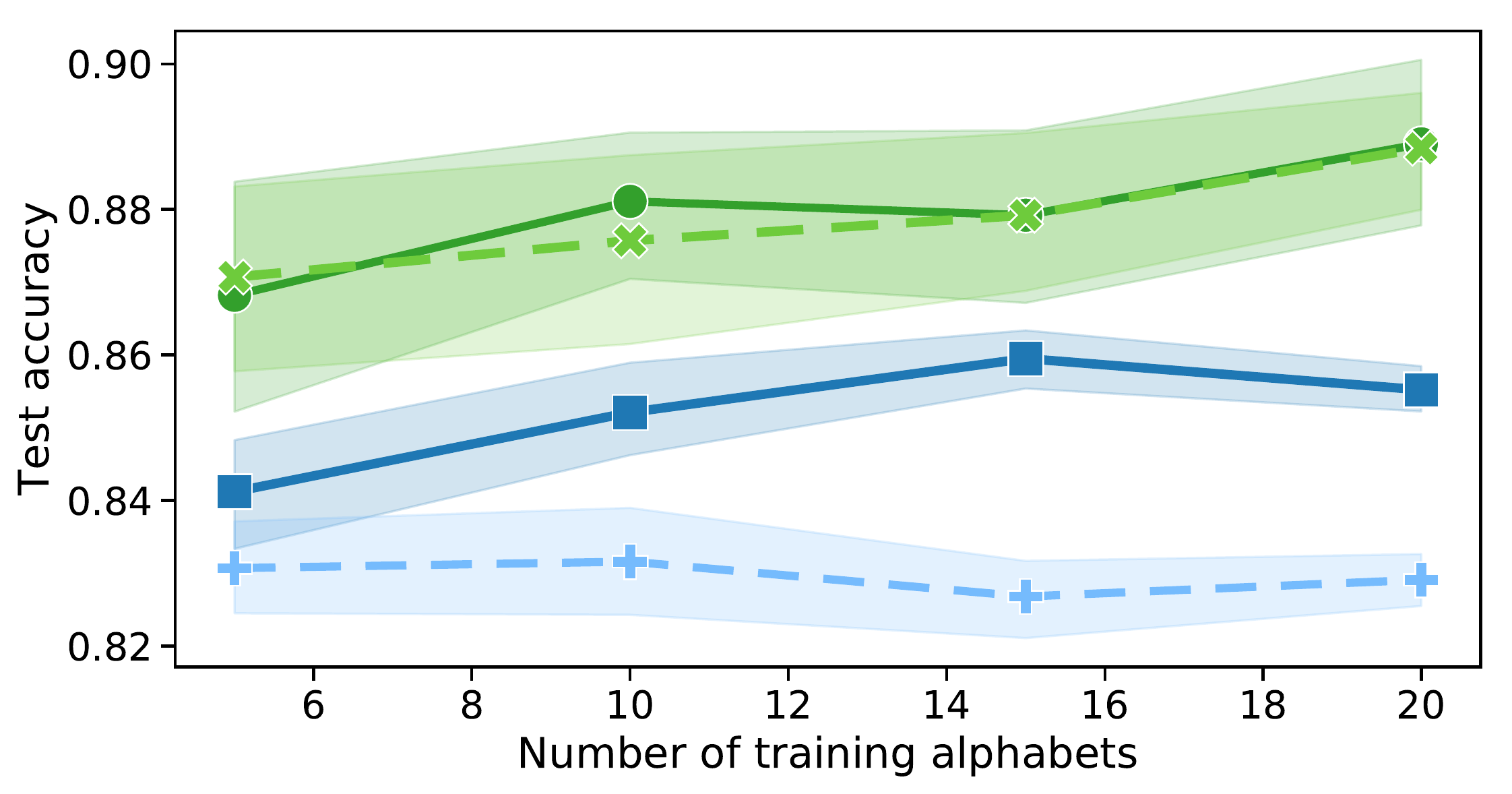}
\caption{$15$ image instance per character.}
\end{subfigure}
\caption{Test accuracy on augmented Omniglot in the small-data regime as a function of the number of training alphabets. In each plot, the number of instances per class is fixed. Each data point is the average of $10$ runs and $95\%$ confidence intervals are shown.}
\label{fig:aug_omni_reduced_instances_appendix}
\end{figure*}

\begin{figure*}[hbt]
\centering
\begin{subfigure}{.45\textwidth}
\centering
\includegraphics[width=.95\textwidth]{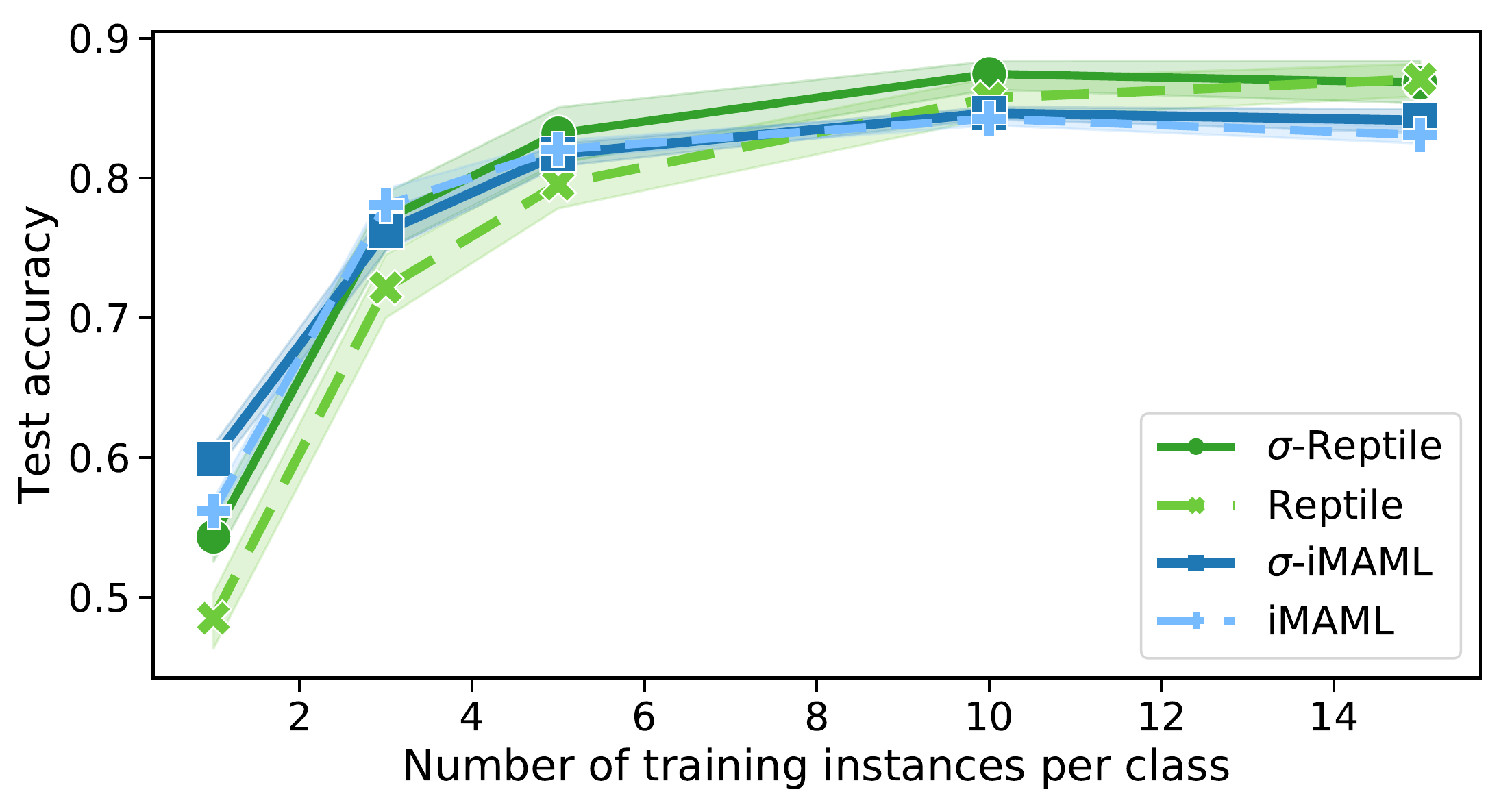}
\caption{$5$ training alphabets.}
\end{subfigure}
~
~
\begin{subfigure}{.45\textwidth}
\centering
\includegraphics[width=.95\textwidth]{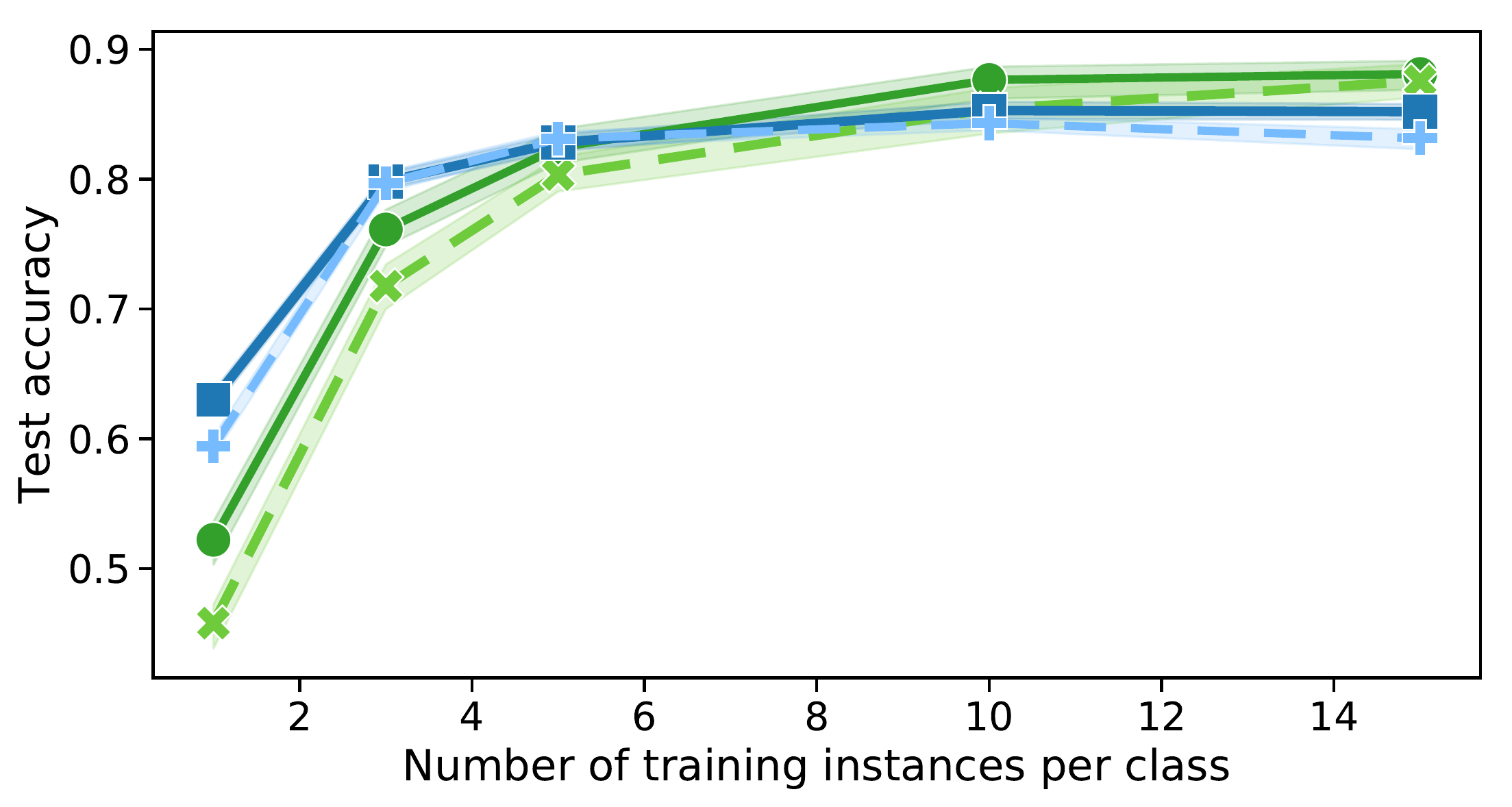}
\caption{$10$ training alphabets.}
\end{subfigure}
\\
\begin{subfigure}{.45\textwidth}
\centering
\includegraphics[width=.95\textwidth]{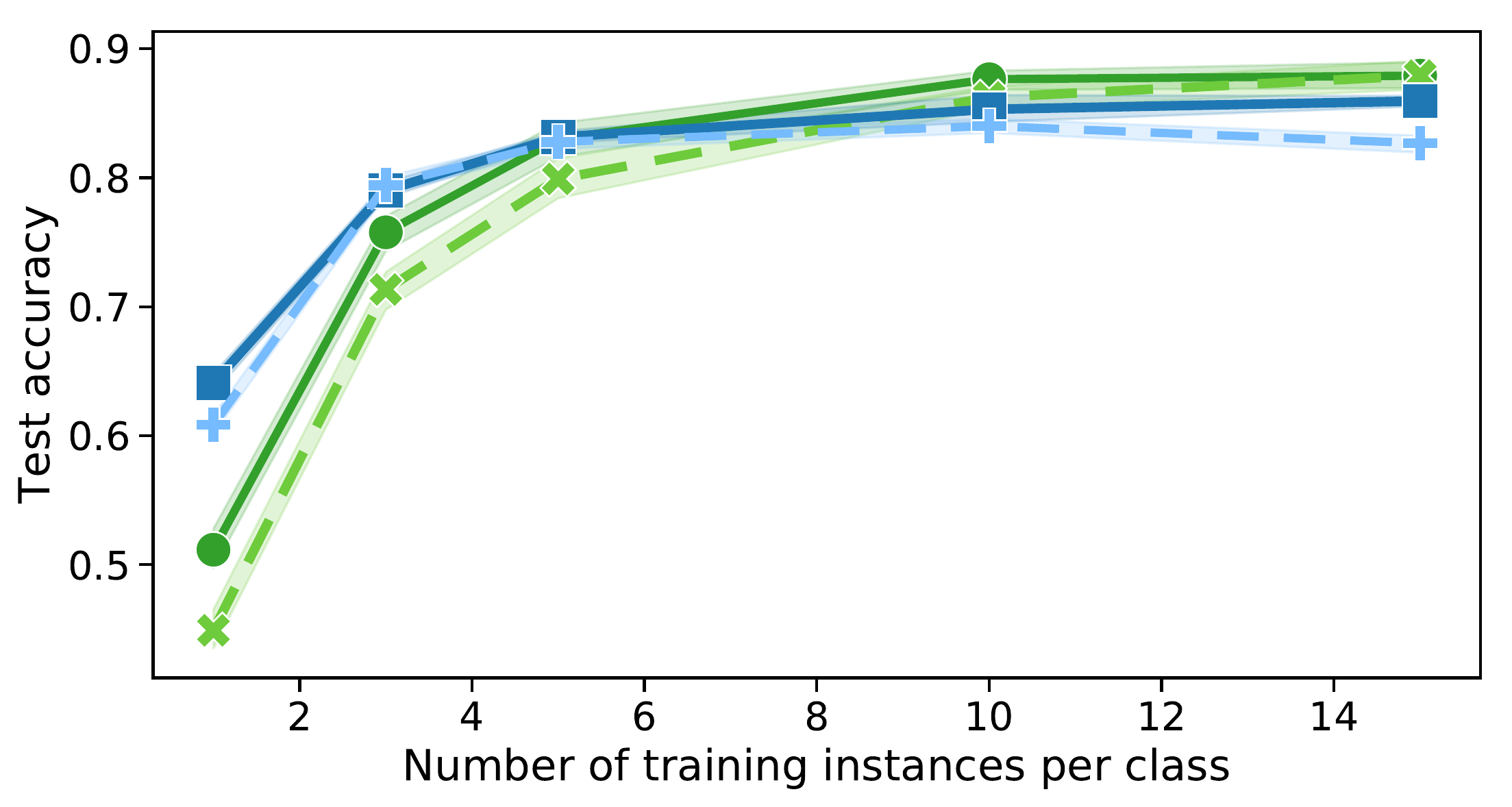}
\caption{$15$ training alphabets.}
\end{subfigure}
~
~
\begin{subfigure}{.45\textwidth}
\centering
\includegraphics[width=.95\textwidth]{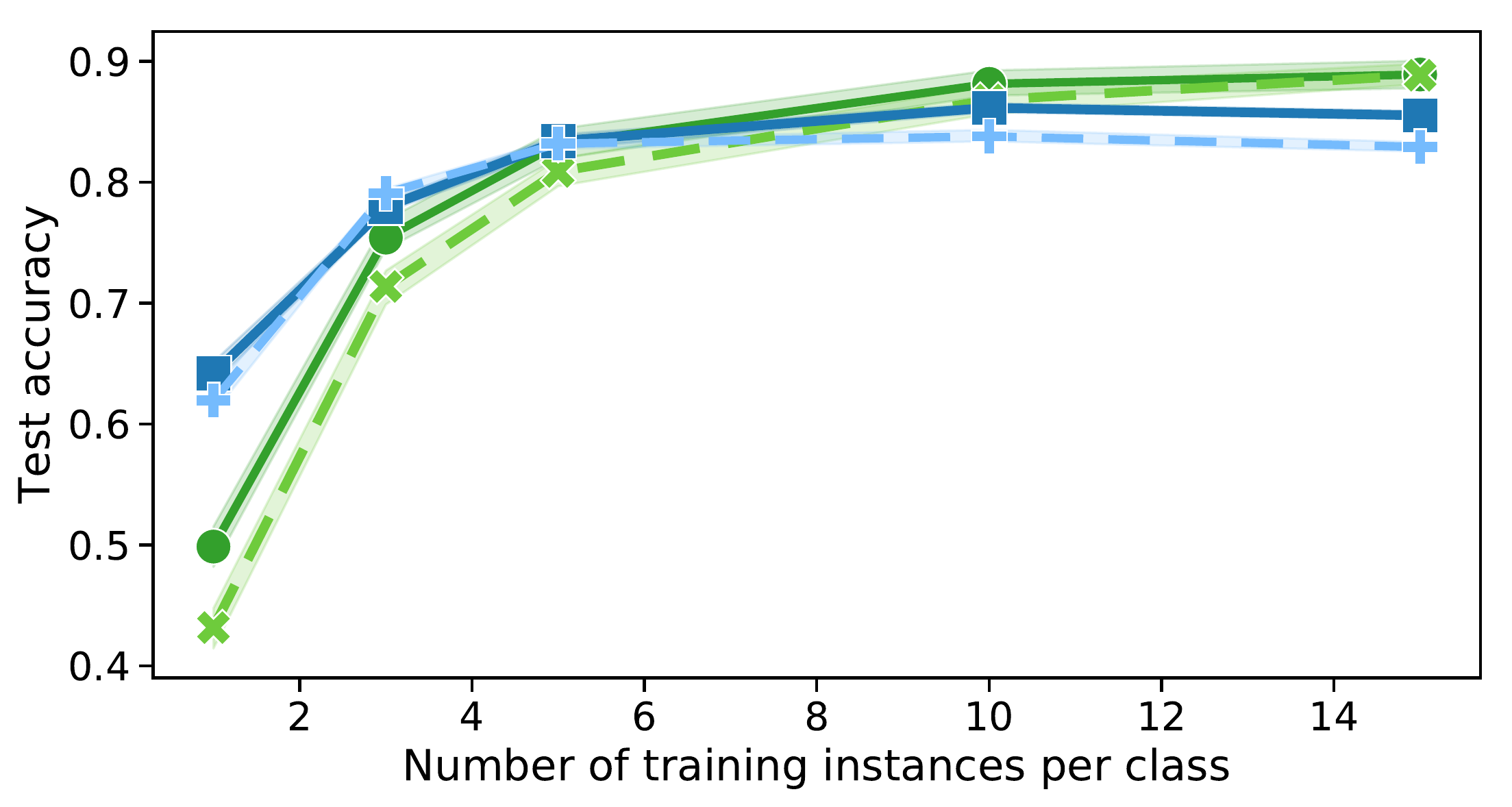}
\caption{$20$ training alphabets.}
\end{subfigure}
\caption{Test accuracy on augmented Omniglot in the small-data regime as a function of the number of training instances. In each plot, the number of training alphabets is fixed. Each data point is the average of $10$ runs and $95\%$ confidence intervals are shown.}
\label{fig:aug_omni_reduced_alphabets_appendix}
\end{figure*}

\subsubsection{Small-data regime}

In the small-data regime, we evaluate the performance of Reptile, iMAML,
\shrreptile{}, and \shrimaml{} across variants of the many-shot augmented Omniglot 
task with different numbers of training alphabets and training image instances per character class,
without rotation augmentation.
We train each algorithm $10$ times with each of $5, 10, 15,$ and $20$ training alphabets and 
$1, 3, 5, 10,$ and $15$ training instances per character class.

We meta-train all algorithms for $1000$ steps using $100$ steps of adaptation per task. We use a meta-batch size of $20$, meaning that the same task can appear multiple times within a batch, although the
images will be different in each task instance due to data augmentation. At meta-evaluation time, we again perform $100$ steps of task adaptation on the task-train instances of each test alphabet, and report the accuracy on the task-validation instances.

Hyperparameters for each of the $4$ algorithms that we evaluate on 
small-data augmented Omniglot are listed in 
\cref{table:small_data_omni_hypers}. 
These were chosen based on a comprehensive hyperparameter search on separate validation data.
For Reptile, we use a linear learning rate decay that reaches $0$ at the end of meta-training, as in~\citet{Nichol2018}. 
For iMAML, we use an L2 regularization scale of $\lambda = 1/\sigma^2 = 1$e$-3$.
For \shrimaml{} and \shrreptile{}, we use a different task optimizer learning rate for different numbers of 
instances. These learning rates are presented in \cref{table:small_data_omni_hypers_shrinkage}.

Due to the space limit, we show results with 1 training instance per character class and results with 15 training alphabets in 
\cref{fig:aug_omni_reduced} of the main text.
The full results of the $20$ experimental conditions are presented in
\cref{fig:aug_omni_reduced_instances_appendix,fig:aug_omni_reduced_alphabets_appendix}, which respectively
show the performance of each method as the number of training instances and alphabets vary.
As discussed in the main text, each shrinkage variant outperforms or matches its corresponding non-shrinkage variant
in nearly every condition. Improvements by the shrinkage variants are more consistent and pronounced 
when the amount of training data is limited.

\subsection{Few-shot Text-to-Speech experiment details}
\label{sec:apx_tts}

\begin{figure}[tbh]
\centering
\includegraphics[width=.95\textwidth]{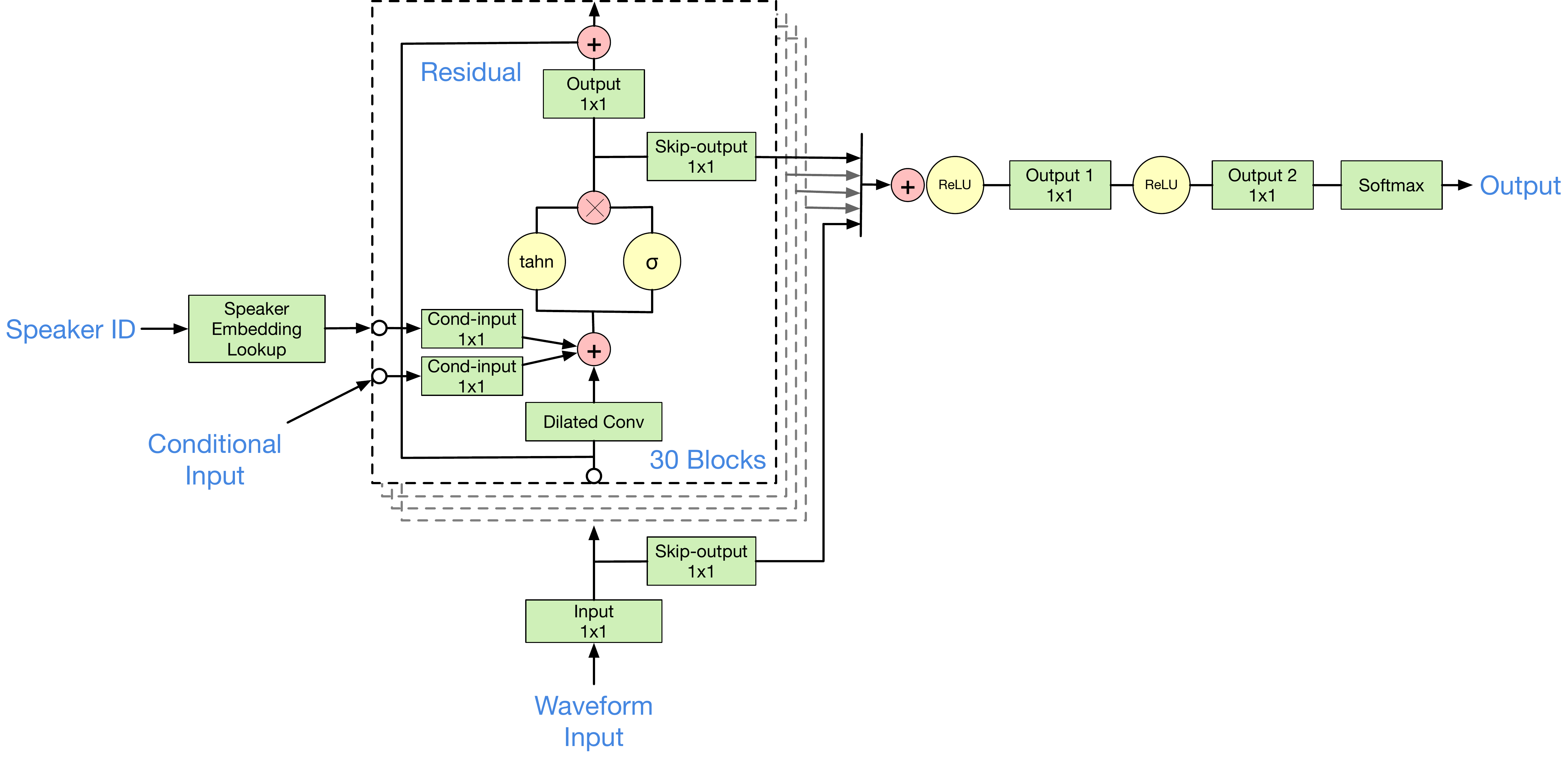}
\caption{Architecture of the multi-speaker WaveNet model. The single-speaker WaveNet model used by Reptile and \shrreptile{}, does not include the speaker embedding lookup table and the corresponding conditional input layer.}
\label{fig:wavenet}
\end{figure}

For few-shot text-to-speech synthesis, various works \citep{arik2018neural,chen2019sample,jia2018transfer,taigman2018voiceloop} 
have made use of speaker-encoding networks or trainable speaker embedding vectors to adapt to a new voice based on a small amount of speech data. These works achieved success to some extent when there were a few training utterances, but the performance saturated quickly beyond 10 utterances \citep{jia2018transfer} due to the bottleneck in the speaker specific components. \citet{arik2018neural} and \citet{chen2019sample} found that the performance kept improving with more utterances by fine-tuning the entire TTS model, but the adaptation process had to be terminated early to prevent overfitting. As such, some modules may still be underfit while others have begun to overfit, similar to the behavior seen for MAML in \cref{fig:synthetic_exp1_MAE_per_module}.

In this paper, we examine the advantages of shrinkage for a WaveNet model \citep{van2016wavenet}. The same method applies to other TTS architectures as well. In preliminary experiments, we found iMAML and \shrimaml{} meta-learn much more slowly than Reptile and \shrreptile{}. We conjecture that this is because iMAML and \shrimaml{} compute meta-gradients based only on the validation data from the last mini-batch of task adaptation. In contrast, the meta update for $\vmetav$ from Reptile and \shrreptile{} accumulates the task parameter updates computed from training mini-batches through the adaptation process. With a task adaptation horizon of 100 steps, this leads to significantly different data efficiencies. As a result, we only evaluate Reptile and \shrreptile{} for this experiment.

The WaveNet model is an augoregressive generative model. At every step, it takes the sequence of waveform samples generated up to that step, and the concatenated fundamental frequency (f0) and linguistic feature sequences as inputs, and predicts the sample at the next step. The sequence of fundamental frequency controls the dynamics of the pitch in an utterance. The short-time frequency is important for WaveNet to predict low level sinusoid-like waveform features. The linguistic features encode the sequence of phonemes from text. They are used by WaveNet to generate speech with corresponding content. The dynamics of the fundamental frequency in f0 together with the phoneme duration contained in the linguistic feature sequence contains important information about the prosody of an utterance (word speed, pauses, emphasis, emotion, etc), which change at a much slower time-scale. While the fundamental frequency and prosody in the inputs contain some information about the speaker identity, the vocal tract characteristics of a voice that is unique to each speaker cannot be inferred from the inputs, and has to be learned by the WaveNet model from waveform samples through task adaptation.

The full architecture of a multi-speaker WaveNet model used by SEA-Emb and SEA-All in \citet{chen2019sample} is shown in \cref{fig:wavenet}. For Reptile and \shrreptile{}, we use a single-speaker model architecture that excludes the speaker embedding lookup table and associated conditional input layers in the residual blocks. The single-speaker model is comprised of one input convolutional layer with $1\times 1$ kernel, one input skip-output $1\times 1$ layer, 30 residual of dilated causal-convolutional blocks (each including 4 layers), and 2 $1\times 1$ output layers before feeding into a 256-way softmax layer. We treat every layer as a module, for a total of 123 modules (the output layer of the last block is not used for prediction).

To speed up meta-learning, we first pretrain a single-speaker model with all training speaker data mixed, and initialize \shrreptile{} and Reptile with that pretrained model. This is reminiscent of other meta-learning works for few-shot image classification that use a pretrained multi-head model to warm-start when the backbone model (e.g. ResNet) is large. This pretrained TTS model learns to generate speech-like samples but does not maintain a consistent voice through a single sentence In contrast to other meta-learning works that fix the feature extracting layers, we then meta-learn all WaveNet layers to identify most task-specific modules.

We run the multispeaker training and pre-training for meta-learning methods for 1M steps on a proprietary dataset of 61 speakers each with 2 hours of speech data with 8-bit encoding sampled at 24K frequency. For \shrreptile{} and Reptile, we further run meta-learning for 8K meta-steps with 100 task adaptation steps in the inner loop for the same set of speakers. Each task consists of 100 utterances (roughly 8 minutes of speech). We evaluate the model on 10 holdout speakers, in two data settings, with 100 utterances or 50 utterances (roughly 4 minutes of speech) per task. We run up to 10,000 task adaptation steps in meta-test. SEA-All and Reptile both finetune all parameters. They overfit quickly in meta-testing after about 1,500 and 3,000 adaptation steps with 4-min and 8-min of data, respectively. We therefore early terminate task adaptation for these algorithms to prevent overfitting.

To measure the sample quality of naturalness, we have human evaluators 
rate each sample on a five-point Likert Scale (1:~Bad, 2:~Poor, 3:~Fair, 4:~Good, 5:~Excellent) and then we compute the mean opinion score (MOS). This is the standard approach to evaluate the quality of TTS models.

Voice similarity is computed as follows. We use a pretrained speaker verification model \citep{wan2017generalized} that outputs an embedding vector, $d(x)$, known as a $d$-vector, for each utterance $x$. We first compute the mean of $d$-vectors from real utterances of each test speaker, $t$.
$\bar{d}_t := \sum_n d(x_{t,n}) / N_t$. Given a model adapted to speaker $t$, we compute the sample similarity for every sample utterance $x_i$ as
\eq{
\mathrm{sim}(x_i, t) = \mathrm{cos}(d(x_i), \bar{d}_t) \,.\nn
}

\subsection{Additional short adaptation experiment: sinusoid regression}
\label{sec:apx_sin}

We follow the standard sinusoid regression protocol of \citet{Finn2017} in which each task consists of regressing input to output of a sinusoid $y = a\sin(x - b)$ uniformly sampled with amplitude $a\in[0.1, 5]$ and phase $b\in[0, \pi]$. Each task is constructed by sampling $10$ labelled data points from input range $x\in[-5, 5]$. We learn a regression function with a 2-layer neural network with 40 hidden units and ReLU nonlinearities and optimise the mean-squared error (MSE) between predictions and true output values.

Our method is agnostic to the choice of modules. For this small model, consider each set of network parameters as a separate module. In total, we define 6 modules: $\{b_i, w_i\}$, for $i=0,1,2$, where $b_i$ and $w_i$ denote the bias and weights of each layer. We run each shrinkage variant for 100K meta-training steps, and evaluate the generalization error on 100 holdout tasks. The hyperparameters of all three shrinkage algorithm variants are given in \cref{table:sine_hypers}.

We show the learned variance from \shrmaml{}, \shrimaml{} and \shrreptile{} in \cref{fig:sin}(\subref{fig:sin_pll_bp_var},\subref{fig:sin_pll_var},\subref{fig:sin_map_var}) respectively. In all experiments, we observe that the learned variances $\sigma_m^2$ for the first 2 modules ($b_0, w_0$) are significantly larger than the rest. 
This implies that our method discovers that these modules are task-specific and should change during adaptation whereas the other modules are task-independent. 

To confirm that the learned variances correspond to task-specificity, we adapt one layer at a time on holdout tasks and keep the other layers fixed. \cref{fig:sin} shows that adapting only the discovered first layer results in both low error and accurately reconstructed sinusoids, whereas adapting the other modules does not.

\begin{table*}[tbh!]
    \caption{Hyperparameters for the few-shot sinusoid regression experiment. Chosen with random search on validation task set.} 
    \label{table:sine_hypers}
    \begin{center}
        \begin{tabular}{lccc}
            \toprule
            ~  & \shrmaml &  \shrreptile{} & \shrimaml{} \\
            \midrule
            Meta-training \\
            \midrule
            ~~Meta optimizer & Adam & Adam & Adam \\
            ~~Meta learning rate ($\vmetav$) &  9.8e-4 & 3.0e-3 & 5.7e-3  \\
            ~~Meta learning rate ($\log \vsigma^2$) &  4.4e-3 & 1.4e-4 & 1.8e-3  \\
            ~~Meta training steps & 100k & 100k & 100k \\
            ~~Meta batch size (\# tasks)  & 5 & 5 & 5 \\
            ~~Damping coefficient & - & 6e-3 & 0.5 \\
            ~~Conjugate gradient steps & - & 7 & 1 \\
            \midrule
                        \multicolumn{3}{l}{Task adaptation
                        } \\
            \midrule
            ~~Task optimizer & ProximalGD & ProximalGD & ProximalGD \\
            ~~Task learning rate & 8.3e-4 & 1.4e-4 & 3.9e-4 \\
            ~~Task adaptation steps & 68 & 100 & 100 \\
            ~~Task batch size (\# data points) & 10 & 10 & 10 \\
            \bottomrule
        \end{tabular}
    \end{center}
\end{table*}

\begin{figure*}[tbh!]
\centering
\begin{subfigure}{.3\textwidth}
\includegraphics[width=\textwidth]{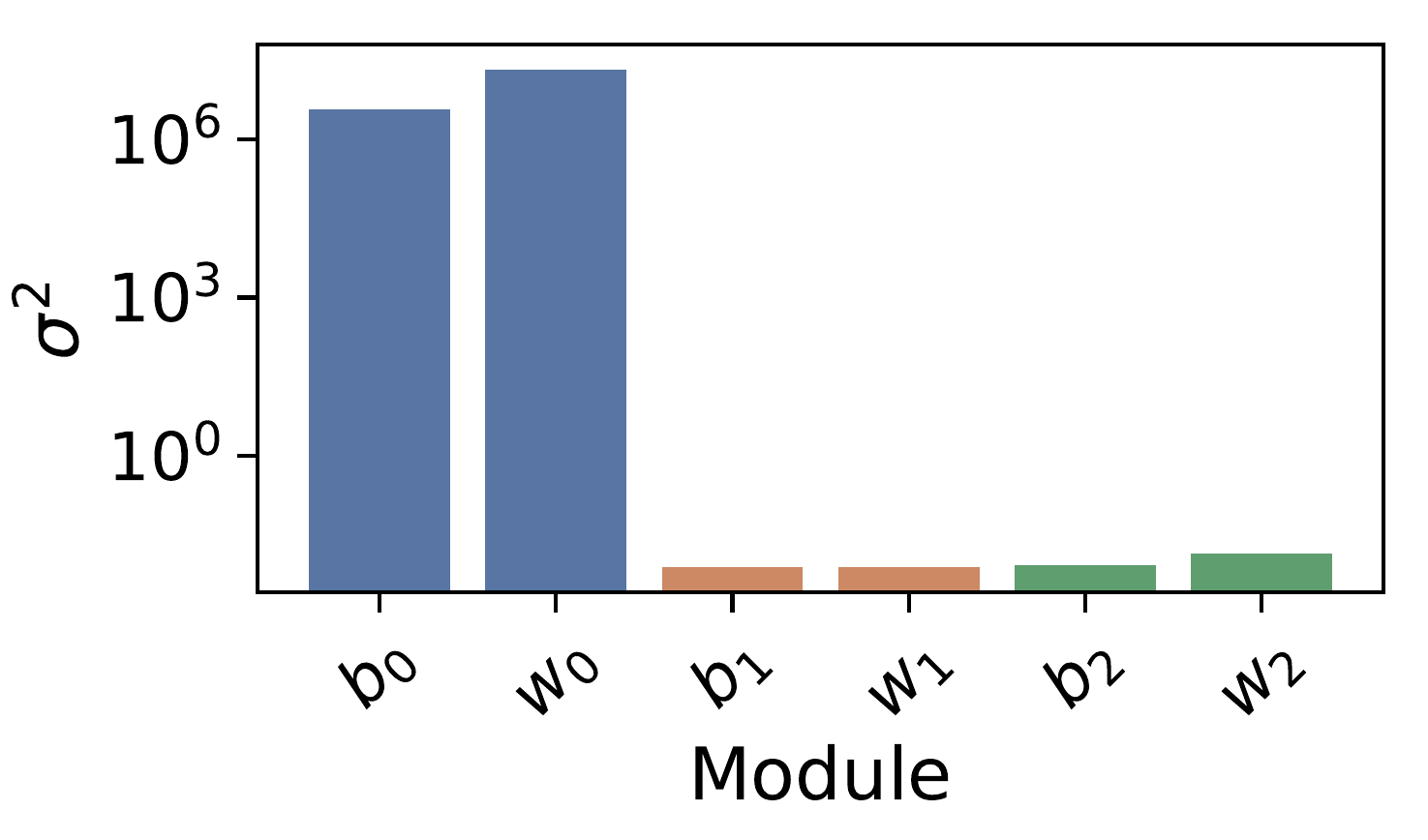}
\caption{Learned $\vsigma$ for each module (\shrmaml{}).}
\label{fig:sin_pll_bp_var}
\end{subfigure}
~
\begin{subfigure}{.3\textwidth}
\includegraphics[width=\textwidth]{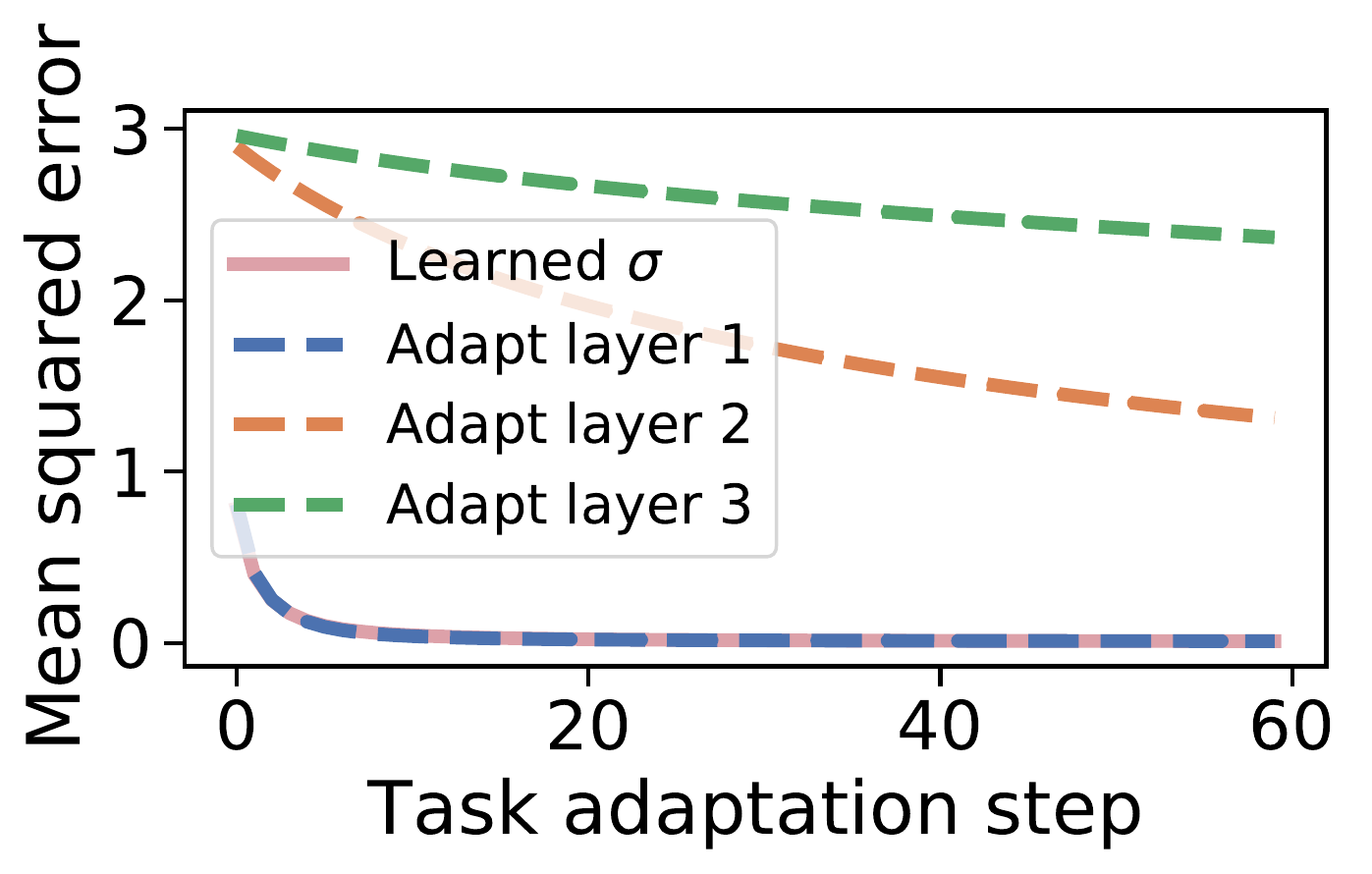}
\caption{MSE under different adaptation schemes.}
\label{fig:sin_pll_bp_loss}
\end{subfigure}
~
\begin{subfigure}{.3\textwidth}
\includegraphics[width=\textwidth]{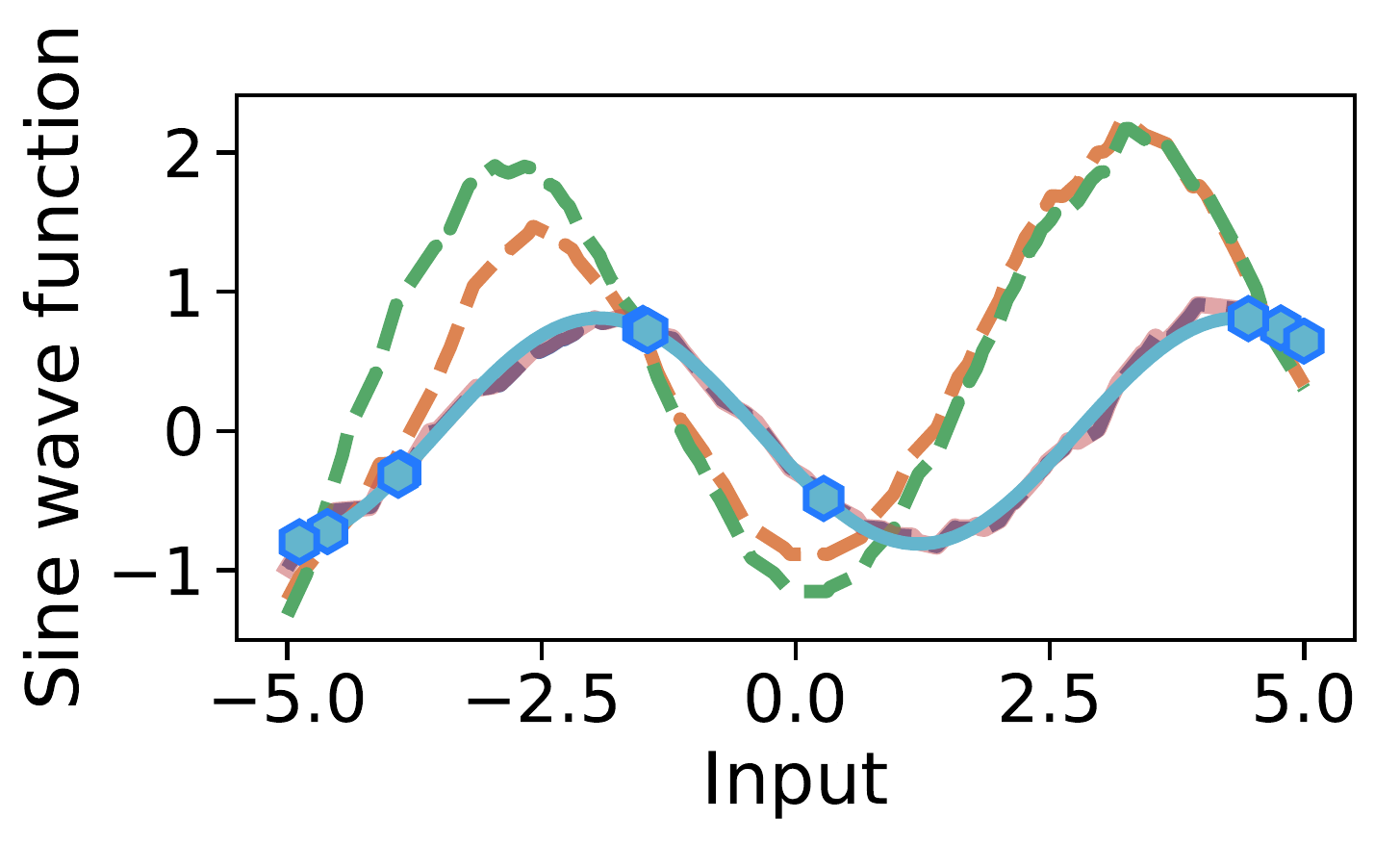}
\caption{Prediction after 60 steps of adaptation.}
\label{fig:sin_pll_bp_pred}
\end{subfigure}
\\
\begin{subfigure}{.3\textwidth}
\includegraphics[width=\textwidth]{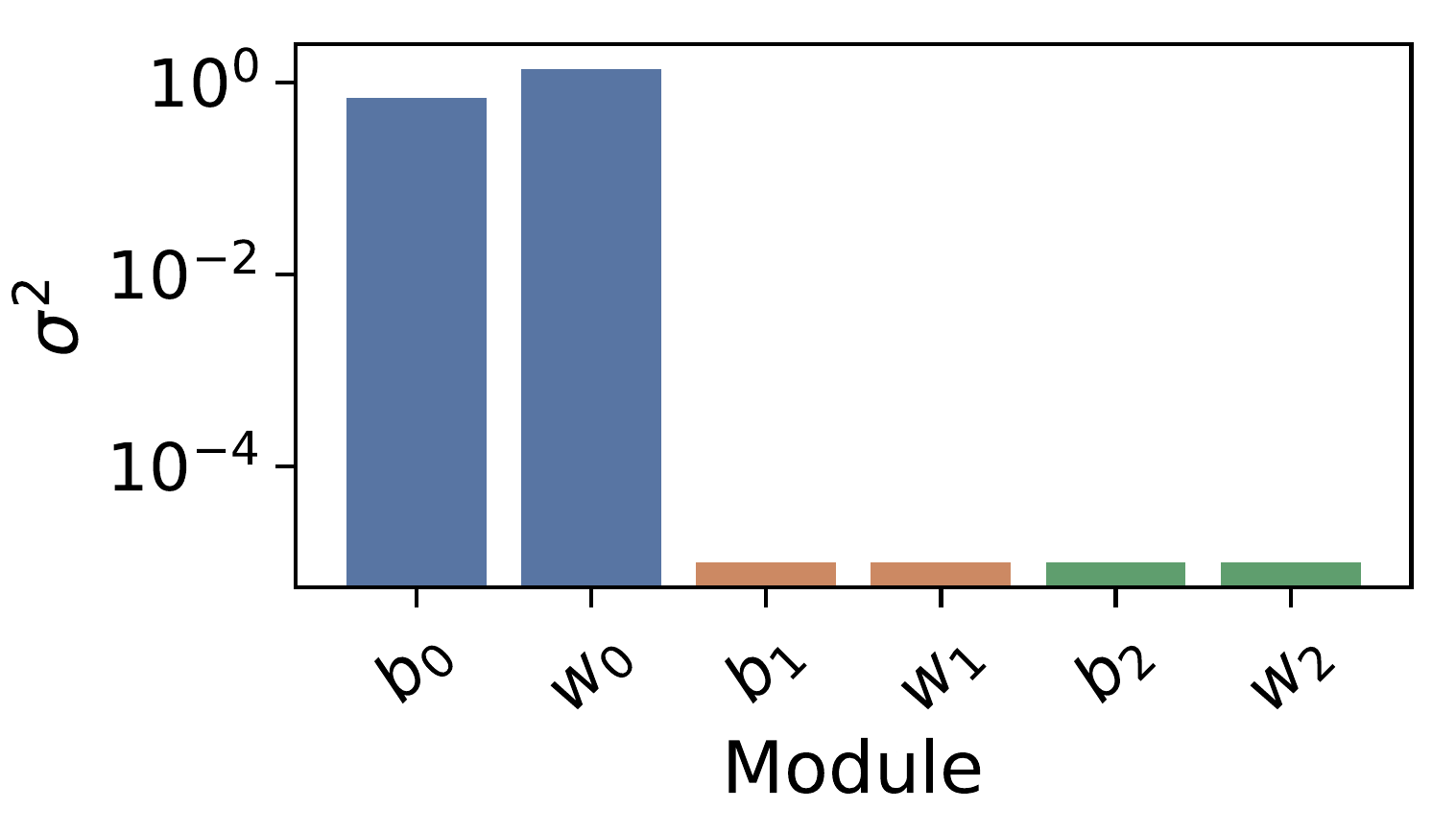}
\caption{Learned $\vsigma$ for each module (\shrimaml{}).}
\label{fig:sin_pll_var}
\end{subfigure}
~
\begin{subfigure}{.3\textwidth}
\includegraphics[width=\textwidth]{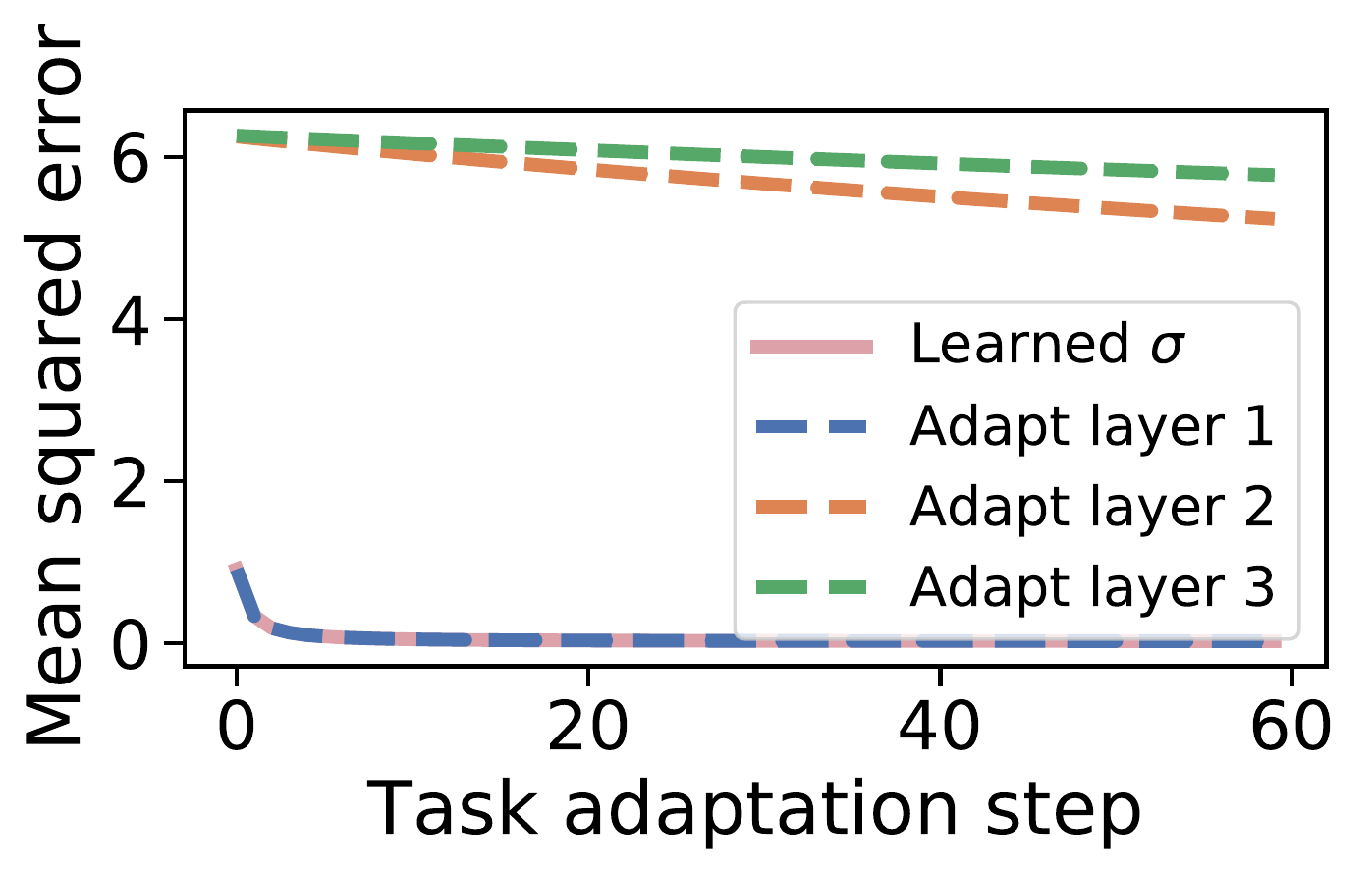}
\caption{MSE under different adaptation schemes.}
\label{fig:sin_pll_loss}
\end{subfigure}
~
\begin{subfigure}{.3\textwidth}
\includegraphics[width=\textwidth]{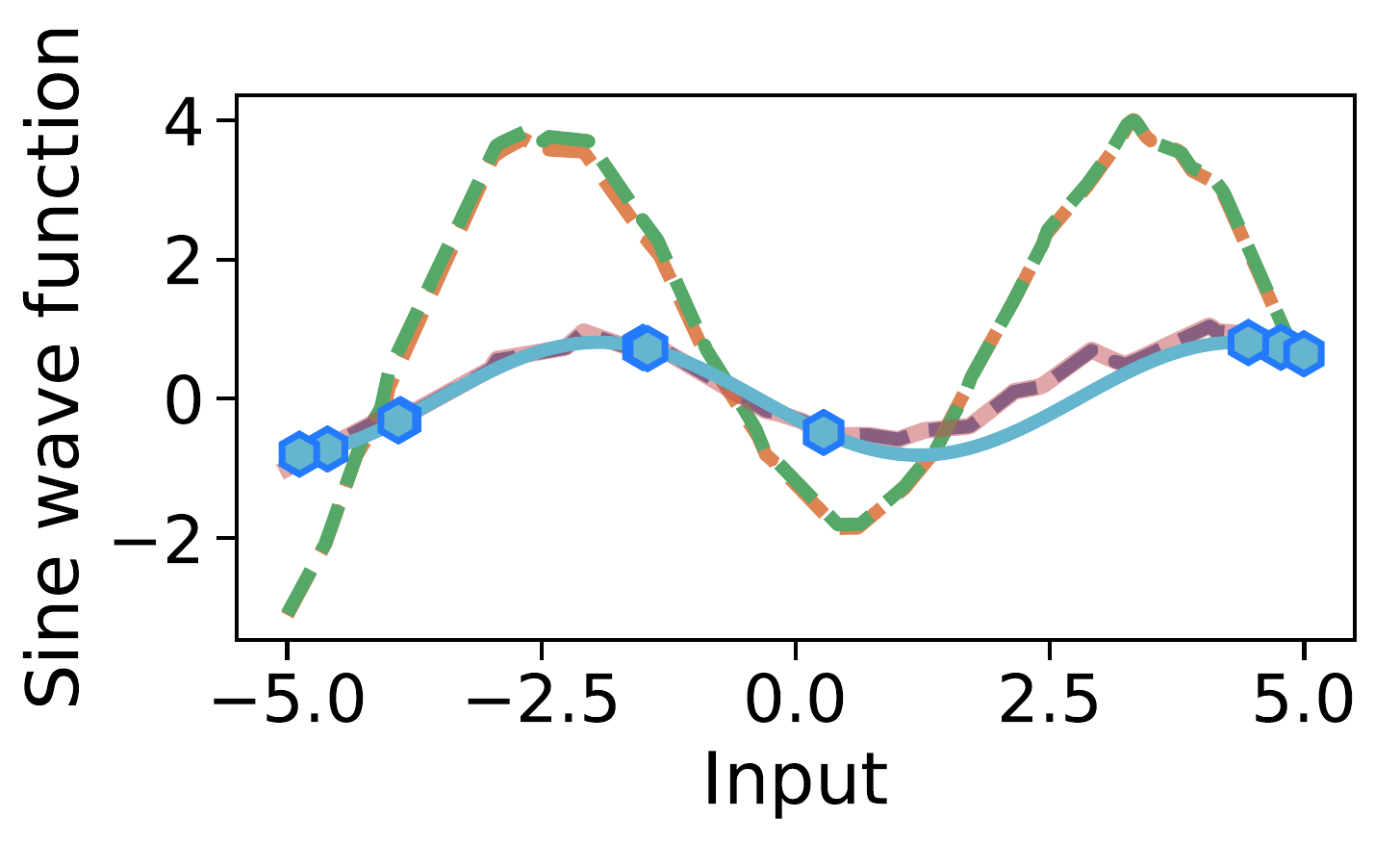}
\caption{Prediction after 60 steps of adaptation.}
\label{fig:sin_pll_pred}
\end{subfigure}
\\
\begin{subfigure}{.3\textwidth}
\includegraphics[width=\textwidth]{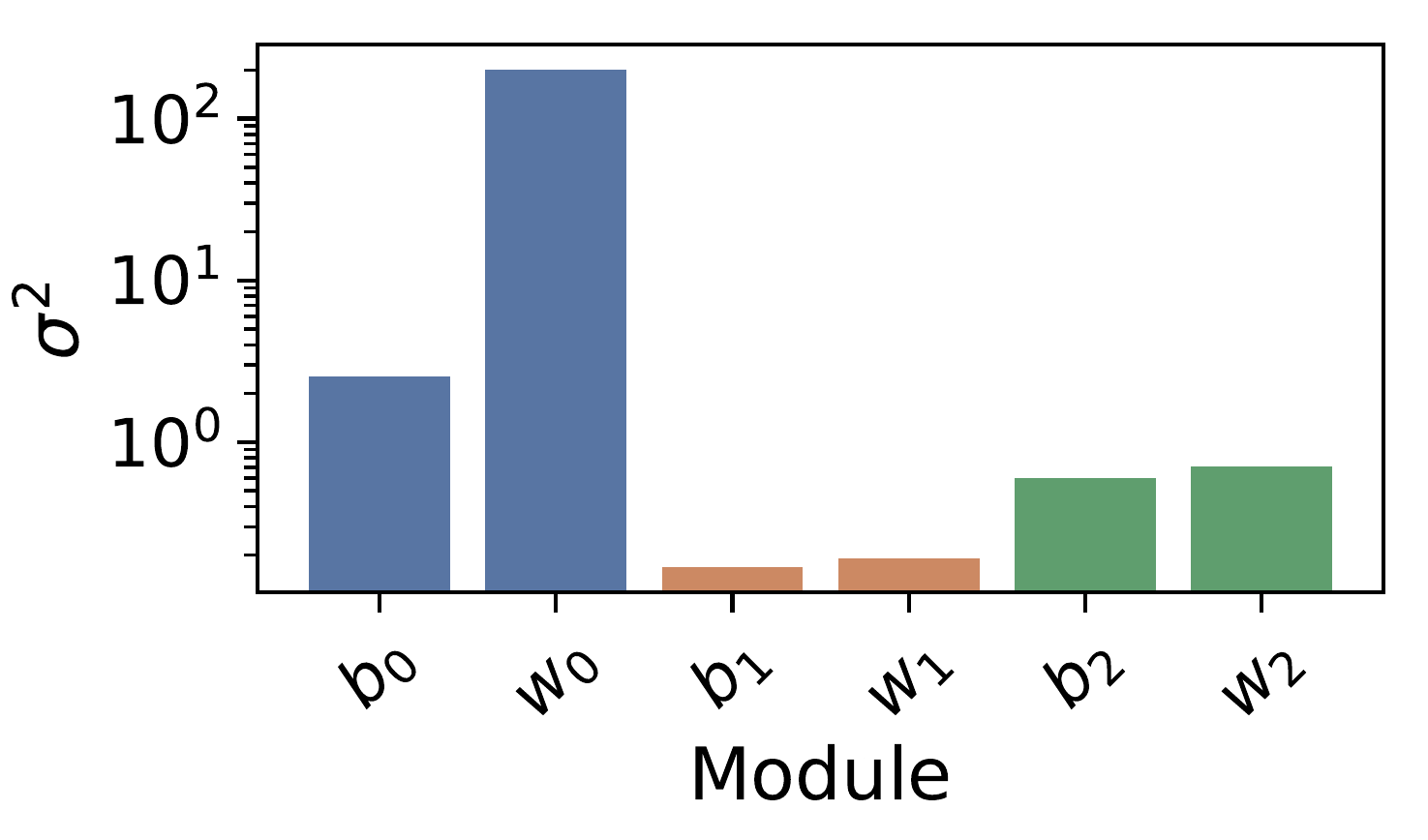}
\caption{Learned $\sigma$ for each module (\shrreptile{}).}
\label{fig:sin_map_var}
\end{subfigure}
~
\begin{subfigure}{.3\textwidth}
\includegraphics[width=\textwidth]{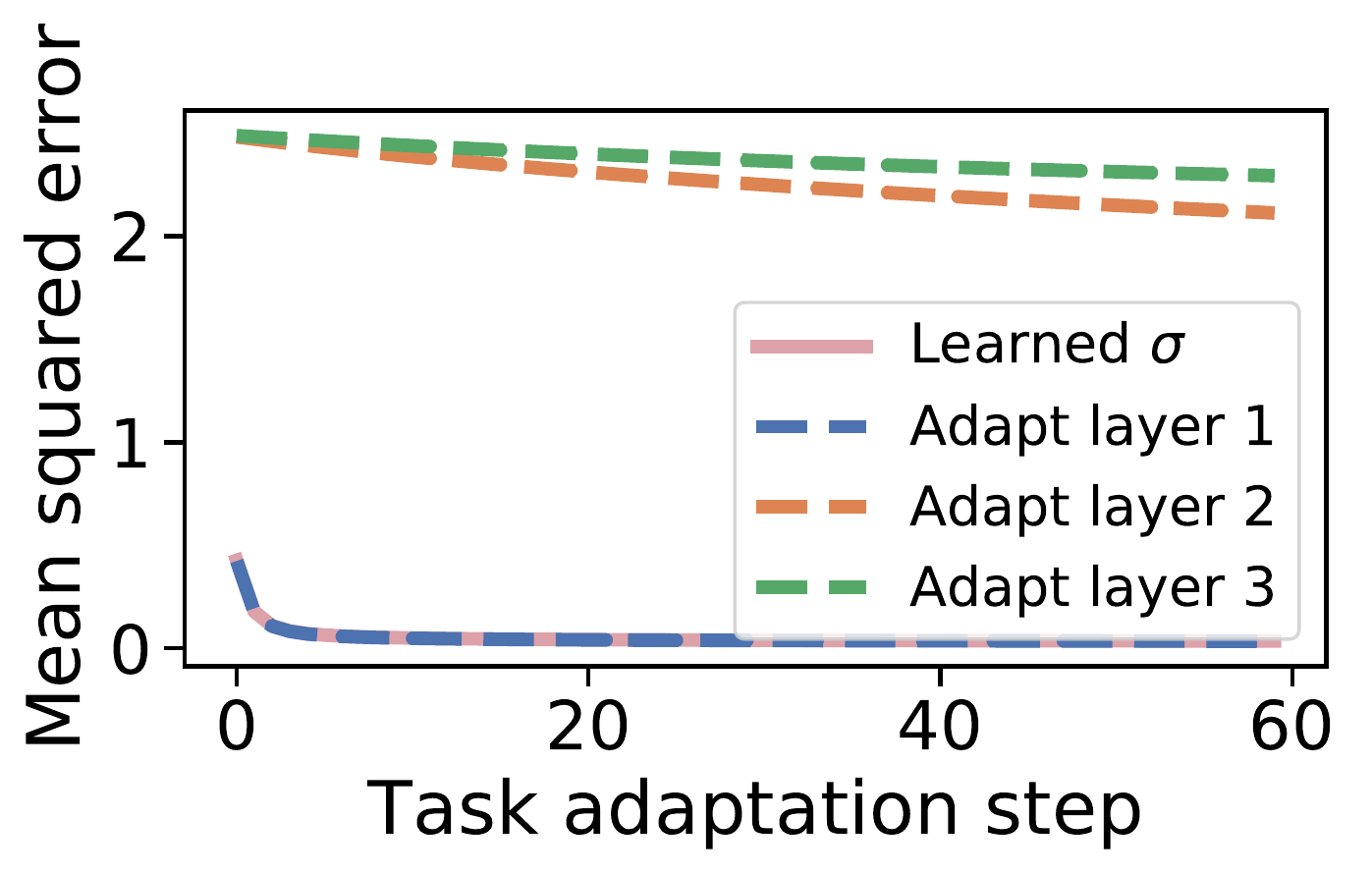}
\caption{MSE under different adaptation schemes.}
\label{fig:sin_map_loss}
\end{subfigure}
~
\begin{subfigure}{.3\textwidth}
\includegraphics[width=\textwidth]{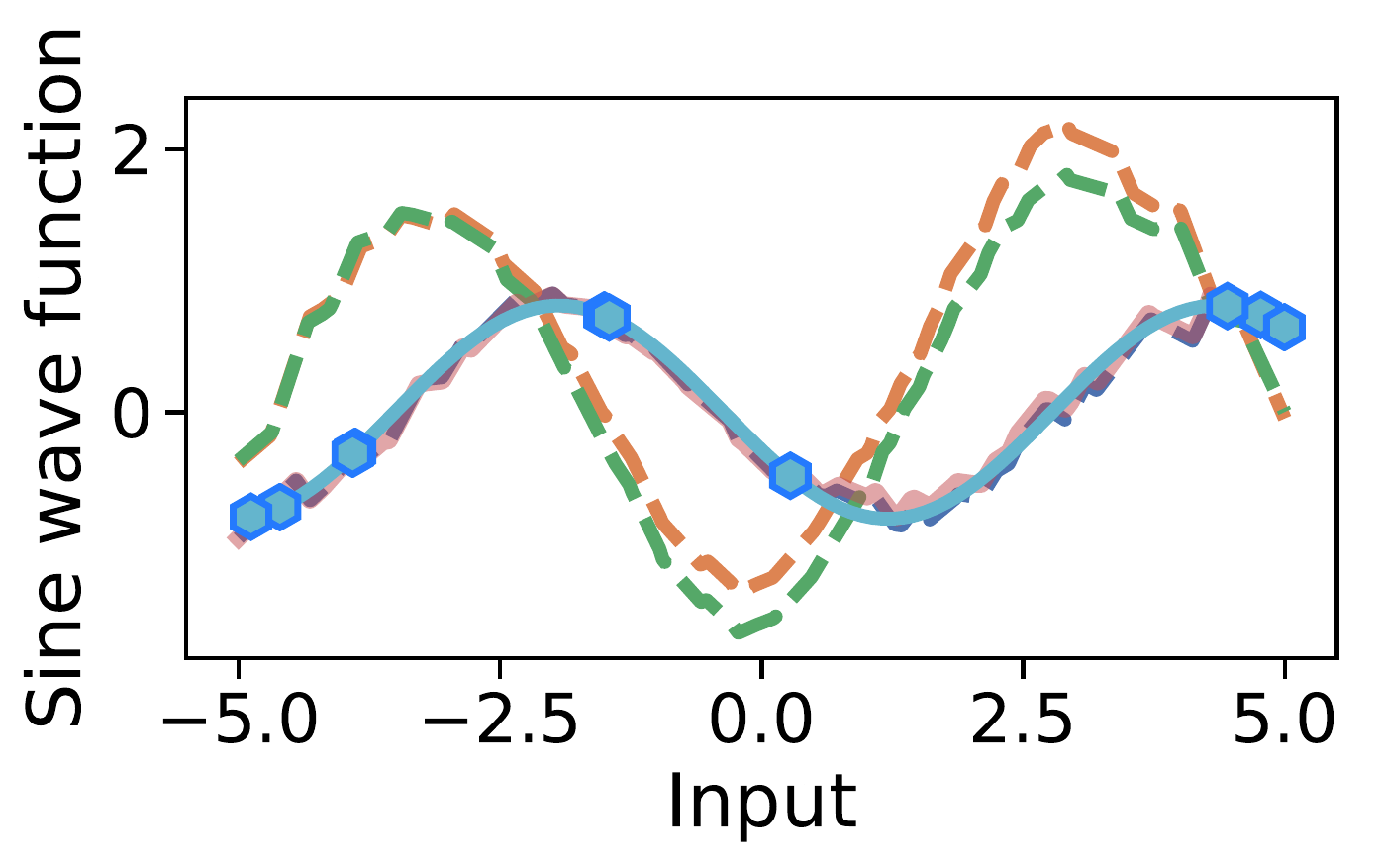}
\caption{Prediction after 60 steps of adaptation.}
\label{fig:sin_map_pred}
\end{subfigure}
\caption{Sinusoid regression with \shrmaml{} (top row), \shrimaml{} (middle row), and \shrreptile{} (bottom row). 
In the left column (\subref{fig:sin_pll_bp_var},\subref{fig:sin_pll_var},\subref{fig:sin_map_var}) we show the learned $\sigma_m$ for each module. 
In the middle column (\subref{fig:sin_pll_bp_loss},\subref{fig:sin_pll_loss},\subref{fig:sin_map_loss}) we show the mean squared error, averaged over 100 tasks, as a function of task adaptation step, while adapting only a single module (the dashed lines) or using the learned $\vsigma$ (pink).
Finally, the right column (\subref{fig:sin_pll_bp_pred},\subref{fig:sin_pll_pred},\subref{fig:sin_map_pred}) shows predictions under each adaptation scheme. Note that the model trained using the learned $\vsigma$ (pink) overlaps with the model with the first layer adapted (dark blue) in (\subref{fig:sin_pll_bp_loss}--\subref{fig:sin_pll_bp_pred},\subref{fig:sin_pll_loss}--\subref{fig:sin_pll_pred},\subref{fig:sin_map_loss}--\subref{fig:sin_map_pred}).}
\label{fig:sin}
\end{figure*}

\subsection{Additional short adaptation experiment: few-shot image classification}

We next look at two standard benchmarks for few-shot image classification, Omniglot and \textit{mini}ImageNet, which perform task adaptation in meta-training for up to 20 steps.

\subsubsection{Few-shot Omniglot}
\label{sec:apx_fewshot_omniglot}

The Omniglot dataset~\citep{lake2015omniglot,santoro2016meta} consists of 20 samples of 1623 characters from 50 different alphabets. 
The dataset is augmented by creating new characters that are rotations of each of the existing characters
by $0, 90, 180,$ or $270$ degrees.
We follow the standard $N$-way $K$-shot classification setting where a task is generated by randomly sampling $N$ characters and training the model on $K$ instances of each ~\citep{santoro2016meta,vinyals2016matching,ravi2017optimization}. 
Character classes are partitioned into separate meta-train and meta-test splits, and the instances (images) of each character are also split into separate (task) train and (task) validation subsets. 

We use the same $4$-block convolutional architecture as in \citet{Finn2017}.
This architecture consists of $4$ convolutional blocks,
each made up of a $3 \times 3$ convolutional layer with $64$ filters and stride $2$, a batch-normalization layer~\citep{ioffe2015batchnorm},
and a ReLU activation, in that order. The output of the final
convolutional block is fed into an output linear layer and a softmax, and trained with the
cross-entropy loss.
All images are downsampled to $28 \times 28$.

Hyperparameters for the six algorithms evaluated in this experiment are presented in \cref{table:omni_hypers}. Module discovery and classification performance for this dataset are discussed below.

\begin{table*}[htb]
    \caption{Hyperparameters for the few-shot Omniglot classification. Chosen with random search on the validation task set.} 
    \label{table:omni_hypers}
    \begin{center}
    \resizebox{\textwidth}{!}{        \begin{tabular}{lcccccc}
            \toprule
            ~  & MAML & Reptile  & iMAML & \shrmaml &  \shrreptile{} & \shrimaml{} \\
            \midrule
            Meta-training \\
            \midrule
            ~~Meta optimizer & Adam & SGD & Adam & Adam & SGD & Adam \\
            ~~Meta learning rate ($\vmetav$) & 4.8e-3 & 1.8 & 1e-3 & 1.7e-4 & 1.6 & 7.8e-3  \\
            ~~Meta learning rate ($\log \vsigma^2$) & - & - & - & 7.4e-3 & 7e-4 & 4e-3  \\
            ~~Meta training steps & 60k & 100k & 60k & 60k & 100k & 60K \\
            ~~Meta batch size (\# tasks) & 32 & 5 & 32 & 32 & 5 & 32 \\
            ~~Damping coefficient & - & - & 1.0 & - & 6e-3 & 0.18, \\
            ~~Conjugate gradient steps & - & - & 4 & - & 1 & 2 \\
            \midrule
                                    \multicolumn{7}{l}{Task adaptation (adaptation step and batch size for meta-test are in parentheses)} \\
            \midrule
            ~~Task optimizer & SGD & Adam & ProximalGD & ProximalGD & Adam & ProximalGD \\
            ~~Task learning rate & 0.968 & 8e-4 & 0.23 & 0.9 & 6e-4 & 1.3 \\
            ~~Task adaptation steps & 1 (50) & 5 (50) & 19 (50) & 3 (50) & 8 (50) & 8 (50) \\
            ~~Task batch size (\# images) & 5 (5) & 10 (5)  & 5 (5) & 5 (5) & 10 (5) & 5 (5) \\
            \bottomrule
        \end{tabular}
    }
    \end{center}
\end{table*}

\subsubsection{\textit{mini}ImageNet}
\label{sec:apx_miniimagenet}

The \textit{mini}ImageNet dataset~\citep{vinyals2016matching,ravi2017optimization} 
is, as the name implies, a smaller and easier many-task, few-shot variant of the ImageNet dataset.
However, its images are larger and more challenging than those of Omniglot.
\textit{mini}ImageNet consists of $100$ classes ($64$ train, $12$ validation, and $24$ test)
with images downsampled to $84 \times 84$. We follow the standard \textit{mini}ImageNet protocol
and train in the $N$-way $K$-shot paradigm using the same $4$-block convolutional architecture
as in previous work~\citep{vinyals2016matching,Finn2017}.
Each convolutional block consists of a $3 \times 3$ convolutional layer with $32$ filters,
a batch-normalization layer~\citep{ioffe2015batchnorm}, and a ReLU activation, in that order.
As in Omniglot, the output of the final convolutional block is fed into an output linear layer
and a softmax, and trained with the cross-entropy loss.

Hyperparameters for each algorithm on \textit{mini}ImageNet are presented in \cref{table:mini_hypers}. For iMAML, we use $\lambda = 0.14$.
Module discovery and classification performance for this dataset are discussed below.

\begin{table*}[htb]
    \caption{Hyperparameters for the few-shot \textit{mini}ImageNet classification experiment. Chosen with random search on validation task set.} 
    \label{table:mini_hypers}
    \begin{center}
    \resizebox{\textwidth}{!}{        \begin{tabular}{lcccccc}
            \toprule
            ~  & MAML & Reptile  & iMAML & \shrmaml &  \shrreptile{} & \shrimaml{} \\
            \midrule
            Meta-training \\
            \midrule
            ~~Meta optimizer & Adam & SGD & Adam & Adam & Adam & Adam \\
            ~~Meta learning rate ($\vmetav$) & 1e-3 & 0.45 & 2.5e-4 & 1e-3 & 6e-4 & 2e-4  \\
            ~~Meta learning rate ($\log \vsigma^2$) & - & - & - & 5e-2 & 2.5e-2 & 0.2  \\
            ~~Meta training steps & 60k & 100k & 60k & 60k & 100k & 60k \\
            ~~Meta batch size (\# tasks) & 4 & 5 & 4 & 5 & 3 & 5 \\
            ~~Damping coefficient & - & - & 6.5e-2 & - & 1e-2 & 2e-3 \\
            ~~Conjugate gradient steps & - & - & 5 & - & 3 & 2 \\
            \midrule
                                    \multicolumn{7}{l}{Task adaptation (adaptation step and batch size for meta-test are in parentheses)} \\
            \midrule
            ~~Task optimizer & SGD & Adam & ProximalGD & ProximalGD & ProximalGD & ProximalGD \\
            ~~Task learning rate & 1e-2 & 1.5e-3 & 4e-2 & 0.1 & 1.35 & 1.6 \\
            ~~Task adaptation steps & 5 (50) & 5 (50) & 17 (50) & 5 (50) & 7 (50) & 1 (50) \\
            ~~Task batch size (\# images) & 5 (5) & 10 (5) & 5 (5) & 5 (5) & 10 (5) & 5 (5) \\
            \bottomrule
        \end{tabular}
    }
    \end{center}
\end{table*}

\begin{figure*}[tbh!]
\centering
\begin{minipage}{1.0\textwidth}
\centering
\begin{subfigure}{.3\textwidth}
\includegraphics[width=\textwidth]{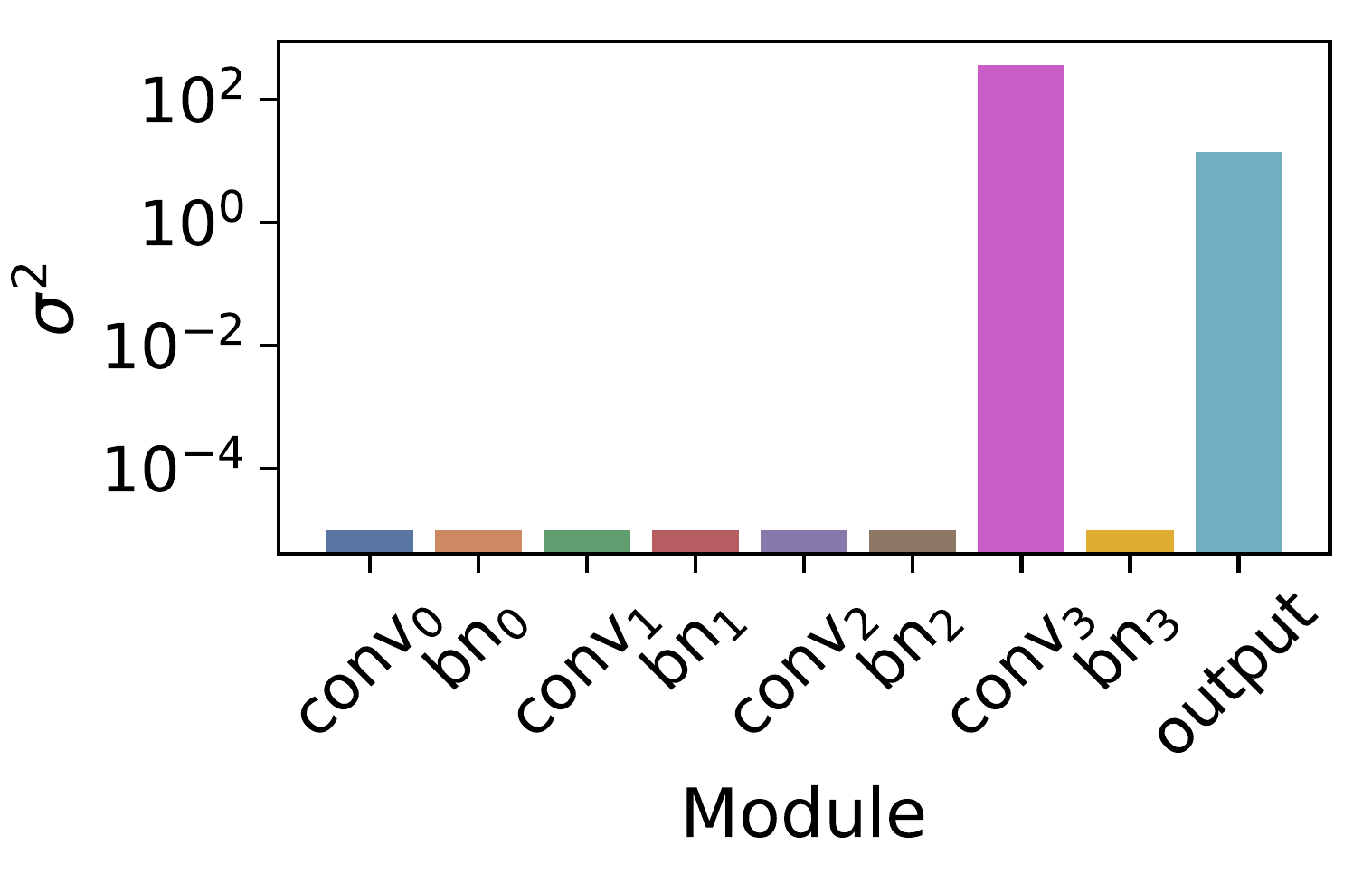}
\caption{\hspace{-0.1em}Learned $\sigma^2$ for \shrmaml{}.}
\label{fig:omniglot_shrinkage_maml_var_appendix}
\end{subfigure}
~
\begin{subfigure}{.3\textwidth}
\includegraphics[width=\textwidth]{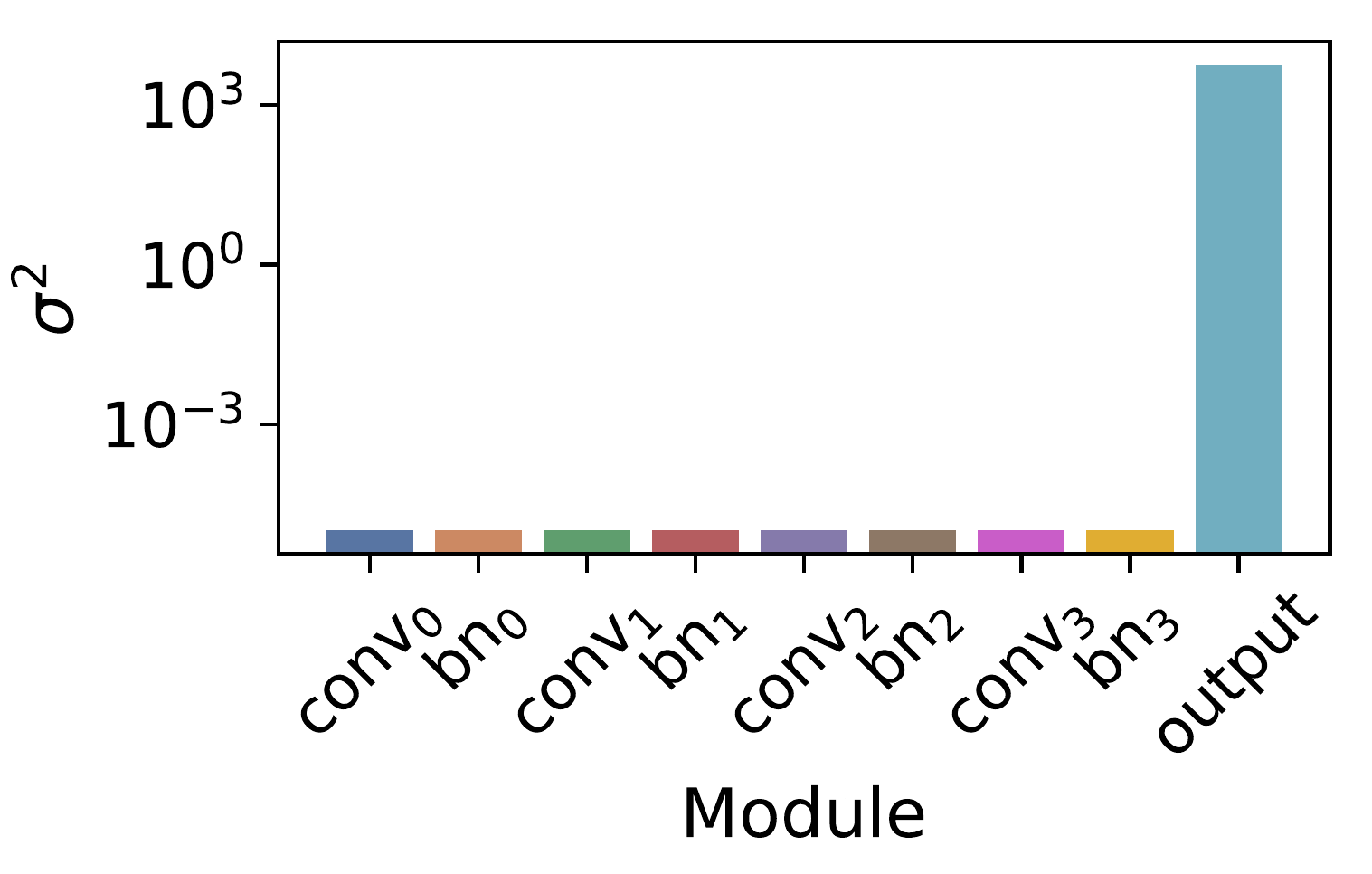}
\caption{Learned $\sigma^2$, \shrimaml{}.}
\label{fig:omniglot_shrinkage_imaml_var_appendix}
\end{subfigure}
~
\begin{subfigure}{.3\textwidth}
\includegraphics[width=\textwidth]{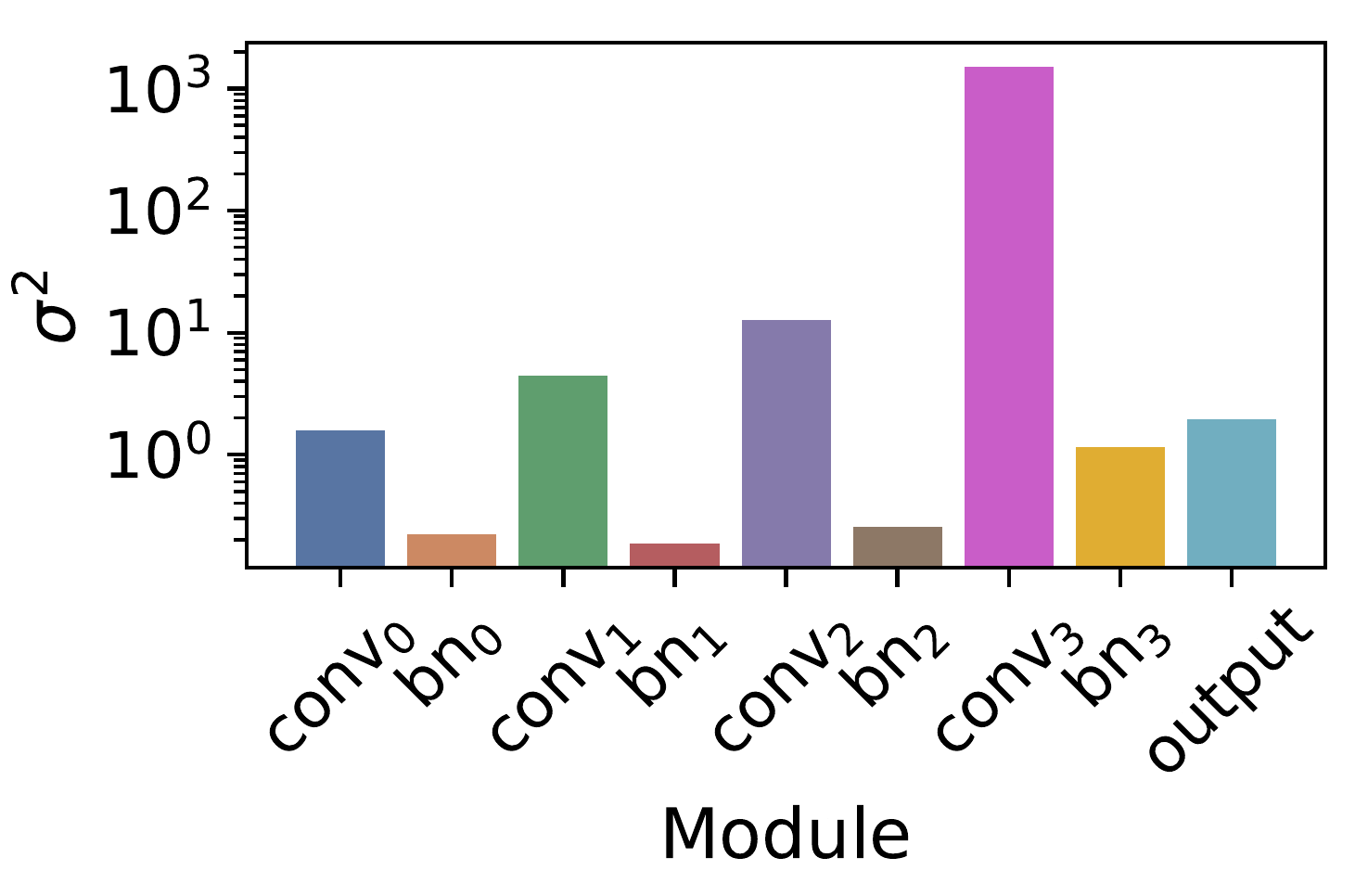}
\caption{Learned $\sigma^2$ for \shrreptile{}.}
\label{fig:omniglot_shrinkage_reotile_var_appendix}
\end{subfigure}
\\
\begin{subfigure}{.3\textwidth}
\includegraphics[width=\textwidth]{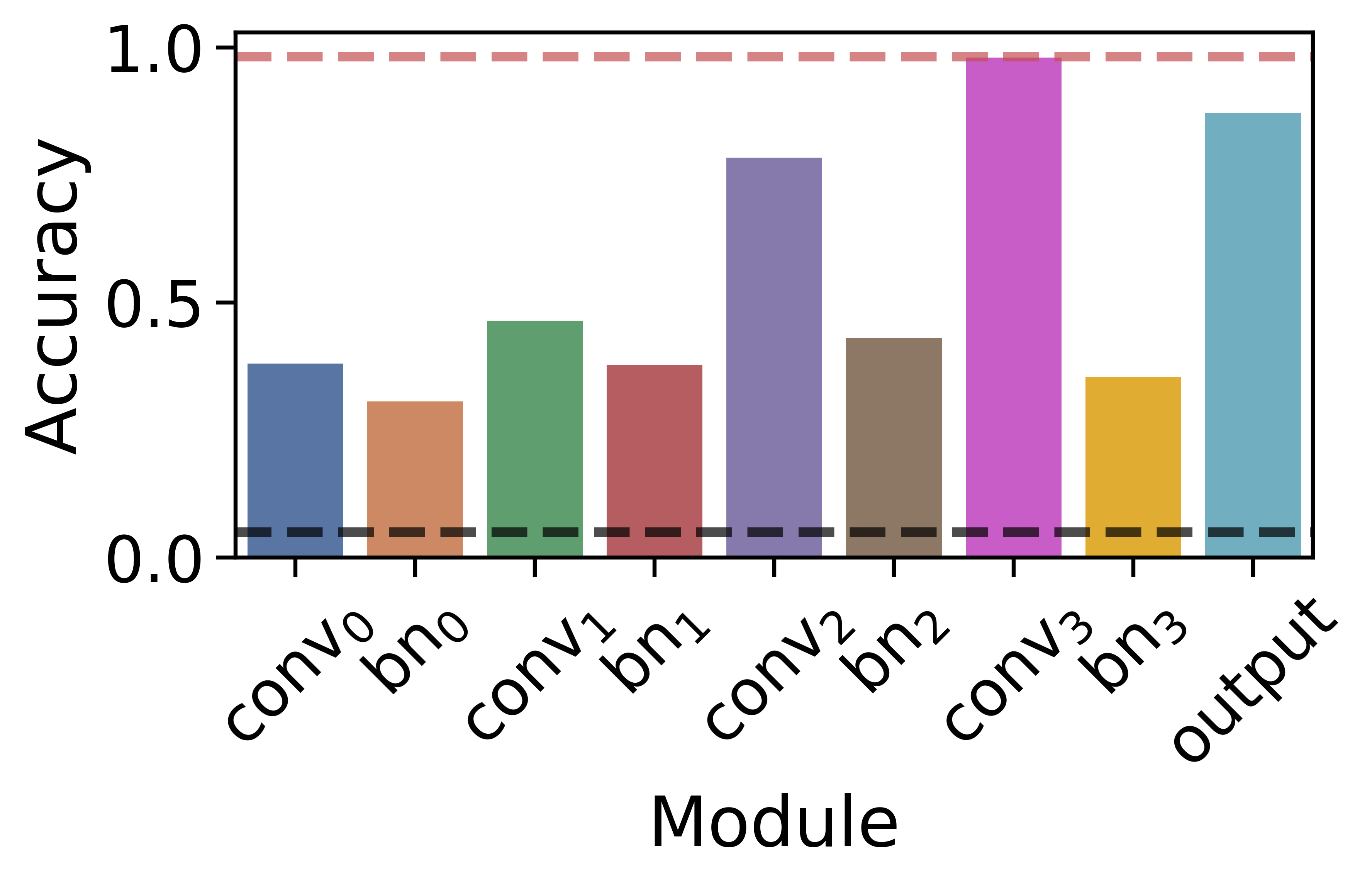}
\caption{Accuracy for \shrmaml{}.}
\label{fig:omniglot_shrinkage_maml_acc_appendix}
\end{subfigure}
~
\begin{subfigure}{.3\textwidth}
\includegraphics[width=\textwidth]{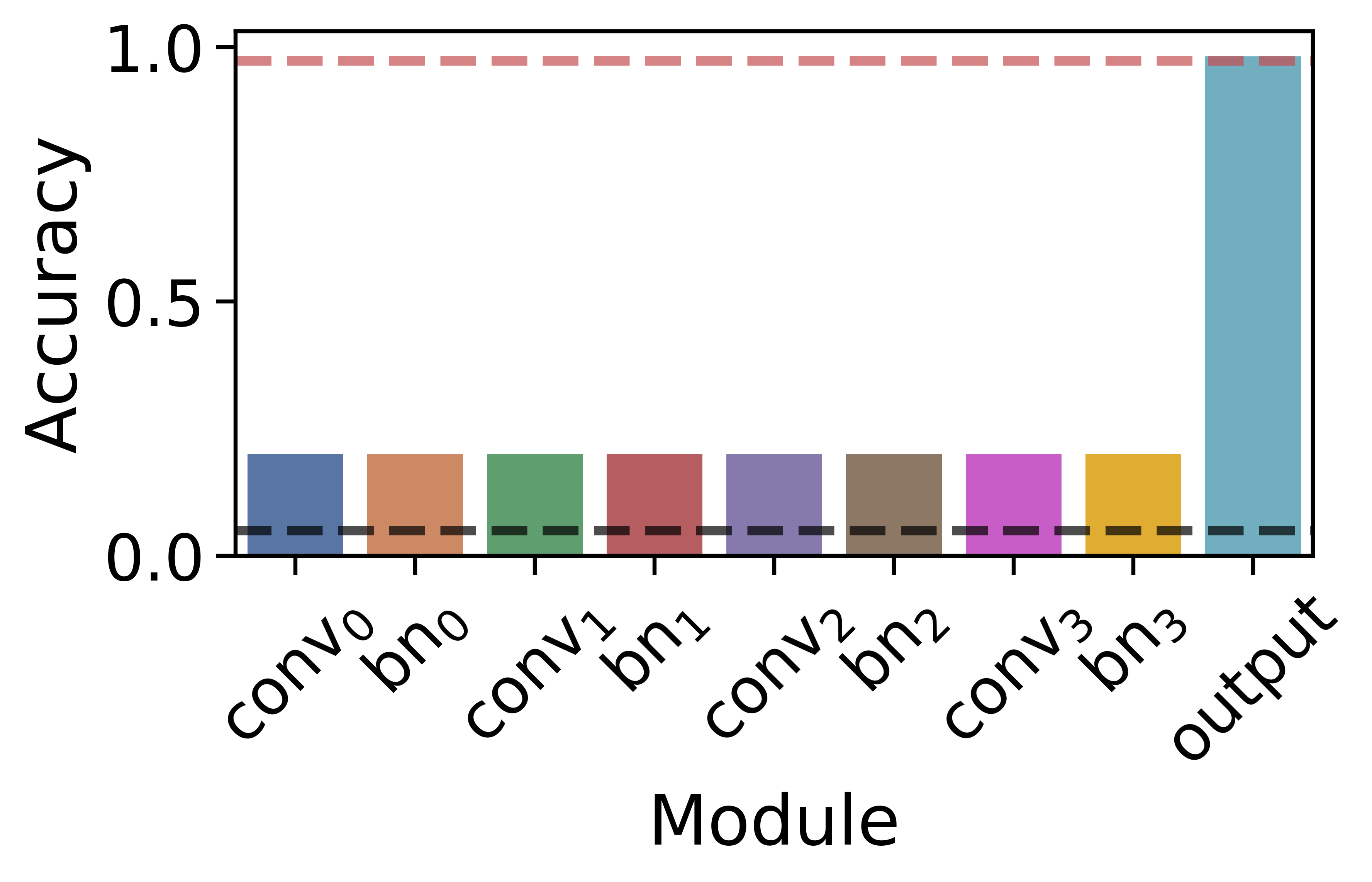}
\caption{Accuracy, \shrimaml{}.}
\label{fig:omniglot_shrinkage_imaml_acc_appendix}
\end{subfigure}
~
\begin{subfigure}{.3\textwidth}
\includegraphics[width=\textwidth]{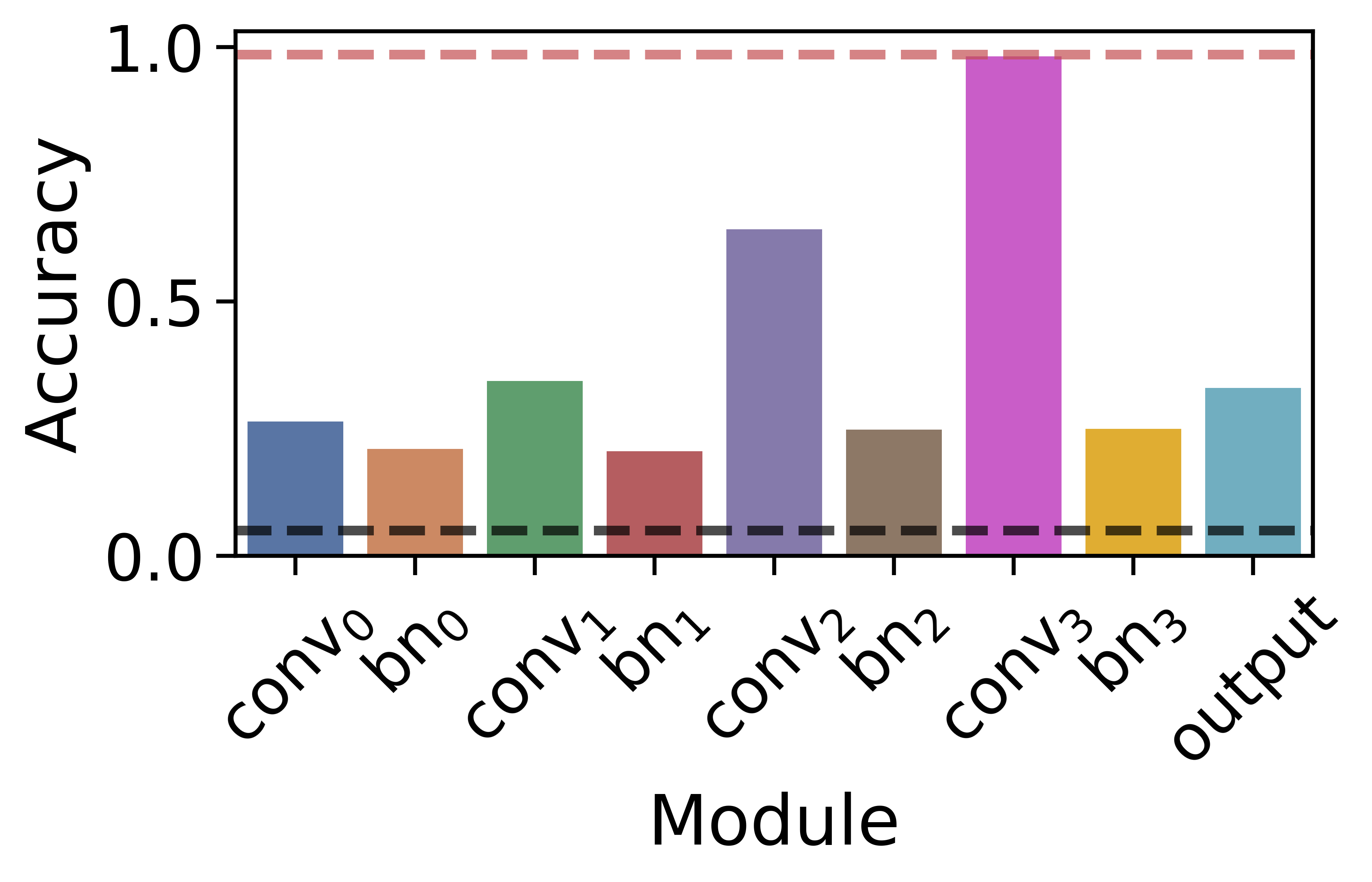}
\caption{Accuracy for \shrreptile{}.}
\label{fig:omniglot_shrinkage_reptile_acc_appendix}
\end{subfigure}
\caption{Learned variances and test accuracies on Omniglot. (a) \& (b) \& (c) show the learned variance per module with \shrmaml{}, \shrimaml{} and \shrreptile{}, respectively. (d) \& (e) \& (f) show the average test accuracy at the end of task adaptation with \shrmaml{}, \shrimaml{} and \shrreptile{}. Each bar shows the accuracy after task adaptation either with all layers frozen except one. Colors map to the colors of (a) \& (b) \& (c). 
$\mathrm{bn}_i$ and $\mathrm{conv}_i$ denote $i$-th  batch normalization layer and convolutional layer, and $\mathrm{output}$ is the linear output layer. The pink dashed line shows the accuracy after adaptation with the learned $\sigma$ and the black dashed line is the chance accuracy.
}
\label{fig:omniglot_appendix}
\end{minipage}
\end{figure*}

\begin{figure*}[bth!]
\centering
\begin{subfigure}{.3\textwidth}
\includegraphics[width=\textwidth]{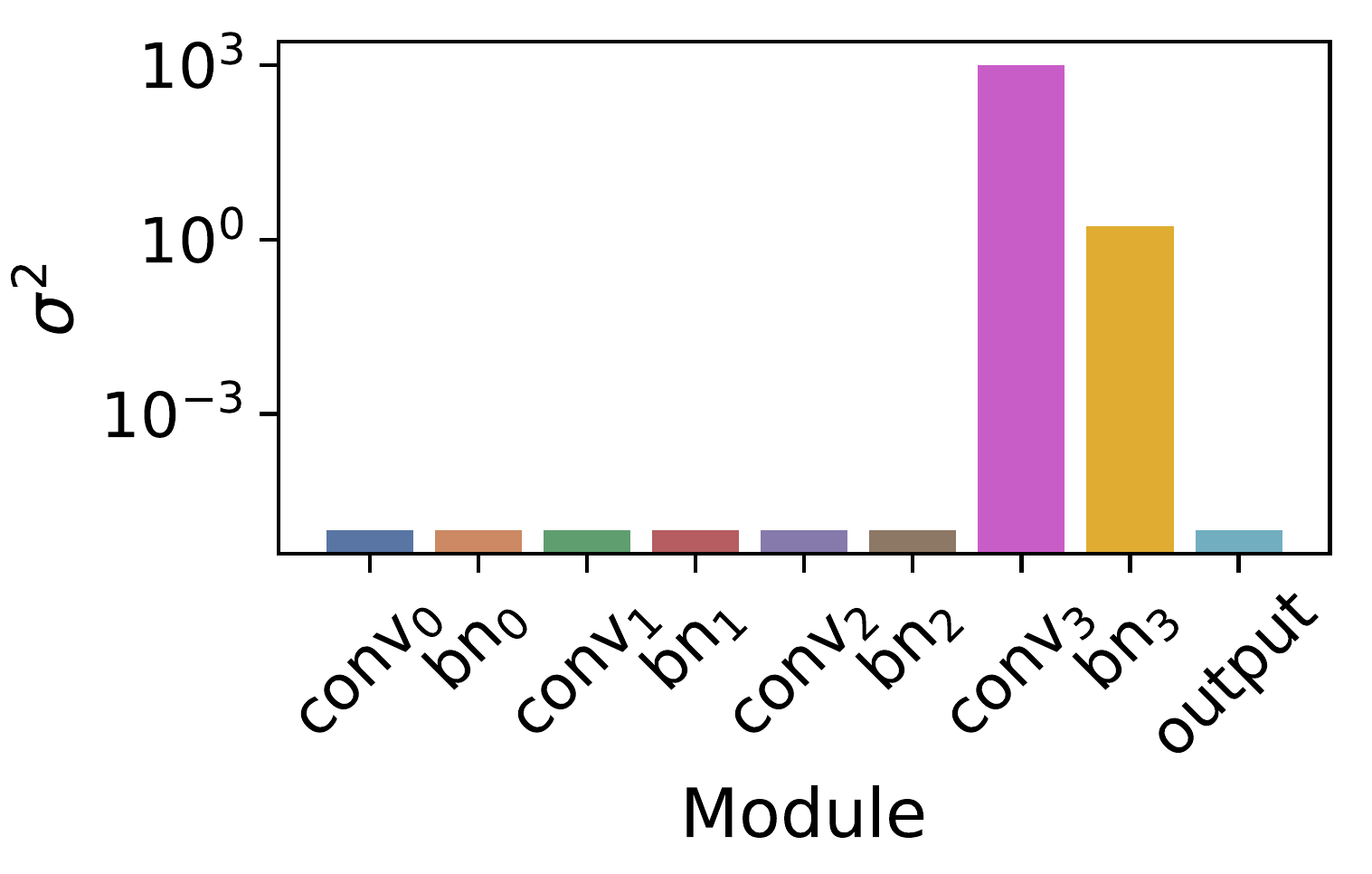}
\caption{\hspace{-0.1em}Learned $\sigma^2$ for \shrmaml{}.}
\label{fig:mini_shrinkage_maml_var_appendix}
\end{subfigure}
~
\begin{subfigure}{.3\textwidth}
\includegraphics[width=\textwidth]{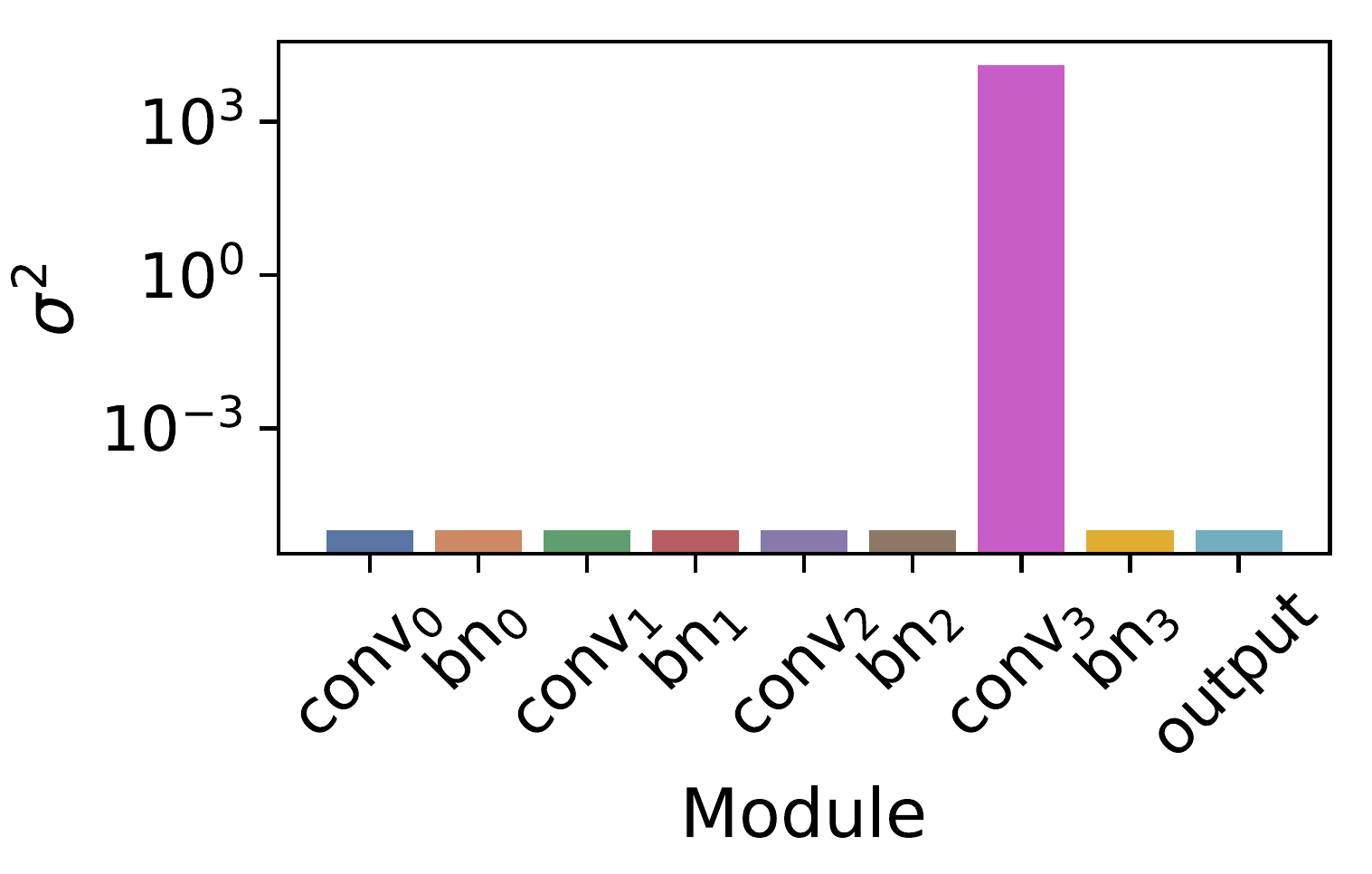}
\caption{Learned $\sigma^2$, \shrimaml{}.}
\label{fig:mini_shrinkage_imaml_var_appendix}
\end{subfigure}
~
\begin{subfigure}{.3\textwidth}
\includegraphics[width=\textwidth]{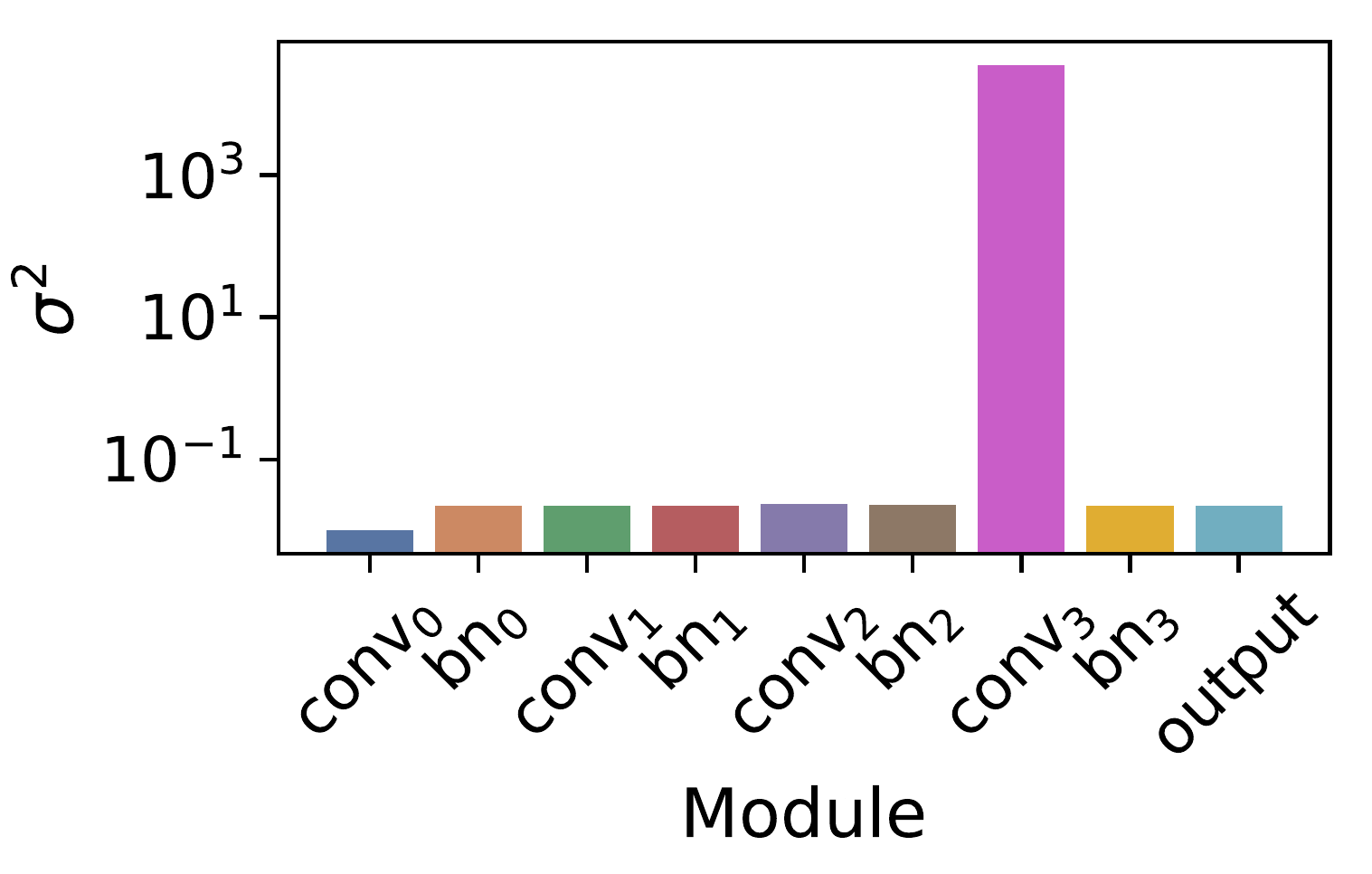}
\caption{Learned $\sigma^2$ for \shrreptile{}.}
\label{fig:mini_shrinkage_reotile_var_appendix}
\end{subfigure}
\\
\begin{subfigure}{.3\textwidth}
\includegraphics[width=\textwidth]{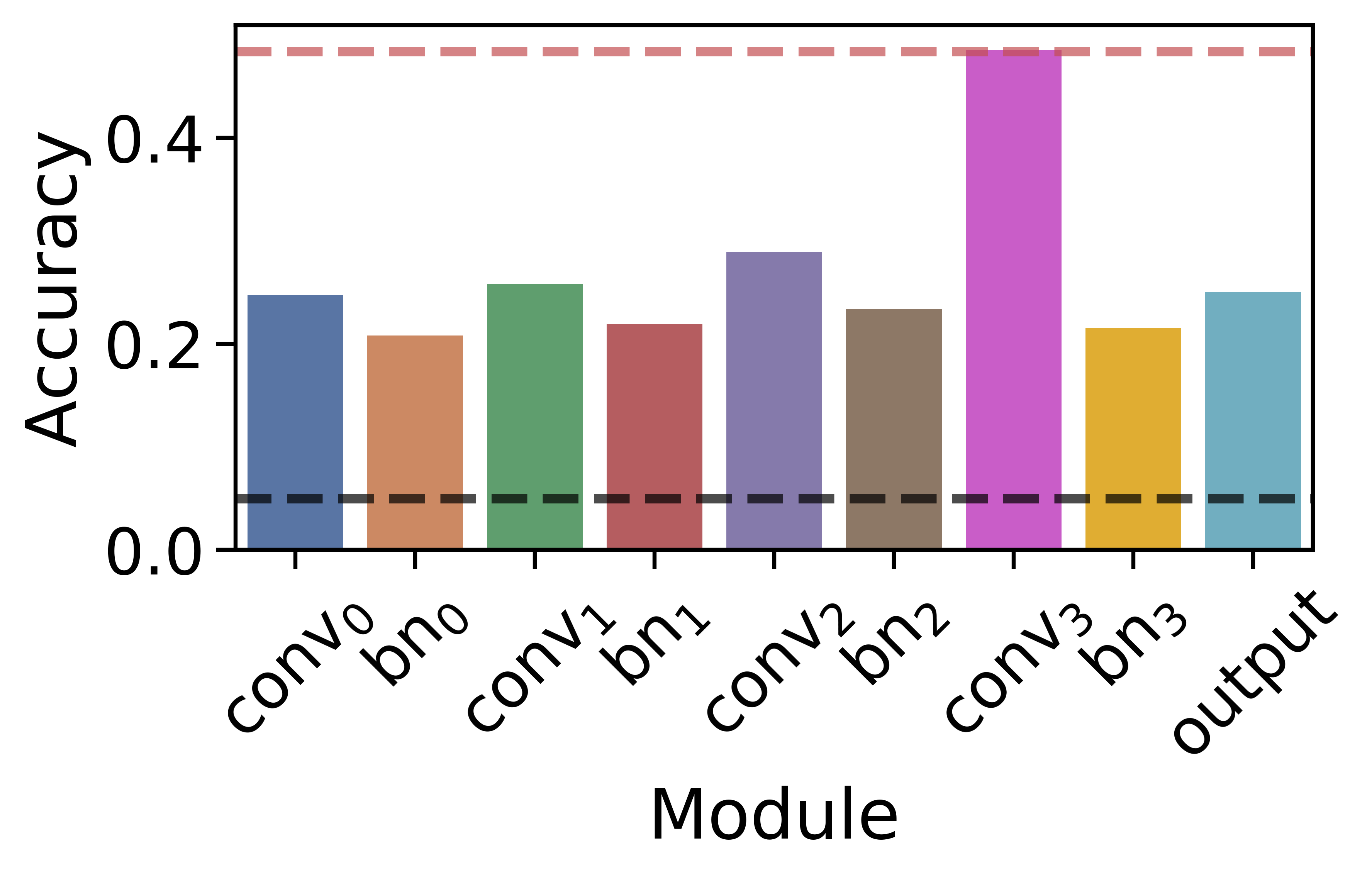}
\caption{Accuracy for \shrmaml{}.}
\label{fig:mini_shrinkage_maml_acc_appendix}
\end{subfigure}
~
\begin{subfigure}{.3\textwidth}
\includegraphics[width=\textwidth]{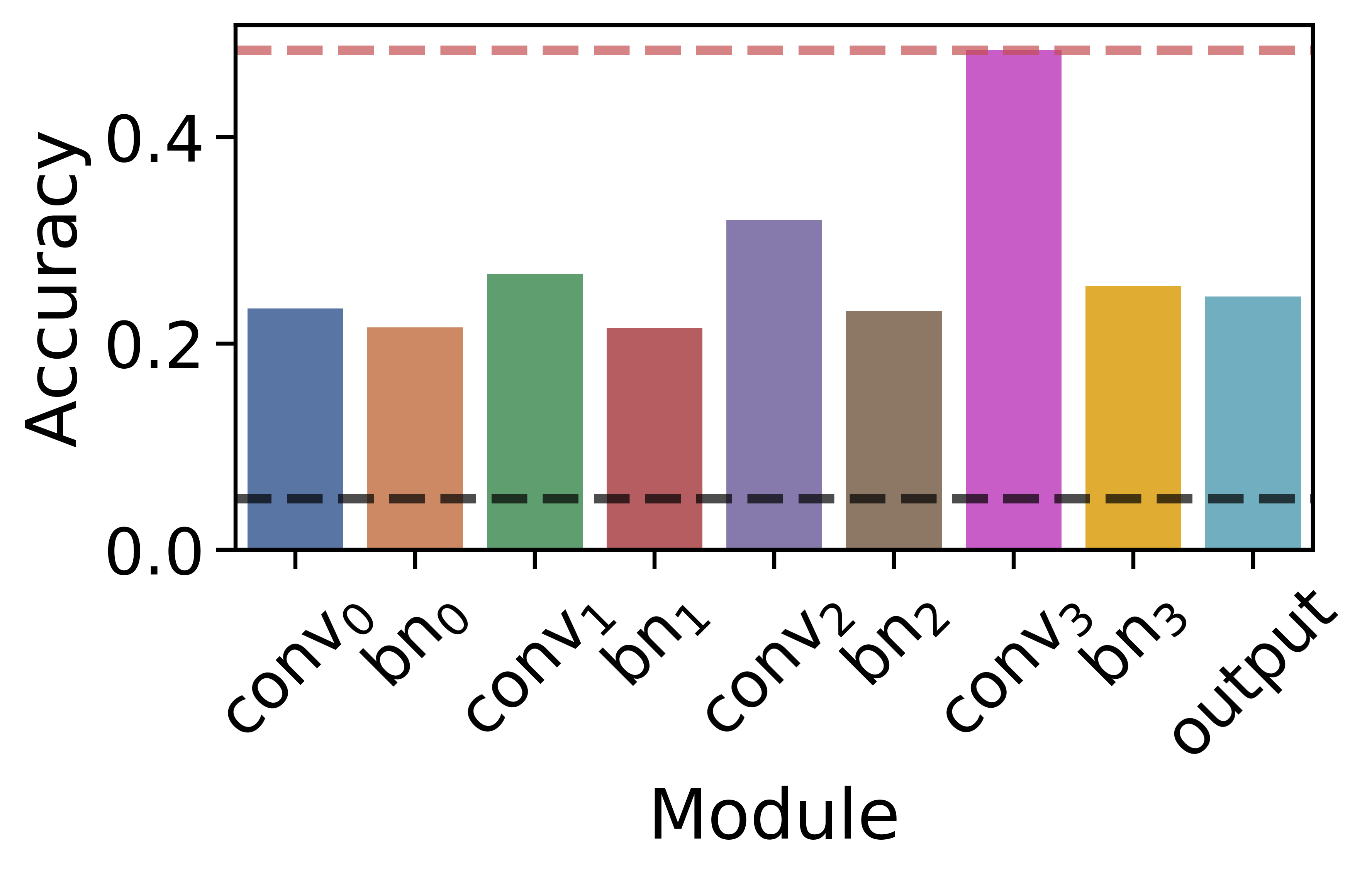}
\caption{Accuracy, \shrimaml{}.}
\label{fig:mini_shrinkage_imaml_acc_appendix}
\end{subfigure}
~
\begin{subfigure}{.3\textwidth}
\includegraphics[width=\textwidth]{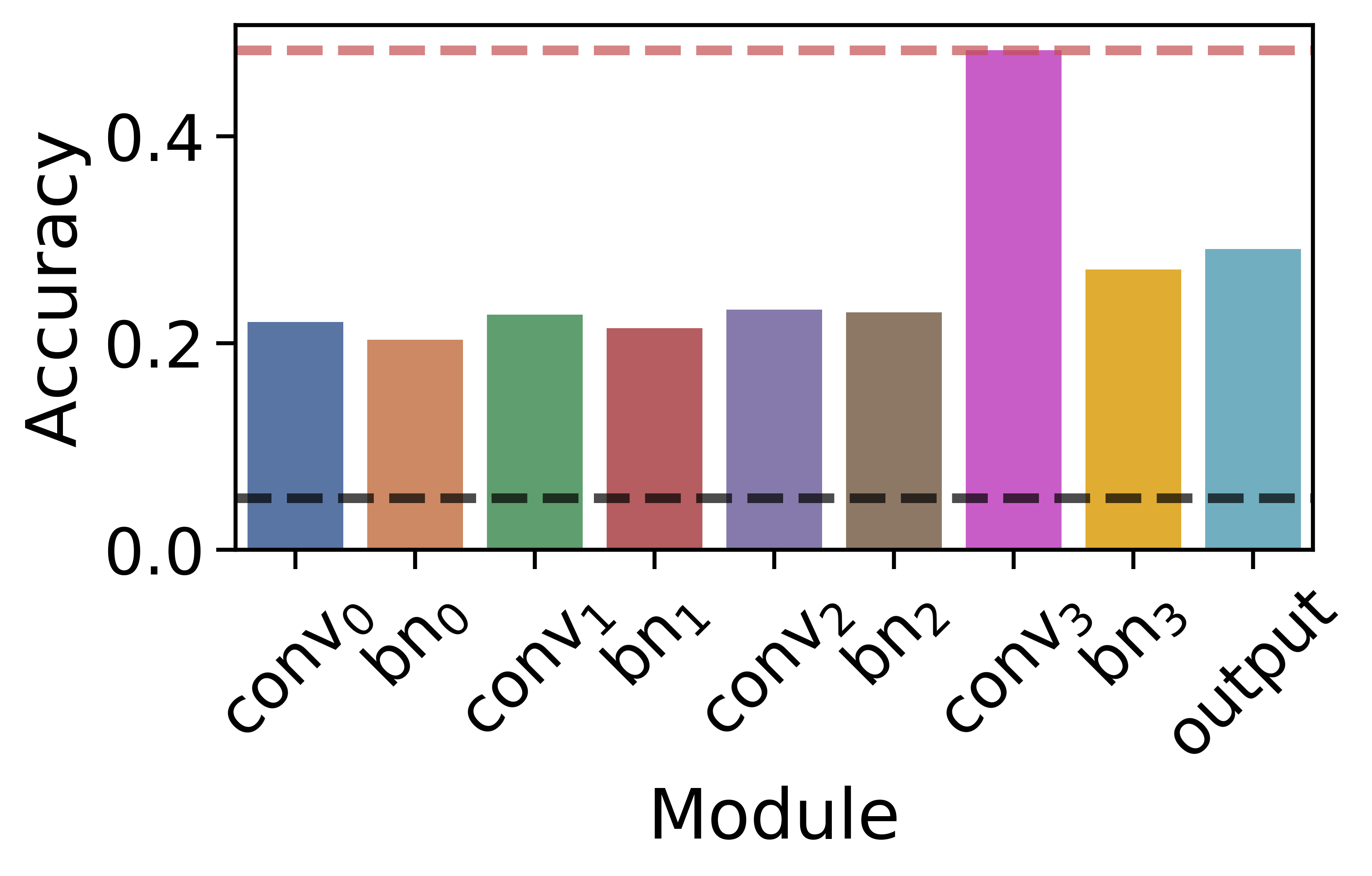}
\caption{Accuracy for \shrreptile{}.}
\label{fig:mini_shrinkage_reptile_acc_appendix}
\end{subfigure}

\caption{Learned variances and test accuracies on \textit{mini}ImageNet. (a) \& (b) \& (c) show the learned variance per module with \shrmaml{}, \shrimaml{} and \shrreptile{}, respectively. (d) \& (e) \& (f) show the average test accuracy during at the end of task adaptation with \shrmaml{}, \shrimaml{} and \shrreptile{}. Each bar shows the accuracy after task adaptation either with all layers frozen except one. Colors map to the colors of (a) \& (b) \& (c). 
$\mathrm{bn}_i$ and $\mathrm{conv}_i$ denote $i$-th  batch normalization layer and convolutional layer, and $\mathrm{output}$ is the linear output layer. The pink dashed line shows the accuracy after adaptation with the learned $\sigma$ and the black dashed line is the chance accuracy.
}
\label{fig:mini_appendix}
\end{figure*}

\begin{figure}[tbh!]
\centering
\includegraphics[width=0.4\textwidth]{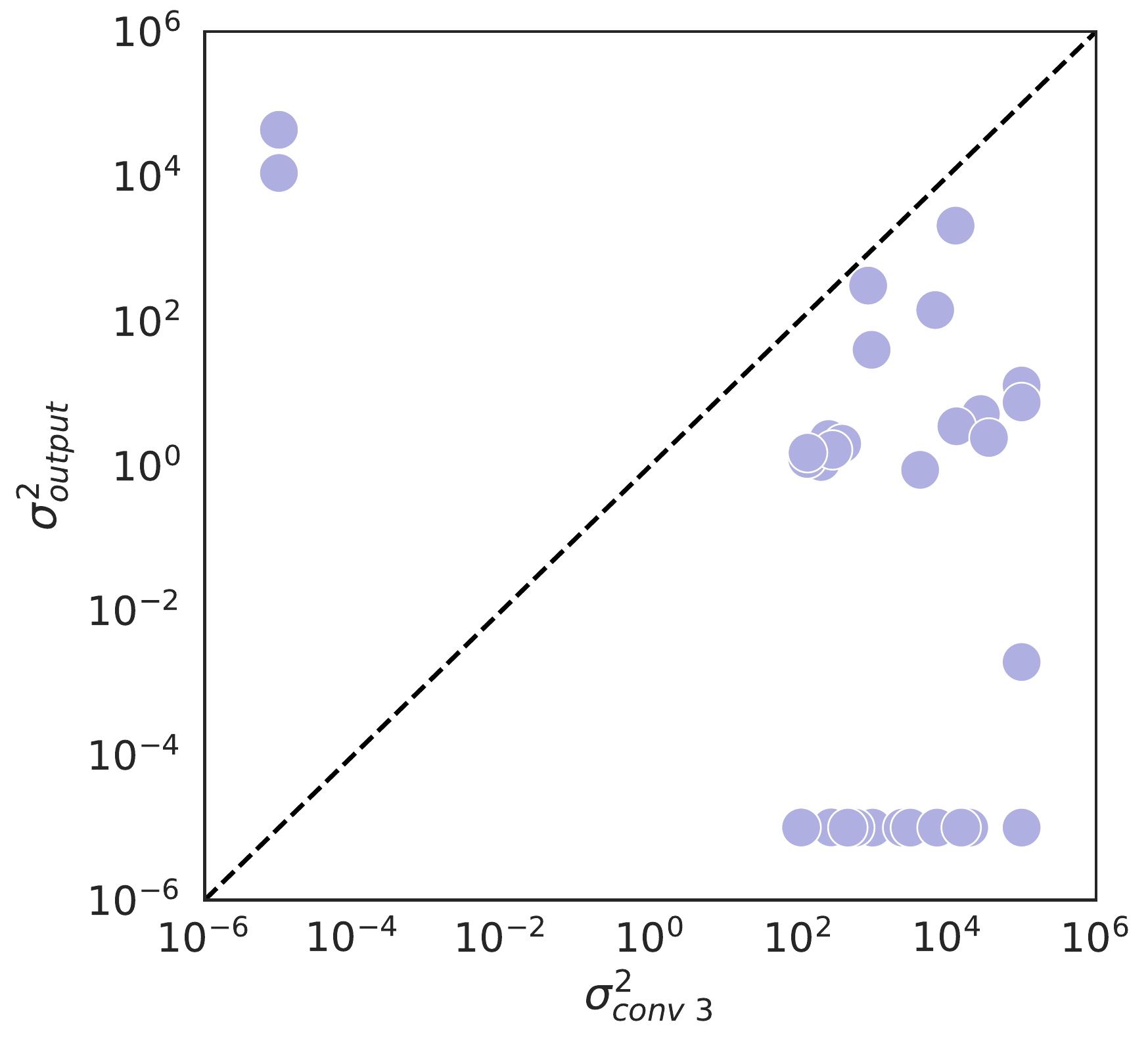}
\caption{Learned variances of the last conv layer vs output linear layer with \shrimaml{} from multiple runs.}
\label{fig:omni_simaml_var_scatter}
\end{figure}

\begin{table*}[tbh!]
    \caption{Test accuracy on few-shot Omniglot and few-shot \textit{mini}ImageNet. For each algorithm, we report the mean and $95\%$ confidence interval over $10$ different runs. For each pair of corresponding methods, we bold the entry if there exists a statistically significant difference.
    }
    \label{table:classification}
    \begin{center}
            \begin{tabular}{lcc}
                
                        \toprule
            ~                   & Omniglot  & \textit{mini}ImageNet \\
            \multicolumn{1}{l}{~} & $N5$ $K1$  & $N5$ $K1$ \\
            \midrule
            \multicolumn{1}{l}{\shrmaml{}} & $98.8 \pm 0.4\% $ & $\mathbf{47.7 \pm 0.5\%}$ \\
            \multicolumn{1}{l}{MAML} & $98.7 \pm 0.7\%$ & $46.1 \pm 0.8\%$ \\
            \midrule
            \multicolumn{1}{l}{\shrimaml{}} & $97.2 \pm 0.8\%$ & $47.6 \pm 1.1\%$  \\
            \multicolumn{1}{l}{iMAML} & $97.2 \pm 1.2\%$ & $47.2 \pm 1.4\%$ \\ 
            \midrule
            \multicolumn{1}{l}{\shrreptile{}} & $\mathbf{97.8 \pm 0.5\%} $ & $47.0 \pm 0.9\%$ \\
            \multicolumn{1}{l}{Reptile} & $96.9 \pm 0.6\%$ & $47.4 \pm 0.9\%$ \\ 
            \bottomrule
        \end{tabular}
        \end{center}
\end{table*}

\subsubsection{Module discovery}
We first discuss module discovery for these two datasets and then compare classification accuracies in the next section.

\cref{fig:omniglot_appendix} and \cref{fig:mini_appendix} present our module discovery results for all three shrinkage algorithms on few-shot Omniglot and \textit{mini}ImageNet. We see that after training there is always one layer that has a learned variance that is much larger than all other layers.
As shown in \cref{fig:omniglot_appendix}(d-e) and \cref{fig:mini_appendix}(d-e), we observe that by adapting only this task-specific module, the model is able to achieve high accuracy at test time, equal to the performance achieved when adapting all layers
according to the learned $\vsigma^2$. 
Conversely, adapting only the task-independent modules leads to poor performance.

Importantly, it is always one of the final two output layers that has the highest learned variance, meaning that these two layers are the most task-specific. 
This corroborates the conventional belief in image-classification convnets that output layers are more task-specific while input layers are more general, which is also validated in other meta-learning works~\citep{raghu2019rapid,gordon2019metalearning}.

However, in most of the subfigures in \cref{fig:omniglot_appendix} and \cref{fig:mini_appendix}, it is actually the \textit{penultimate}
layer that is most task-specific and should be adapted. 
The only algorithm for which the penultimate layer was not the most task-specific
was \shrimaml{} in the Omniglot experiment.
To study this further, we run \shrimaml{} on Omniglot repeatedly with a random initialization and hyper-parameter settings and keep the learned models that achieve at least $95\%$ test accuracy. 
In \cref{fig:omni_simaml_var_scatter}, we show the learned $\sigma_m^2$ 
of the final convolutional layer ($conv_3$) and the linear output layer ($output$) from those runs. In almost every case, $conv_3$ dominates $output$ by an order of magnitude, meaning that the selection of the output layer in \cref{fig:omniglot_appendix} was a rare event due to randomness. Adapting $conv_3$ remains the most stable choice, and most runs that adapt only it achieve a test accuracy that is not statistically significantly different from the result in \cref{fig:omniglot_appendix}(b).
Our results extend the results of the above meta-learning works that 
focus only on adapting
the final output layer, and match a recent independent observation
in \citet{arnold2019decoupling}.

Considering the different task-specific modules discovered in augmented Omniglot, TTS, sinusoid regression, and these two short-adaptation datasets, it is clearly quite challenging to hand-select task-specific modules \textit{a priori}. 
Instead, meta-learning a shrinkage prior provides a flexible and general method for identifying the task-specific and task-independent modules for the domain at hand.

\textbf{Learning a prior for a single module.}\\
We also tried learning a single shared prior $\sigma^2$ for all variables. When using a single module like this, 
$\sigma^2$ grows steadily larger during meta-learning, and the shrinkage algorithms simply reduce to their corresponding standard meta-learning algorithms without the shrinkage prior. This highlights the necessity of learning different priors for different components of a model.

This observation differs from the results of \citet{rajeswaran2019meta}, which reported an optimal value of $\lambda = 1/\sigma^2 = 2$ for iMAML. One possible explanation for this difference is that too small of a value for $\lambda$ leads to instability in computing the 
implicit gradient, which hinders learning. 
Also, when searching for the optimal $\lambda$ as a hyperparameter the dependence of the 
validation loss on $\lambda$ becomes less clear when $\lambda$ is sufficiently small. 
In our few-shot Omniglot experiments, we find that the best value of $\lambda$ for iMAML is $\lambda = 0.025$.

\subsubsection{Classification accuracy}
We compare the predictive performance of all methods on held-out tasks in \cref{table:classification}. 
While many of the results are not statistically different, the shrinkage variants
show a modest but consistent improvement over the non-shrinkage algorithms.
Further, \shrmaml{} does significantly outperform MAML on \textit{mini}ImageNet and
\shrreptile{} does significantly outperform Reptile in Omniglot.

Note that we do not in general expect the shrinkage variants to significantly outperform their counterparts here because the latter do not exhibit overfitting issues on these datasets 
and have comparable or better accuracy than other Bayesian methods in the literature~\citep{Grant2018,ravi2018amortized}.

\end{document}